\pdfoutput=1
\documentclass{article}
\PassOptionsToPackage{linesnumbered}{algorithm2e}
\usepackage{arxiv}
\usepackage{graphicx}
\usepackage{float}
\usepackage{amsmath}
\usepackage{amssymb}
\usepackage{mathtools}
\usepackage{amsthm}
\usepackage{bm}
\usepackage{dsfont}
\usepackage{times}
\usepackage{tabularx}
\usepackage{amsfonts}
\usepackage{nicefrac}
\usepackage{microtype}
\usepackage{xcolor}
\usepackage[numbers,sort&compress]{natbib}
\usepackage[toc,page,header]{appendix}
\usepackage{titletoc}
\usepackage[colorlinks=true,citecolor=blue]{hyperref}
\usepackage{url}
\usepackage{enumitem}
\usepackage{subcaption}
\usepackage[english]{babel}
\usepackage{parskip}
\usepackage{overpic}
\usepackage{multirow}
\usepackage{makecell}
\usepackage{tablefootnote}
\usepackage{comment}

\theoremstyle{plain}
\newtheorem{theorem}{Theorem}
\newtheorem{lemma}[theorem]{Lemma}
\newtheorem{proposition}[theorem]{Proposition}

\theoremstyle{remark}
\newtheorem{remark}[theorem]{Remark}
\theoremstyle{definition}

\newtheorem{assumption}{Assumption}

\usepackage[capitalize,noabbrev]{cleveref}
\crefname{equation}{Eq.}{Eqs.}
\crefname{assumption}{Assumption}{Assumptions}
\crefname{definition}{Definition}{Definitions}
\crefname{theorem}{Theorem}{Theorems}
\crefname{lemma}{Lemma}{Lemmas}
\crefname{proposition}{Proposition}{Propositions}
\crefname{remark}{Remark}{Remarks}
\crefname{corollary}{Corollary}{Corollaries}
\crefname{algocf}{Algorithm}{Algorithms}
\usepackage{etoolbox}
\AtBeginEnvironment{lemma}{\crefalias{theorem}{lemma}}
\AtBeginEnvironment{proposition}{\crefalias{theorem}{proposition}}

\newenvironment{algorithm}{\begin{algorithm2e}}{\end{algorithm2e}}
\SetAlgoShortEnd
\newenvironment{keywords}{\par\small\noindent\textbf{Keywords:} }{\par\normalsize}
\newcommand{\acks}[1]{\section*{Acknowledgments and Disclosure of Funding}#1}
\usepackage{accents}

\providecommand{\ind}[1]{\mathds{1}\cbrc{ #1 }}
\providecommand{\Ham}[1]{\mathrm{Ham}\brc{ #1 }}
\providecommand{\E}{\mathbb{E}}

\newcommand{\norm}[1]{\left\lVert #1 \right\rVert}
\newcommand{\abs}[1]{\left| #1 \right|}

\makeatletter
\DeclareRobustCommand{\rvec}[1]{\mathpalette\do@rvec{#1}}
\newcommand{\do@rvec}[2]{\fix@rvec{#1}{+}\reflectbox{\ensuremath{\m@th#1\vec{\reflectbox{\ensuremath{\fix@rvec{#1}{-}\m@th#1#2\fix@rvec{#1}{+}}}}}}\fix@rvec{#1}{-}}
\newcommand{\fix@rvec}[2]{\ifx#1\displaystyle
    \mkern#23mu
  \else
    \ifx#1\textstyle
      \mkern#23mu
    \else
        \mkern#22mu
    \fi
  \fi
}
\makeatother

\providecommand{\KL}[2]{\mathrm{KL}( #1 || #2 )}
\providecommand{\TV}[2]{\mathrm{TV}( #1 , #2 )}
\newcommand{\brc}[1]{\left( #1 \right)}
\newcommand{\sbrc}[1]{\left[ #1 \right]}
\newcommand{\cbrc}[1]{\left\{ #1 \right\}}
\providecommand{\eps}{\varepsilon}
\providecommand{\T}{\intercal}
\renewcommand{\d}{\mathrm{d}}

\providecommand{\calM}{\mathcal{M}}

\providecommand{\calL}{\mathcal{L}}

\providecommand{\calO}{\mathcal{O}}

\providecommand{\mask}{\mathrm{[MASK]}}

\title{Sharp Convergence Rates for Masked Diffusion Models}

\author{
\begin{minipage}[t]{0.46\textwidth}\centering
{Yuchen Liang\thanks{These authors contributed equally.}}\\[4pt]
Indiana University
\end{minipage}
\And
\begin{minipage}[t]{0.46\textwidth}\centering
{Zhiheng Tan\footnotemark[1]}\\[4pt]
The Ohio State University
\end{minipage}
\AND
\begin{minipage}[t]{0.46\textwidth}\centering
{Ness Shroff}\\[4pt]
The Ohio State University
\end{minipage}
\And
\begin{minipage}[t]{0.46\textwidth}\centering
{Yingbin Liang}\\[4pt]
The Ohio State University
\end{minipage}
}

\begin{document}

\maketitle

\begin{abstract}Discrete diffusion models have achieved strong empirical performance in text and other symbolic domains, with masked (absorbing-rate) variants emerging as competitive alternatives to autoregressive models. Among existing samplers, the Euler method remains the standard choice in many applications, and more recently, the First-Hitting Sampler (FHS) has shown considerable promise for masked diffusion models. Despite their practical success, the theoretical understanding of these samplers remains incomplete. Many existing analyses are conducted in Kullback–Leibler (KL) divergence, which often yields loose parameter dependencies and requires strong assumptions on score estimation. Moreover, these guarantees do not cover recently developed high-performance sampler of FHS. In this work, we first develop a direct total-variation (TV)-based analysis for the Euler method that overcomes these limitations. Our results relax assumptions on score estimation, improve parameter dependencies, and establish convergence guarantees without requiring any surrogate initialization. Also for this setting, we provide the first convergence lower bound for the Euler sampler, establishing tightness with respect to all of the data dimension $d$, the vocabulary size $S$, and the target accuracy $\varepsilon$, up to logarithmic factors. Finally, we analyze the FHS sampler and show that it incurs no sampling error beyond that induced by score estimation, which we show to be tight with a matching lower error bound. 
Our analysis introduces two techniques that may be of independent interest: a comparison with a reference unmasking schedule for the Euler method and a pathwise argument exploiting the decoupling structure of FHS.
\end{abstract}

\begin{keywords}
  discrete diffusion models, masked diffusion models, Euler method, First-Hitting Sampler, convergence analysis
\end{keywords}

\allowdisplaybreaks

\section{Introduction}

Diffusion models have become one of the central paradigms of generative modeling, offering state-of-the-art performance across a wide range of domains, including image synthesis \citep{dhariwal2021diff-beat-gan}, video generation \citep{melnik2024video-diffusion-survey}, audio generation \citep{huang2023audio}, and text modeling \citep{austin2021discrete,lou2024entropy}. The core idea is to define a forward noising process that gradually perturbs data into a tractable distribution, together with a reverse denoising process that reconstructs samples from noise. 

Recently, for applications in discrete and symbolic domains such as language, graphs, and molecules \citep{diffusion-survey-graph, diffusion-survey-drug-design}, {\bf discrete-state diffusion models} have demonstrated advantages over their continuous counterparts, achieving strong perplexity and negative log-likelihood performance on language modeling benchmarks \citep{austin2021discrete, lou2024entropy, shi2024simplified, ou2025absorb}. In the discrete setting, the forward noising process is typically formulated as a continuous-time Markov chain (CTMC) governed by a structured rate matrix \citep{campbell2022discrete}. Two rate matrices are most commonly used: a {\em uniform} rate matrix, whose stationary distribution is uniform over the state space, and an {\em absorbing} rate matrix, whose stationary distribution concentrates on a fully masked absorbing state. The absorbing-rate diffusion models, often referred to as masked diffusion models (MDMs), are particularly appealing due to its close connections to established architectures such as BERT \citep{devlin2019bert} and conditional masked language models (CMLM) \citep{ghazvininejad2019cmlm}. Empirically, MDMs have achieved competitive performance relative to autoregressive models \citep{shi2024simplified}.

Despite these empirical advances, the theoretical understanding of discrete diffusion models, especially MDMs, remains limited. 
Existing works using deterministic-step-size samplers have largely relied on convergence analyses in terms of Kullback-Leibler (KL) divergence, most under the uniform-rate matrix \citep{chen2024uniformization,zhang2025conv-disc,conforti2025markov,ren2025stoc-int,liang2025sampler,dmitriev2026discrete}, and, more recently, some under the absorbing-rate \citep{liang2025absorb,huang2025mask,conforti2025discrete,dmitriev2026discrete}. 
However, for the absorbing-rate discrete diffusion, there are currently no lower bounds showing whether existing guarantees are tight. In {\em continuous} diffusion, prior work \citep{li2024accl-prov} often finds that TV analyses yield sharper bounds and more accurate parameter dependence such as the dimension $d$. This motivates a TV-based analysis for {\em discrete} diffusion to obtain tighter guarantees that potentially admit matching lower bounds. Doing so is nontrivial, since KL-specific tools such as change-of-measure arguments \citep{ren2025stoc-int,dmitriev2026discrete} do not extend to TV.
Furthermore, these KL-based guarantees often require strong assumptions on the estimated score function. For example, all of \citep{chen2024uniformization,zhang2025conv-disc,ren2025stoc-int,liang2025sampler} require the assumption of a bounded score estimate, which needs typically be achieved through an additional clipping step during training \citep{zhang2025conv-disc}. 
These limitations give rise to our first central question. 

\emph{Question 1: Can we establish convergence guarantees with \emph{tighter} parameter dependencies via a direct analysis in TV distance? Moreover, can we derive matching lower bounds for the existing upper bounds?}

More recently, the First-Hitting Sampler (FHS) has been shown to be highly effective for generation with MDMs \citep{zheng2025fhs}. Unlike classical deterministic-step-size samplers, such as Euler methods, whose number of steps must grow unboundedly to achieve vanishing TV distance, FHS can produce a high-quality sample in exactly $d$ steps, where $d$ denotes the data dimension. Despite its empirical success, the theoretical error bound of FHS has not been explored. Given its different design philosophy, the analysis of FHS requires new techniques that explicitly exploit its underlying structure. This leads to our next question below.

\emph{Question 2. Given a $d$-step sampling procedure like FHS, how to characterize the sampling error of the algorithm? Moreover, can we also derive a matching lower bound for FHS?}

In this work, we address both questions and give affirmative answers to both of them.

\subsection{Our Contributions}

The central contributions of this work are (i) a new direct TV analysis tailored to the Euler method for absorbing-rate discrete diffusion models with a matching lower bound, and (ii) a novel analysis for the FHS algorithm, leading to 
tight error bound that matches our lower bound.
Our detailed contributions include:

\textbf{Improved Parameter Dependencies for the Euler Method (with straightforward generalization to $\tau$-leaping type samplers):} For the Euler method for the absorbing-rate (masked) discrete diffusion models,
we sharpen existing theoretical guarantees with an improved dependence on the target accuracy $\eps$, data dimension $d$, and vocabulary size $S$.
Note that the same guarantees can be straightforwardly extended to vanilla $\tau$-leaping samplers using the technique in \citet{liang2025sampler}.
In particular, we show that, in order to reach a slightly perturbed target distribution, the required number of steps is 
$\widetilde{\calO}\brc{d \min\{d,\log S\}/ \sqrt{\eps}}$. Compared to prior bounds mostly with $\eps^{-1}$ dependence, ours improves the dependence on $\eps$ by a factor of $\calO(1/\sqrt{\eps})$ and achieves a joint dependence of $d\min\{d,\log S\}$ on $d$ and $S$, combining logarithmic dependence on $S$ with an $S$-independent upper bound of $\calO(d^2)$. Notably, our analysis also removes the need for bounded score assumptions and the need for a surrogate initialization typically required in previous works \citep{liang2025absorb,huang2025mask,conforti2025discrete,dmitriev2026discrete}.

\textbf{Tight Lower Bound for the Euler Method:} We provide, to our knowledge, the first \emph{lower-bound} analysis for absorbing-rate discrete diffusion models using the Euler method in TV distance. 
The lower bound matches our upper bound in the dependence on all of the data dimension $d$, the vocabulary size $S$, and the target accuracy $\eps$.
Thus, our result identifies the intrinsic limits of convergence when the Euler method is used.

\textbf{Error Analysis for FHS and Tight Lower Bound:}
We provide a rigorous analysis for the First-Hitting Sampler (FHS) \citep{zheng2025fhs} for masked diffusion models, which generates a sample in exactly $d$ steps. 
Specifically, we show that FHS incurs no sampling error beyond that arising from score estimation itself. In particular, the resulting error bound is independent of both the data dimension $d$ and the vocabulary size $S$. 
We further show that this bound is tight, by constructing a worst-case instance for which the FHS output error matches exactly the upper bound of the score estimation error. Our result further highlights the advantages of masked diffusion models, for which efficient finite-step sampling with dimension-free guarantees is possible. 
{Specifically, we develop a novel error analysis that decomposes the total sampling error into token-wise components. Our key insight is that FHS exploits an \textit{intrinsic decoupling structure} in MDMs: the transition time $\tau$ and transition index $i$ can be sampled exactly from the true reverse process, regardless of the accuracy of score estimation. Consequently, the overall sampling error depends solely on the estimation error incurred at each predicted token, and is therefore independent of $d$ and $S$.}

\begin{table*}[t]
    \centering
    \caption{{\em Summary of results for absorbing-rate (masked) discrete diffusion samplers in terms of the number of steps needed to achieve $\sqrt{\eps}$-accuracy in total-variation.}}
    \begin{tabular}{c|c|c|c}
        \textbf{Algorithm} & \textbf{Estimation Error} & \textbf{Num of steps} & \textbf{Reference} \\ \hline \hline

        $\tau$-leaping & Score Entropy &  $\calO\brc{\frac{d S \log M}{\eps}}$ & \citet{liang2025absorb} \\ \hline

        Euler & Score Entropy & $\calO\brc{\frac{d^2 S^2 M^2}{\eps^{3/4}}}$ & \citet{huang2025mask} \\ \hline

        DMPM & Score Entropy & $\calO\brc{\frac{d S}{\eps}}$ & \citet{conforti2025discrete} \\ \hline

        Modified $\tau$-leaping & Score Entropy & $\calO\brc{\frac{d \log S}{\eps}}$ & \citet{dmitriev2026discrete} \\ \hline
        
        \color{blue} Euler & \color{blue} $L_1$ & \color{blue} $\calO\brc{\frac{d \min\{d, \log S\}}{\eps^{1/2}}}$ & \color{blue} (This paper, Thm~\ref{thm:conv_tv_absorb}) \\ \hline

        \color{blue} FHS & \color{blue} Score Entropy & \color{blue} Exactly $d$ & \color{blue} (This paper, Thm~\ref{thm: error bound of fhs})
    \end{tabular}
    
    \vspace{5mm}
    \parbox{\linewidth}{\small \textit{\bf Comparison:} Our result for the Euler method achieves the best dependence on $\eps$ and $S$ among $\tau$-leaping type algorithms while retaining small dependencies on $d$ and being $M$-free. 
    Note that we achieve the first $S$-free convergence rate.
    Moreover, we provide the first sampling guarantee for FHS, which attains the same sampling error as other methods but using \emph{exactly} $d$ steps.\\
    \textit{\bf Notes:} $(i)$ We only include those samplers finishing in deterministic number of steps, which exclude uniformization \citep{chen2024uniformization}.
    $(ii)$ Here, $d$ denotes the dimension (e.g., length of the generated sentence), $S$ denotes the vocabulary size, and $M$ denotes the upper bound on the score estimates.
    $(iii)$ In Proposition~\ref{prop:est-relax}, we show that our $L_1$ estimation error is more relaxed than score-entropy.
    $(iv)$ While \citet{huang2025mask} does not show the dependence on $M$ in the main text, the dependence can be straightforwardly derived from their proof.
    $(v)$ The step complexity in \citet{dmitriev2026discrete} is $\calO(\mathcal{D}/\epsilon)$, where $\mathcal{D}$ is distribution-dependent. Here we consider the worst case, where $\mathcal{D}\le d\log S$.
    }
    \label{tab:literature}
     \vspace{-5mm}
\end{table*}

\subsection{Related Works}

See \Cref{app:works} for a full list of related works. As follows we only include a brief survey of the most relevant works.

\textbf{Theory on Uniform-Rate Discrete Diffusion Models.} Considerable effort has been devoted to improving convergence guarantees for uniform-rate diffusion models. Unlike continuous diffusion, the resulting guarantees in the discrete setting crucially depend on the choice of sampler. Early analyses focused on uniformization-based methods \citep{chen2024uniformization,ren2025stoc-int}, exact per-step samplers \citep{zhang2025conv-disc}, and $\tau$-leaping schemes \citep{campbell2022discrete,dmitriev2026discrete,ren2025stoc-int}. More recently, theoretical guarantees have been established for DMPM \citep{conforti2025discrete,conforti2025markov} as well as for more practical Tweedie $\tau$-leaping and Euler method \citep{liang2025sampler}.
Notably, all of these results are derived in terms of KL divergence, which yields suboptimal guarantees when translated to total variation distance. 
More recently, \citet{liang2026inhomogeneous} investigated the effect of non-homogeneous noise schedules for uniform-rate models.

\textbf{Theory on Masked (Absorbing-rate) Discrete Diffusion Models.}
\citet{liang2025absorb} was the first to establish convergence guarantees for masked diffusion models under the $\tau$-leaping and the uniformization samplers, showing improved dimensional dependence compared to uniform-rate diffusion. \citet{huang2025mask} studied faster convergence rates using the Euler method and a specialized uniformization scheme, the Mask-Aware Truncated Uniformization (MATU), which eliminates the need for a bounded score assumption. In parallel, \citet{conforti2025discrete} derived improved dimension-dependent rates for the DMPM sampler, also without assuming bounded score. More recently, \citet{dmitriev2026discrete} developed modified truncated $\tau$-leaping that achieves a convergence rate of $\calO(\mathcal{D}/\epsilon)$ under the KL divergence, where $\mathcal{D}\le d\log S$ is distribution-dependent.

\section{Preliminaries of Discrete Diffusion Models}
\label{sec:prelim}

As in standard diffusion models, discrete diffusion models comprise a forward noising process and a reverse generation process. In this section, we describe two common frameworks used to model discrete diffusion dynamics. 

\subsection{CTMC Forward Process}

Discrete diffusion models can be modeled as continuous-time Markov chains (CTMCs) defined over the {\em discrete} data space $[S]^d$. Here $d$ is the number of tokens and each token is drawn from a vocabulary of size $S$ \citep{campbell2022discrete}.
In the forward process, the initial data is denoted by $x_0 \in [S]^d$, with corresponding probability mass function (p.m.f.) $q_0$. Let $R_t \in \mathbb{R}^{S^d \times S^d}$ be the rate matrix of the forward process. For states $x,y \in [S]^d$, $R_t(x,y)$ specifies the instantaneous rate of transition from $x$ to $y$ at time $t$. The conditional distribution of $y$ at $t+\Delta t$ given $x$ at $t$ satisfies
\begin{equation*}
    q_{t+\Delta t|t}(y|x) = \ind{y=x} + R_t(x,y)\Delta t + o(\Delta t)
\end{equation*}
where $\ind{\cdot}$ is the indicator function. For a valid CTMC, $R_t(x,y)\geq 0$ whenever $x\neq y$, and $\sum_{y}R_t(x,y)=0$. 

To simplify computation for large $S$ and $d$, it is common to assume that tokens evolve {\em independently} and {\em homogeneously} across dimensions \citep{lou2024entropy,campbell2022discrete}. Under this assumption, the forward conditional distribution factorizes as
\begin{equation} \label{eq:forward-factorize}
    q_{t|0}(x_t|x_0) = \prod_{i=1}^d q_{t|0}^i(x_t^i|x_0^i),
\end{equation}
where $q_{t|0}^i$ is the conditional for the $i$-th token. Equivalently, if we define the corresponding token-level rate matrix by $R_t^{tok}\in\mathbb{R}^{S\times S}$, then \citet{campbell2022discrete} shows that
\begin{equation}\label{eq:def_forward_rate}
    R_t(x,y) = \begin{cases}
        R_t^{tok}(x^i, y^i), & \text{if }\Ham{x,y}=1, \\
        0, & \text{if }\Ham{x,y} \geq 2.
    \end{cases}
\end{equation}
where $\Ham{x,y}$ is the Hamming distance between $x$ and $y$.  
Following \citet{campbell2022discrete}, we set $R_t^{tok} = \beta_t R_{\text{base}}$ with a noise schedule $\beta_t>0$. In this paper, we focus on the constant schedule $\beta_t \equiv 1$, as in prior works \citep{liang2025absorb,zhang2025conv-disc,ren2025stoc-int,liang2025sampler}. 

In the literature of CTMC framework, 
the \textbf{absorbing-rate} (corresponding to masked diffusion model) matrix has attracted special attention \citep{austin2021discrete,lou2024entropy}. 
Define one vocabulary word as $\mask \in [S]$. 
The rate matrix is given by
\begin{equation}
\label{def:rate_absorb}
    R_{\text{base}} := \mathbf{1}_S e_\mask^\T - I_S,
\end{equation}
where $\mathbf{1}_S$ is an all-one vector of length $S$, and $e_i$ is a unit vector with only the $i$-th element being 1. Then, $q_T \approx \bm{\delta}_{\mask^d}$, a singleton at the all-mask state.

Throughout this paper, we assume that no mask is in the data, and thus the data distribution $q_0$
is supported on $([S]\setminus\{\mask\})^d$.

\subsection{CTMC Reverse Process and the Euler Method}

The reverse process under the CTMC framework is defined as the exact time reversal of the forward CTMC, with initial distribution $\rvec{q}_0 = q_T$ \citep{campbell2022discrete,kelly2011reverse}. In other words, $\rvec{q}_t = q_{T-t}$ for all $t\in[0,T]$. By \citet[Proposition~1]{campbell2022discrete}, the reverse process is also a CTMC with the rate matrix given by
\begin{equation}\label{eq:def_rev_proc}
    \rvec{R}_t(x,y) = R_{T-t}(y,x)\frac{q_{T-t}(y)}{q_{T-t}(x)}, \quad x\neq y,
\end{equation}
and $\rvec{R}_t(x,x) = -\sum_{y\neq x}\rvec{R}_t(x,y)$.  
In practice, to avoid instability near $t=0$, one often employs {\em early-stopping} by truncating the reverse chain at $t=T-\delta$ for some small $\delta>0$. 
Also, since the ratio $\frac{q_{t}(y)}{q_{t}(x)}$ (a.k.a., the \textbf{concrete score function}) is intractable, one typically train a neural-network to estimate it. Let $s_{t}(y,x)$ be such an estimator. 
One typical loss is the score-entropy (SE) loss \citep{lou2024entropy}:
\begin{equation}
\label{eq:def_score_entropy}
    \mathcal{L}_{SE}(s;t) = \mathbb{E}_{x_t \sim q_t} \sum_{y:y\neq x_t} R_{t}(y,x_t) \cdot \brc{s_t(y,x_t) - \frac{q_t(y)}{q_t(x_t)} -\frac{q_t(y)}{q_t(x_t)}\log\frac{s_t(y,x_t)}{q_t(y) / q_t(x_t)}}.
\end{equation}
Let $q_t$ be the marginal p.m.f. at time $t \in [0,T-\delta]$ in the sampling process. Since $q_T$ is unavailable, the sampling process starts with $p_0$ to be the stationary distribution of the forward CTMC. Then, the continuous-path is discretized for practical algorithms. Let $\{t_k\}_{k\in [N]}$ be the discretization points on which $s_{T-t_k}$ is accessible, where $t_0 = 0$ and $t_N = T-\delta$. Let $\eta_k := t_{k+1}-t_k$ be the step-sizes. Define the estimated reverse rate as 
\begin{equation} \label{eq:def_est_reverse_rate}
    \hat{R}_{t_k}(x,y) := R_{T-t_k}(y,x) s_{T-t_k}(y,x).
\end{equation}
To achieve efficient sampling from the continuous-time process, one commonly used sampler is the \textbf{Euler method} \citep{lou2024entropy,nisonoff2025dcfg}.\footnote{Our results for the Euler method in this paper naturally extend to the vanilla $\tau$-leaping and Tweedie $\tau$-leaping using the techniques in \citet{liang2025sampler}.}
For each $k=0,\dots,N-1$, the next-token on the $i$-th index is given as
\begin{equation} \label{eq:def_euler}
x^i_{k+1} =
\begin{cases}
    a, & \text{w.p.}~\hat{R}_{k}^i(x_{k}^i,a) (t_{k+1}-t_k),~\forall a \neq x^i_{k}\\
    x^i_{k}, & \text{w.p.}~1 + \hat{R}_{k}^i(x_{k}^i,x_{k}^i) (t_{k+1}-t_k)
\end{cases},
\end{equation}
where $\hat{R}_k^i(x_{k}^i,a) := \hat{R}_{t_k}(x_{k}, x_{k}^{-i} \oplus_i a)~(\text{where}~a \neq x_{k}^i)$ is the token-wise rate, with $\hat{R}_k^i(x^i,x^i) = -\sum_{a: a\neq x^i} \hat{R}_k^i(x^i,a)$. Here $x^{-i} \oplus_i a$ denotes the vector obtained by replacing the $i$-th entry with $a$. 
Also, the updates at each coordinate are conditionally independently given $x_{k}$.

\subsection{D3PM and MDM}

Apart from the continuous-time Markov chain (CTMC) framework, discrete diffusion models can also be modeled as in Discrete Denoising Diffusion Probabilistic Models (D3PMs). 
In the \textbf{forward process}, still assuming that each dimension propagates independently as in \cref{eq:forward-factorize}, the (token-wise) transition probability is represented as 
\begin{equation}
\label{eq: d3pm forward}
q^i_{t|0}(x_t^i|x_0^i) = \mathrm{Cat}(\bar{Q}_t^\T e_{x_0^i}),
\end{equation}
where $\mathrm{Cat}(\mathbf{p})$ represents a categorical distribution with parameter $\mathbf{p}$, 
and $\bar{Q}_t$ is the transition probability matrix. 
Specifically, define the \textbf{mask} transition matrix as
\begin{equation} \label{eq:d3pm-mask-Q}
    \bar{Q}_t = \alpha_t I + (1-\alpha_t) \mathbf{1}_S e_{\mask}^T.
\end{equation}
Here $\alpha_t$ is some pre-defined noise schedule, and in this paper we adopt a typical choice that $\alpha_t = e^{-t}$.

Following \citet{austin2021discrete,ho2020ddpm}, a so-called ``bridge distribution'' $q_{s|t,0}$ for $s < t$ is particularly useful.
Indeed, under the mask transition in \cref{eq:d3pm-mask-Q}, one can show that \citep{zheng2025fhs,sahoo2024mdlm}
\begin{equation}
\label{eq:def_d3pm_bridge}
    q(x_s^i|x_t^i,x_0^i)= 
    \begin{cases}
        \mathrm{Cat}(e_{x_t^i}), & \text{if } x_t^i\neq \mask \\
        \mathrm{Cat}\Big(\frac{(1-\alpha_s)e_{\mask} + (\alpha_s - \alpha_t)(e_{x_0^i})}{1-\alpha_t}\Big), & \text{if } x_t^i= \mask
    \end{cases}.
\end{equation}
With such a bridge distribution, in the sampling process of MDMs with discretized time grid $\{t_k\}_{k \in [N]}$, a sample $x_{t_{k-1}}$ can be obtained token-wisely with the sampling probabilities $q(x_{t_{k-1}}^i|x_{t_k}) = \sum_{x_0^i} q(x_{t_{k-1}}^i|x_{t_k}^i,x_0^i)  q(x_0^i|x_{t_k})$.
In practice, however, the posterior $q(x_0^i|x_{t_k})$ is typically unknown. Thus, in order to sample $x_{t_{k-1}}^i$ in MDMs, one typical way is to train a predictive model $\mu^i_\theta(x_t,t)$ to estimate the posterior conditional distribution $q(x_0^i|x_{t_k})$. Then, given the discretized time $t_{k}$'s and the current state $x_{t_k}$, the algorithm replace the $e_{x_0}^i$ above with $e_{\hat{x}_0}^i$ sampled from the predictive model $\mu_\theta^i(x_{t_k},t_k)$ at each step. 

Throughout the paper, since there is no mask in the data, we also assume that $\mu_\theta^i$ puts no mass on mask tokens and is normalized to be a probability distribution, i.e., $\mu_\theta^i(x,t)[\mask]=0$ and $\sum_{a\ne\mask}\mu_\theta^i(x,t)[a]=1$.

One typical loss function for training $\mu_\theta$ is \textit{negative evidence lower bound} (NELBO) of the parameterized model \citep{zheng2025fhs}, given as
\begin{equation}
\label{eq:def of nelbo}
\mathcal{L}_\infty(x_0)
= \int_0^\infty \frac{\alpha_t'}{1-\alpha_t}
\; \mathbb{E}_{q_{t|0}(x_t \mid x_0)} \left[
\sum_{l:\, x_t^{l}=\mask} e_{x_0^l}^\top
\log \mu_\theta^{l}(x_t,t)
\right] dt,
\end{equation}
where $\alpha_t' = \d \alpha_t/\d t = -e^{-t}$. 

\paragraph{Relationship between CTMC and D3PM formulations.} Interestingly, there is a close relationship between the predictive model $\mu_\theta$ in MDMs and the score estimator $s_t(y,x)$ in the absorbing-rate CTMC \citep{zheng2025fhs}:
\begin{equation}
    \label{eq:equivalence-between-score-and-mu}
    s_t(x^{-i} \oplus_i a,x) = \frac{\alpha_t}{1-\alpha_t}\,\mu_\theta(x,t)[a],\quad \forall x^i = \mask, a \neq \mask.
\end{equation}
Indeed, with $\alpha_t = e^{-t}$, 
one can show that the forward process of such an MDM is equivalent to that of the time-homogeneous absorbing-rate CTMC. Therefore, such MDM shares the same reverse process with the absorbing-rate CTMC. 
This relationship can also be generalized from between estimators into between true values, i.e., concrete score $\frac{q_t\!\left(x_t^1,\cdots,\hat{x}_t^i,\cdots,x_t^d\right)}{q_t\!\left(x_t^1,\cdots,x_t^i,\cdots,x_t^d\right)}$ and true token-wise reverse conditional distribution $q_0\!\left(\hat{x}_t^i | x_t^{\mathrm{UM}}\right)$ \citep{ou2025absorb}:
\begin{equation}
\label{eq: equivalence-between-true-score-and-clean-distribution}
\frac{q_t\!\left(x_t^1,\cdots,\hat{x}_t^i,\cdots,x_t^d\right)}
     {q_t\!\left(x_t^1,\cdots,x_t^i,\cdots,x_t^d\right)}
=
\frac{\alpha_t}{1-\alpha_t}
\,
q_0^i\!\left(\hat{x}_t^i | x_t^{UM}\right),
\end{equation}
where $x_t^{UM}$ denotes the collection of unmasked coordinates. Given such a close relationship, we will use ``MDMs'' and ``absorbing-rate diffusion models'' interchangeably.

\subsection{First Hitting Sampler (FHS)}

First Hitting Sampler (FHS), as referenced in \Cref{alg:first-hitting-sampling}, is a sampling algorithm specifically designed for MDMs, which was first proposed in \citet{zheng2025fhs}. 
Instead of employing a pre-defined discretized time-grid as seen in previous samplers (e.g., $\tau$-leaping and the Euler's method),
FHS simulates the sampling process in MDMs 
by directly characterizing each individual unmasking event. 
As a result, FHS avoids unnecessary intermediate steps and requires exactly 
$d$ sampling steps to generate a complete sequence, making it a substantially more efficient sampler. 

Despite such a significant empirical improvement, there still lacks theoretical characterization of the convergence error of FHS, i.e., how score estimation error influences the convergence rate of FHS. Furthermore, it also remains unclear whether the minimizer of the training objective (NELBO) for MDMs coincides with the minimizer of the score entropy loss. To address these open questions and provide a principled understanding of FHS, we will conduct analysis on FHS in \Cref{sec: convergence rate fhs}.

\begin{algorithm}
\caption{First Hitting Sampler (FHS) \citep{zheng2025fhs}}
\label{alg:first-hitting-sampling}

\KwIn{
Noise schedule $\alpha_t = e^{-t}$ with inverse $\alpha^{-1}(u) = -\log u$,
pretrained masked diffusion model $\mu_\theta$}
\BlankLine
\hrule
\BlankLine

$x_d \gets [\mask,\mask,\ldots,\mask]$\;
$\alpha_{\tau_d} \gets 0$\;

\For{$n \gets d$ \KwTo $1$}{
  Sample $u_n \sim \mathrm{Uniform}([0,1])$\;
  
  $\tau_{n-1} \gets \alpha^{-1}\!\bigl(1 - u_n^{1/n}(1-\alpha_{\tau_n})\bigr)$\;
  
  Randomly select $l$ uniformly from $\{\, i : x_n^{(i)}=\mask \,\}$\;
  
  $x_{n-1} \gets x_n$\;
  
  Sample $z \sim \mathrm{Cat}\!\bigl(\mu_\theta^l(x_n,\tau_{n-1})\bigr)$\;
  
  $x_{n-1}^{(l)} \gets z$\;
}

\KwOut{$x_0$}
\end{algorithm}

\subsection{Key Notations}

Let $x^i~(1\leq i \leq d)$ denote the $i$-th element of a vector $x \in [S]^d$ and $x^{-i} \in [S]^{d-1}$ denote the $i$-th element removed.
Define $\Ham{x,y}$ as the Hamming distance between two vectors $x$ and $y$.
For a positive integer $n$, $[n] := \{1,\dots,n\}$. 
Write $\bm{1}_S$ as a vector of length $S$ that contains all 1's, and $I_S$ as an identity matrix of size $S \times S$.
We have provided additional notations in \Cref{app:notations}.

\section{Convergence Guarantees for the Euler Method}
\label{sec:euler}

In this section, we present our improved upper error bounds and the first convergence lower bound for the Euler method.
In particular, differently from most existing theoretical results (e.g., \citet{liang2025absorb,ren2025stoc-int,dmitriev2026discrete}), our results are directly characterized by TV distance without providing an upper bound in KL divergence followed by invoking the Pinsker's inequality. As a result, our approach yields tighter bounds on the number of steps and improves other parameter dependencies as well.

\subsection{Convergence Upper Bound for the Euler Method}

In this subsection, we characterize the convergence rate with explicit parameter dependencies for the Euler method. Our analysis is developed under the following estimation error.

\begin{assumption}
\label{ass:score_tv}
    The score estimation error satisfies that $\calL_{TV}(s) \leq \sqrt{\eps_{\text{score}}}$, where
    \begin{equation*}
        \calL_{TV}(s) := \sum_{k=0}^{N-1} (t_{k+1}-t_k) \E_{x_{k} \sim q_{T-t_k}} \bigg[
        \sum_{y: y \neq x_{k}} R_{T-t_k}(y,x_{k}) \abs{ s_{T-t_k}(y,x_{k}) - \frac{q_{T-t_k}(y)}{q_{T-t_k}(x_{k})} } \bigg] .
    \end{equation*}
\end{assumption}

Here we assume that the score estimation is accurate in $L_1$ (i.e., the expected sum of absolute differences between the rate matrices), a condition that naturally arises from our direct TV-based analysis. A similar criterion also appears in the TV-based analysis of \citet{campbell2022discrete}.

While both \Cref{ass:score_tv} and \citet[Theorem~1]{campbell2022discrete} capture the estimation error in terms of the sum of absolute rate differences, our assumption is more relaxed than that in \citet[Theorem~1]{campbell2022discrete}. Specifically, \citet[Theorem~1]{campbell2022discrete} requires that the sum of difference is small for \textit{all} $x$ \textit{and} time $t$, which results in a loose bound around $t \approx 0$ where the score function diverges for all $x$. In comparison, \Cref{ass:score_tv} only requires the sum \textit{average} over both $x$ and $t$ to be small. 

An alternative objective commonly used in the literature is the score entropy loss \citep{lou2024entropy}, which has also been adopted in prior theoretical analyses \citep{zhang2025conv-disc,ren2025stoc-int,liang2025sampler}.
In \Cref{prop:est-relax} below, we connect the two estimation errors and show that our \Cref{ass:score_tv} is more relaxed than the (maximum) score entropy error.
\begin{proposition} \label{prop:est-relax}
Recall the score entropy loss at time $t$ in \cref{eq:def_score_entropy}. 
Then, if $\sup_{t\in(\delta,T]}\mathcal{L}_{SE}(s;t) \leq \eps'_{\mathrm{score}}$, with any discretization grid, we have
\[ \mathcal{L}_{TV}(s) \leq \pi\sqrt{2d\eps'_{\mathrm{score}}}. \]
In particular, Assumption~\ref{ass:score_tv} holds with $\eps_{\text{score}} \lesssim d\eps'_{\mathrm{score}}$.
\end{proposition}

The following theorem provides our 
convergence rate for the Euler method.

\begin{theorem}\label{thm:conv_tv_absorb}
Suppose that \Cref{ass:score_tv} hold.
Initialize the Euler method (see \Cref{eq:def_euler}) at $\bm\delta_{\mask^d}$ and choose the step-sizes as $\eta_k = \kappa \min\cbrc{1, T-t_k}$ (where $t_k$ is the time-discretization in the reverse time), with $0<\kappa\le1/2$ and $\delta$ a forward-time grid point. Then for sufficiently large $d, S, T$ and small $\delta$, we have
\begin{equation*}
 \TV{q_0}{p_{\mathrm{Euler}}}
 \lesssim 
 de^{-T} + d \delta + \sqrt{\eps_{\mathrm{score}}} + \kappa d \min\{d,\log S\}.
\end{equation*}
Thus, if $\eps_{\text{score}} \lesssim \eps$, $\delta \asymp \sqrt{\eps} / d$, $T \asymp \log(d/\sqrt{\eps})$, and $\kappa \asymp \frac{\sqrt\eps}{d\min\{d,\log S\}}$, we achieve $\TV{q_0}{p_{\mathrm{Euler}}} \lesssim \sqrt{\eps}$ with $N = \widetilde{\calO}\brc{d\min\{d,\log S\}/\sqrt\eps}$ sampling steps.
\end{theorem}

Note that such a guarantee can be directly extended for $\tau$-leaping-type samplers using the techniques in \citet{liang2025sampler} (see details from Remark~\ref{rem:euler-new-tau}). Compared to prior works, \Cref{thm:conv_tv_absorb} features the following improvements. 
\begin{itemize}

\item[(i)] \textbf{(Improved dependence on the target accuracy $\eps$)} \Cref{thm:conv_tv_absorb} provides the \textit{first} guarantee for the \textit{absorbing-rate} discrete diffusion models to achieve $\calO(1/\sqrt{\eps})$ sampling complexity, among deterministic sampling schedules such as $\tau$-leaping or Euler samplers. This improves the dependence of the number of steps on $\eps$ in \citet{liang2025absorb,conforti2025discrete,dmitriev2026discrete} by a factor of $\calO(1/\sqrt{\eps})$. 

\item[(ii)] \textbf{(Improved dependence on system parameters $d$ and $S$)} Beyond its improved dependence on $\eps$, \Cref{thm:conv_tv_absorb} provides a step-complexity bound with joint worst-case dependence $d\min\{d,\log S\}\le d\log S$. In comparison, the guarantees of \citet{liang2025absorb,huang2025mask,conforti2025discrete} all exhibit at least linear dependence on $S$. The distribution-dependent bound of \citet{dmitriev2026discrete} scales with $\mathcal{D}$, where $\mathcal{D}$ is distribution-dependent with $\mathcal{D} \le d\log S$ in the worst case. Compared with this worst-case bound, our dependence on $d$ and $S$ improves when $\log S > d$. Specifically, our result additionally provides a guarantee independent of $S$ when it becomes large enough. 
This regime can arise when generating short recommendation lists from large item catalogs. For example, at the scale of Amazon Reviews 2023, which contains over 48 million items \citep{hou2026retrieval-recommendation}, generating five items jointly with one token per item would give $d=5<\log S\approx17.7$.

\item[(iii)] \textbf{(Exact Initialization)} While all of \citet{liang2025absorb,huang2025mask,dmitriev2026discrete} require the algorithm to initialize from a surrogate distribution, here we can safely initialize from $\bm{\delta}_{\mask^d}$, which is the singleton at the all-mask state. The underlying reason is that while $\KL{q_T}{\bm{\delta}_{\mask^d}} = \infty$, the TV error does not diverge due to such an irregularity. Thus, the convergence error can still be well-controlled even if starting from $p_0 = \bm{\delta}_{\mask^d}$ (see \Cref{eq:euler-new-boundary}). 

\item[(iv)] \textbf{(Removing boundedness assumption)} Differently from \citet{liang2025absorb,huang2025mask}, we do not require that the score estimates are bounded. 

\end{itemize}

The full proof of \Cref{thm:conv_tv_absorb} is in \Cref{app:thm:conv_tv_absorb_proof}. 
The overall idea is to compare Euler's unmasking events with the true probabilistic masking events viewed in reverse. By property of masked diffusion, these masking times are (a) independent over the dimensions, (b) independent of the data, and (c) after transformation $a = e^{-u}$ (where $u$ is the forward time), uniformly distributed over $(0,1)$. For the discretization points $\{t_k\}$, let $\Pi = (K_1,\dots,K_d)$, where $K_i = k$ if token $i$ is unmasked in the interval $(T-t_{k+1}, T-t_k]$, and let $q_\Pi$ be the output distribution in an Euler-style parallel update under $\Pi$ (with perfect estimators). Note that $q_\Pi$ reduces to $q_0$ when the generation sequence is entirely sequential. 

To start the proof, we first decompose the TV error into the discretization error, the finite-range error (including initialization and early-stopping errors), and the estimation error (\Cref{eq:euler-tv-decomposition-main3}). We subsequently show that these errors can be upper-bounded by $\calO(\kappa d \log S)$ (\Cref{eq:euler-new-loss-term1}), $de^{-T} + d\delta$ (\Cref{eq:euler-new-boundary}), and $e^\kappa \mathcal L_{\mathrm{TV}}(s)$ (\Cref{eq:euler-new-prediction-product,lem:euler-new-prediction}), respectively. 
Combining these bounds yields an $S$-dependent bound. In particular, the discretization error is upper-bounded by comparing the number of previously unmasked tokens under the reference schedule with the uniform rank in sequential unmasking (cf. \Cref{eq:euler-new-log-expansion,eq:euler-new-chain}). Then, for the $S$-independent bound, we note the fact that $q_\Pi=q_0$ whenever no two tokens are unmasked in the same step. Thus, we can upper-bound the dominant term in the discretization error (i.e., that due to simultaneous-unmasking) by the probability of the complementary event, which would yield a factor of $\calO(\kappa d(d-1))$ via a union bound (\Cref{eq:euler-new-collision}).
Together with the
same finite-range and estimation bounds, this eventually yields an
$\calO(\kappa d^2)$ discretization bound, which combines with
the $S$-dependent bound to give
$\calO(\kappa d\min\{d,\log S\})$.

\subsection{Convergence Lower Bound for the Euler Method}
\label{sec:lower-euler}

Given the strong performance guarantees of the Euler method, an important question is whether these upper bounds are tight, as this directly implies the fundamental theoretical limits of the algorithm. As follows, we provide a convergence lower bound under the absorbing-rate using the Euler sampler, assuming no estimation error.

\begin{theorem}\label{thm:lower_absorb}
Consider the exact-score Euler method for the absorbing-rate model, initialized from $\bm{\delta}_{\mask^d}$ and using the same adaptive grid as in \Cref{thm:conv_tv_absorb}, with $\delta$ a grid point.
Then, there exist universal constants $c,\kappa_0>0$ such that, for every $d\ge2$, $S\ge3$, $T\ge2$, $0<\delta\le1/2$, and $0<\kappa\le\kappa_0$, there exists a data distribution $q^*$ with no $\mask$ tokens such that
\begin{align*}
    \TV{q^*}{p_{\mathrm{Euler}}}
    \ge c\min\{\kappa d\min\{d,\log S\},1\}.
\end{align*}
As a result, for sufficiently small $\eps$, achieving
$\TV{q^*}{p_{\mathrm{Euler}}}\lesssim\sqrt\eps$ requires
\begin{align*}
    \kappa\lesssim\frac{\sqrt\eps}{d\min\{d,\log S\}},
    \qquad
    N=\Omega\left(\frac{d\min\{d,\log S\}}{\sqrt\eps}\right).
\end{align*}
\end{theorem}

\Cref{thm:lower_absorb} is the \textit{first} convergence lower bound for the Euler method under the absorbing rate. It shows that, on this adaptive grid, achieving $\sqrt{\eps}$-TV error in the worst case requires at least $\Omega(d\min\{d,\log S\}/\sqrt{\eps})$ steps, even with exact scores. 

\textbf{(Tight sampling complexity for Euler)} For the Euler method on this adaptive grid, together with \Cref{thm:conv_tv_absorb}, \Cref{thm:lower_absorb} establishes that our upper bound on the number of steps is optimal in its dependence on $d$, $S$, and $\eps$, up to logarithmic factors.

The full proof of \Cref{thm:lower_absorb} is provided in \Cref{app:thm:lower_absorb_proof}. 
The key idea is to construct $q^*$ using independent shared bits on a graph $G = (V,E)$ of token locations. In particular, $V = [d]$ is the set of token locations. For $E$, we first represent the token values in bits. Then, let an (undirected) edge $\{u,v\} \in E$ if and only if there exists a shared bit $B_{uv} \in \{0,1\}$ in tokens $x^u$ and $x^v$. We also select the total number of these shared bits such that $|E|\asymp d\min\{d,\log S\}$. Overall, the randomness of $x$ is jointly determined by these shared bits. See an example in \Cref{subsec:euler-lower-q-example}. 

To obtain a lower bound in TV, it suffices to lower-bound the probability of error events, which consist of those cases where the two copies of at least one shared bit disagree along the generation path. In case of shared bits $B_{uv}$, this means that the final corresponding bit values in $x^u$ and $x^v$ disagree for at least one edge $\{u,v\} \in E$. Indeed, conditional on no earlier disagreements, a new disagreement can arise only when both endpoints are unmasked in the same Euler step, in which case the independent coordinate updates produce a disagreement with probability $1/2$ per-bit. Also note that such a disagreement persists thereafter by properties of the absorbing rate. By analyzing simultaneous-unmasking events across the edges of $G$, we show that this probability is of order $\min\{\kappa |E|,1\}$. Since $|E|\asymp d\min\{d,\log S\}$, this yields the desired TV lower bound $\min\{\kappa d\min\{d,\log S\},1\}$.

\section{Convergence Guarantees for the FHS Algorithm}
\label{sec: convergence rate fhs}

In this section, we provide our convergence analysis of FHS (see \Cref{alg:first-hitting-sampling}). Our analysis will be performed under the following assumption on the score estimation error.

\begin{assumption}
\label{ass: fhs integrated Lse error}
Recall the definition of $\mathcal{L}_{SE}$ from \cref{eq:def_score_entropy}. 
Suppose that $s_t$ satisfies that
\begin{equation*}
    \int_0^{+\infty} \mathcal{L}_{SE}(s;t) \mathrm{d}t \le \varepsilon''_{\mathrm{
score}}.
\end{equation*}
\end{assumption}

In the studies using the integrated score entropy loss, the time integral is typically truncated at $T \asymp \log (d/\eps)$ \citep{chen2024uniformization,ren2025stoc-int}. We instead extend the upper limit to $+\infty$, for two reasons. First, if we suppose $\KL{q_0^i(\cdot\mid x^{UM})}{\mu_\theta^i(x,t)}\le C$
uniformly over $t\ge T$, states $x$ where $q_t(x) > 0$, and masked coordinates $i$,
where $C$ is a universal constant, then the above choice of $T$ yields 
(cf. \citet{liang2025absorb}): 
\[ \int_T^{+\infty}\mathcal{L}_{SE}(s;t) \mathrm{d}t \lesssim d \int_T^{+\infty} \frac{1}{e^t-1} \mathrm{d}t \lesssim \eps. \]
Therefore, extending the integration range to $+\infty$ changes the integrated loss by at most \(\mathcal O(\varepsilon)\), as the additional tail contributes at most $\calO(\eps)$. Second, we show in the following proposition that such integrated score entropy loss is equivalent to the standard training protocol (i.e., NELBO) in \cref{eq:def of nelbo} for masked diffusion models up to a fixed constant that depends only on the data distribution $q_0$.

\begin{proposition}
\label[proposition]{prop: equivalence between lse and nelbo}
Suppose that there are no mask tokens in the data, and $\mu_\theta$ assigns zero probability to mask tokens. Then, we have
\begin{equation*}
    \int_0^{+\infty} \mathcal{L}_{SE}(s;t)\, \mathrm{d}t =\mathbb{E}_{x_0 \sim q_0}[\mathcal{L}_\infty(x_0)]
    - \sum_{k = 1}^d \frac{1}{k} \mathbb{E}_{x_0 \sim q_0} \mathbb{E}_{\calM_k}
\sum_{l \in \calM_k}
H\left(q_0^l(\cdot|x_0^{\calM_k^c})\right),
\end{equation*}
where $\mathcal{L}_{SE}$ is defined in \cref{eq:def_score_entropy} and $\mathcal{L}_\infty$ in \cref{eq:def of nelbo}. Here, $\calM_k \subseteq [d]$ denotes a subset of exactly $k \in [d]$ masked tokens, drawn uniformly from all possible candidates, and
$H(\cdot)$ denotes the entropy. Given such an $\calM_k$, we also define $\calM_k^c:=[d] \setminus \calM_k$ and $x_0^{\calM_k^c}:=\{x_0^l\}_{l\in \calM_k^c}$.
\end{proposition}

\Cref{prop: equivalence between lse and nelbo} implies that the expected NELBO $\mathbb{E}_{x_0 \sim q_0}[\mathcal{L}_\infty(x_0)]$ is equivalent to the integrated $\mathcal{L}_{SE}$ up to a constant (i.e., constant w.r.t. $\mu_\theta$), thus establishing the effectiveness of the training process for MDMs \citep{sahoo2024mdlm}, which aims to minimize $\mathbb{E}_{x_0 \sim q_0}[\mathcal{L}_\infty(x_0)]$.
Indeed, our proof shows that 
\[\mathbb{E}_{x_0 \sim q_0}[\mathcal{L}_\infty(x_0)]
    - \sum_{k = 1}^d \frac{1}{k} \mathbb{E}_{x_0 \sim q_0} \mathbb{E}_{\calM_k}
\sum_{l \in \calM_k}
H\left(q_0^l(\cdot|x_0^{\calM_k^c})\right)\ge 0,\]
and the equality holds if and only if $q_0^l(\cdot|x_t^{UM}) = \mu_\theta^l(x_t,t)[\cdot]$ for almost every $t>0$, $q_t$-almost every $x_t$,
and every masked coordinate $l$, where $x_t^{UM}$ denotes the collection of unmasked coordinates of $x_t$. 
This implies that the right-most summation term
is an attainable lower bound in the training process,
which can be reached when the trained model is perfect.
The full proof of \Cref{prop: equivalence between lse and nelbo} is in \Cref{app:prop: equivalence between lse and nelbo}.

We now provide the error bound for FHS in the following theorem.

\begin{theorem}
\label{thm: error bound of fhs}
     Suppose that there are no mask tokens in the data, and $\mu_\theta$ assigns zero probability to mask tokens. Also suppose that \Cref{ass: fhs integrated Lse error} holds. 
     Then, 
\[
\textstyle \KL{q_0}{\mathbb{P}_{\mathrm{FHS}}} \le \varepsilon''_{\mathrm{score}}.
\]
Here $\mathbb{P}_{\mathrm{FHS}}$ denotes the output distribution induced by the FHS algorithm. 
\end{theorem}

\Cref{thm: error bound of fhs} shows that FHS can achieve $\varepsilon$-KL error in \textit{exactly} $d$ sampling steps. This is the \textit{first} result that shows that masked diffusion samplers could achieve vanishing convergence error (as the training error decreases) within the {\em finite} number of steps. 
Notably, the resulting error bound introduces no additional sources of error beyond score estimation error and exhibits no further dependence on either the data dimension $d$ or the vocabulary size $S$. This highlights the efficiency of the FHS algorithm compared to previous diffusion samplers, which generally require substantially more than $d$ steps to reach the same level of accuracy. More generally, our result is a further showcase of the advantage of choosing masked diffusion models over their uniform-rate counterpart.

The full proof of \Cref{thm: error bound of fhs} is provided in \Cref{app:thm: error bound of fhs}. We highlight our main idea as follows. We start by showing in \Cref{prop: fhs exact} the decoupling property of MDMs: the transition time $\tau$ and transition index $i$ can be sampled exactly from true reverse process, and remain independent of the estimation error introduced by $\mu_\theta$.
We then use this decoupling property to decompose the path-wise error into the error arising from predicting $x_0^i$ via $\mu_\theta$ in each sampling step. Finally, we transform the path-wise error to resulting sampling error on $x_0$ using data processing inequality.

Continuing the error upper bound of FHS in \Cref{thm: error bound of fhs}, we now investigate whether the upper bound is tight. As follows,  we derive a lower bound in \Cref{prop: fhs upper bound is tight}. Specifically, since \Cref{ass: fhs integrated Lse error} allows any form of $\mu_\theta$ resulting from the training procedure, our chosen $\mu_\theta$, as one particular instance, provides an error lower bound over the class of all such estimators. 

\begin{theorem}
\label{prop: fhs upper bound is tight}
    For $S\ge3$, there exists a pair $(q_0, \mu_\theta)$ satisfying \Cref{ass: fhs integrated Lse error} such that
    \[ \KL{q_0}{\mathbb{P}_{\mathrm{FHS}}} \ge \varepsilon''_{\mathrm{score}}.\]
    This lower bound exactly matches the upper bound in \Cref{thm: error bound of fhs}, showing that the KL error bound is tight in the worst case, including its leading constant.
\end{theorem}

\Cref{prop: fhs upper bound is tight} is the \textit{first} convergence lower bound for the FHS algorithm. 

\textbf{(Tight convergence bound for FHS)} For FHS, together with \Cref{thm: error bound of fhs}, \Cref{prop: fhs upper bound is tight} establishes that our upper bound on the KL error is optimal in its dependence on the integrated score-entropy error $\varepsilon''_{\mathrm{score}}$, including the leading constant.

The full proof is provided in \Cref{app:prop: fhs upper bound is tight}. Different from the proof of \Cref{thm:lower_absorb}, we construct both a target distribution $q_0$ \textit{and a special form of the estimated predictive model} $\mu_\theta$. 
Indeed, such a construction corresponds to a worst-case pair $(q_0,\mu_\theta)$, where $\mu_\theta$ can possibly be obtained during the training process under \Cref{ass: fhs integrated Lse error}. Such a construction allows us to explicitly characterize both the score estimation error and the final KL convergence error.

\section{Conclusion}

In this work, we presented a novel direct TV analysis for discrete diffusion models that yields tighter parameter dependencies than prior KL-based results. 
Our framework improves convergence rates, removes the need for bounded score assumptions and surrogate initialization, and establishes the first tight lower bounds in TV. 
Further, 
we showed that FHS incurs no sampling error beyond that arising from score estimation itself, without additional dependence on the system parameters. We then showed that this error bound is tight via a lower bound by a worst-case construction.
Looking ahead, an interesting direction is to 
investigate high-order variants of FHS for further acceleration.

\acks{This work has been supported in part by the Army Research Laboratory under Cooperative Agreement Number W911NF-24-2-0205 and W911NF-23-2-0225, by the U.S. National Science Foundation under the grants: NSF AI Institute (AI-EDGE) 2112471, CNS-2312836, CNS-2225561, CNS-2239677, and MFAI-2601957, and by the Office of Naval Research under grant N00014-24-1-2729. The views and conclusions contained in this document are those of the authors and should not be interpreted as representing the official policies, either expressed or implied, of the Army Research Laboratory or the U.S. Government. The U.S. Government is authorized to reproduce and distribute reprints for Government purposes notwithstanding any copyright notation herein.}

\newpage

\appendix

\begin{center}
    \huge \textbf{Appendix}
\end{center}

\crefalias{section}{appendix} 

\startcontents[section]
{
\hypersetup{linkcolor=blue}
\printcontents[section]{l}{1}{\setcounter{tocdepth}{2}}
}

\section{Full Related Works}
\label{app:works}

\textbf{Empirical Studies on Discrete Diffusion Models.} Due to the large portion of works, we will only include those early works and those most relevant works recently. Compared to continuous-space diffusion models, discrete-space diffusion models are emerging as strong alternatives for generative tasks involving symbolic data \citep{lou2024entropy,campbell2022discrete} (see also surveys: \citep{diffusion-survey-graph,diffusion-survey-drug-design}). The continuous-time formulation of discrete diffusion was first introduced in \citep{campbell2022discrete}. More recently, \citep{lou2024entropy} proposed the score-entropy estimation loss and demonstrated empirical success in text generation. They further introduced Tweedie $\tau$-leaping, a new sampler derived from an approximation of Tweedie’s formula. For per-step updates, most empirical implementations adopt categorical sampling, which has shown strong practical performance.

More recently, an increasing number of empirical studies have focused on masked diffusion models, which is often found to deliver superior performance in text and image applications. \citep{shi2024simplified} simplified the variational training objective into a weighted cross-entropy integral and introduced a state-dependent masking schedule to dynamically adjust rates for better generation quality. Building on this, \citep{ou2025absorb} reparameterized the concrete score as a time-dependent scalar and a conditional distribution, leading to the Reparameterized Absorbing Discrete Diffusion (RADD) model for efficient training and sampling. Similarly, \citep{sahoo2024mdlm} exploited the absorbing state structure to derive a tighter ELBO through Rao-Blackwellization and proposed a semi-autoregressive decoding strategy for flexible sequence generation. Building on this line of work, \citep{zheng2025fhs} introduced the First-Hitting Sampler, which exactly samples from the continuous-time diffusion process under the assumption of a perfectly learned score.  More recently, \citep{nie2025llada} developed a pre-trained model for masked diffusion models.

\textbf{Theory on Uniform-Rate Discrete Diffusion Models.} In contrast to the extensive theoretical literature on continuous diffusion models, rigorous results for discrete diffusion remain relatively limited. An early contribution is \citet{campbell2022discrete}, which analyzed $\tau$-leaping under the total variation metric but relied on strong estimation assumptions and exhibited unfavorable parameter dependencies. Subsequent works shifted focus to controlling the score–entropy estimation error. In particular, \citet{chen2024uniformization} established convergence guarantees for the uniformization sampler on the $d$-dimensional hypercube, later extended to general discrete spaces $[S]^d$ in \citet{ren2025stoc-int}.
For deterministic step-size methods, \citet{zhang2025conv-disc,conforti2025markov} assumed access to an exact per-step solver, while \citet{dmitriev2026discrete,ren2025stoc-int} analyzed the more practical $\tau$-leaping sampler. Among these works, \citet{zhang2025conv-disc} required control of the score-entropy loss along the entire continuous sampling trajectory, whereas \citet{dmitriev2026discrete,conforti2025markov,ren2025stoc-int} imposed this requirement only on the discrete time grid. More recently, \citet{liang2025sampler} developed an analysis for the Euler method and Tweedie $\tau$-leaping \citep{lou2024entropy}, which avoids the Girsanov change-of-measure technique and yields improved parameter dependencies. This also directly motivates the present work.
Notably, all of these analyses \citep{chen2024uniformization,dmitriev2026discrete,zhang2025conv-disc,conforti2025markov,ren2025stoc-int,liang2025sampler} were conducted in terms of KL divergence, which leads to suboptimal bounds when converted to total variation distance. Beyond convergence analysis, \citet{rojas2026theory-improved-guide} proposed a theoretically motivated improvement to discrete classifier-free guidance (CFG), and \citet{huang2025quantized} extended discrete diffusion techniques to continuous data distributions via quantization.

\textbf{Theory on Masked (Absorbing-rate) Discrete Diffusion Models.}
On the theoretical side, \citet{liang2025absorb} was the first to establish convergence guarantees for masked diffusion models, showing improved dimensional dependence compared to uniform-rate diffusion. \citet{huang2025mask} studied faster convergence rates using the Euler method and a specialized uniformization scheme, the Mask-Aware Truncated Uniformization (MATU), which eliminates the need for a bounded score assumption. In parallel, \citet{conforti2025discrete} derived improved dimension-dependent rates for the DMPM sampler, also without the need of a bounded score. Recently, \citet{dmitriev2026discrete} derived a convergence rate of $\calO(\mathcal{D}/\epsilon)$ under the KL divergence for the modified truncated $\tau$-leaping sampler, where $\mathcal{D}\le d\log S$ is a distribution-dependent quantity, improving previous bounds in the dependence on $S$ and potentially $d$.
A complementary line of work studies masked diffusion sampling from an information-theoretic perspective, where a time-agnostic predictor is used to unmask a prescribed number of uniformly chosen tokens per step. \citet{li2025mask} upper-bounded the KL error under a mask-size schedule by a mutual-information term plus an excess training loss. Subsequently, \citet{chen2025mask} and \citet{lavenant2025error} characterized the sampling error exactly for any unmasking schedule, and \citet{chen2025mask} further showed that the number of iterations cannot be reduced to $o(d)$ uniformly over all target distributions. 
These works use different estimation assumptions and are based on different sampling procedures from ours, and thus are not directly comparable with our results.
Beyond the analysis of unconditional models, \citet{rojas2026theory-improved-guide,he2026guidance} examined classifier-free guidance in the context of masked diffusion models.

\section{List of Notations}\label{app:notations}

We write $\ind{x=y}$ as a function of $x$ and $y$ which equals 1 only if $x=y$. For $i=1,\dots,d$, 
we write $e_i$ is a vector where only the $i$-th element is 1 and other elements are $0$'s, and we write $\bm{\delta}_i$ as the distribution of a singleton whose p.m.f. is $e_i$. 
For a positive integer $S$, $[S] := \{1,\dots,S\}$. Write $\bm{1}_S$ as a vector of length $S$ that contains all 1's, and $I_S$ as an identity matrix of size $S \times S$.
Write $m(x)$ to denote the number of $\mask$ states in the vector $x$.

\section{Proof of \texorpdfstring{Proposition~\ref{prop:est-relax}}{Proposition 1}}
\label{app:prop:est-relax}

Throughout the proof we assume that there are no mask tokens in the data, and that the score estimator is parameterized as in \cref{eq:equivalence-between-score-and-mu} such that $\mu_\theta$ assigns zero probability to masks. 

Fix a forward time $u > 0$.
For a state $x$ and a masked coordinate $i$, define the true and estimated clean-token conditionals as
\[
    p_{i,x}(a):=q_0^i(a\mid x^{UM}),
    \qquad
    \widehat p_{i,x,u}(a):=\mu_\theta^i(x,u)[a],
    \qquad \forall a \neq \mask.
\]
Since there are no mask tokens in the data, both $p_{i,x}$ and $\widehat p_{i,x,u}$ are probability distributions over the non-mask vocabulary. Indeed, from \eqref{eq:equivalence-between-score-and-mu} and \eqref{eq: equivalence-between-true-score-and-clean-distribution}, we have
\[
    \frac{q_u(x^{-i}\oplus_i a)}{q_u(x)}
    = \frac{e^{-u}}{1-e^{-u}} p_{i,x}(a),
    \qquad
    s_u(x^{-i}\oplus_i a,x)
    = \frac{e^{-u}}{1-e^{-u}} \widehat p_{i,x,u}(a).
\]
Note that here we have kept the potential time dependence for the posterior estimator.
Therefore, writing $\mathcal{M}(x):=\{i:x^i=\mask\}$, the TV estimation loss at time $u$ is equivalent to
\begin{align*}
    \calL_{TV}(s;u) &= \E_{x_{u} \sim q_{u}} \sbrc{ \sum_{y: y \neq x_{u}} R_{u}(y,x_{u}) \abs{ s_{u}(y,x_{u}) - \frac{q_{u}(y)}{q_{u}(x_u)} } } \\
    &= \E_{x_{u} \sim q_{u}} \sbrc{ \sum_{i \in \calM(x_u)} \sum_{a \neq \mask} \abs{ s_{u}(x^{-i}\oplus_i a,x_{u}) - \frac{q_{u}(x^{-i}\oplus_i a)}{q_{u}(x_u)} } } \\
    &= \frac{e^{-u}}{1-e^{-u}} \E_{x_u \sim q_u} \sum_{i \in \calM(x_u)} \sum_{a \neq \mask} \abs{ p_{i,x_u} (a) - \widehat p_{i,x_u,u}(a) } \\
    &= \frac{e^{-u}}{1-e^{-u}} \E_{x_u \sim q_u} \sum_{i \in \calM(x_u)} \norm{ p_{i,x_u}-\widehat p_{i,x_u,u} }_1.
\end{align*}
On the other hand, the score-entropy loss at time $u$ is
\begin{align*}
    \calL_{SE}(s;u) &= \E_{x_u \sim q_u} \sum_{y:y\neq x_u} R_{u}(y,x_u)
    \brc{s_u(y,x) - \frac{q_u(y)}{q_u(x)} -\frac{q_u(y)}{q_u(x)}\log\frac{s_u(y,x)}{q_u(y) / q_u(x)}} \\
    &= \frac{e^{-u}}{1-e^{-u}} \E_{x_u \sim q_u} \sum_{i\in \calM(x_u)} \sum_{a \neq \mask} \brc{\widehat p_{i,x_u,u}(a) - p_{i,x_u}(a) -p_{i,x_u}(a) \log\frac{\widehat p_{i,x_u,u}(a)}{p_{i,x_u}(a)}} \\
    &= \frac{e^{-u}}{1-e^{-u}} \E_{x_u \sim q_u} \sum_{i\in \calM(x_u)} \sum_{a \neq \mask} \brc{p_{i,x_u}(a) \log\frac{p_{i,x_u}(a)}{\widehat p_{i,x_u,u}(a)}} \\
    &= \frac{e^{-u}}{1-e^{-u}} \E_{x_u \sim q_u} \sum_{i\in \calM(x_u)} \KL{p_{i,x_u}}{\widehat p_{i,x_u,u}}.
\end{align*}
Now, for each $x_u$ and $i$, by Pinsker's inequality we have
\[ \norm{ p_{i,x_u}-\widehat p_{i,x_u,u} }_1 \leq \sqrt{2 \KL{p_{i,x_u}}{\widehat p_{i,x_u,u}} }. \]
Thus, 
\begin{align*}
    \calL_{TV}(s;u) &= \frac{e^{-u}}{1-e^{-u}} \E_{x_u \sim q_u} \sum_{i \in \calM(x_u)} \norm{ p_{i,x_u}-\widehat p_{i,x_u,u} }_1 \\
    &\leq \frac{e^{-u}}{1-e^{-u}} \sqrt{2} \E_{x_u \sim q_u} \sum_{i \in \calM(x_u)} \sqrt{\KL{p_{i,x_u}}{\widehat p_{i,x_u,u}} } \\
    &\stackrel{(i)}{\leq} \frac{e^{-u}}{1-e^{-u}} \sqrt{2 d} \E_{x_u \sim q_u} \sqrt{\sum_{i \in \calM(x_u)} \KL{p_{i,x_u}}{\widehat p_{i,x_u,u}} } \\
    &\stackrel{(ii)}{\leq} \sqrt{2d \frac{e^{-u}}{1-e^{-u}} \calL_{SE}(s;u)} \\
    &= \sqrt{\frac{2d}{e^u-1} \calL_{SE}(s;u)},
\end{align*}
where $(i)$ follows from Cauchy-Schwarz inequality, and $(ii)$ follows from Jensen's inequality. 

Now, let $u_k:=T-t_k$. It follows that
\begin{align*}
    \mathcal{L}_{TV}(s)
    &\leq
    \sqrt{2d\eps'_{\mathrm{score}}}
    \sum_{k=0}^{N-1}
    \frac{t_{k+1}-t_k}{\sqrt{e^{u_k}-1}}\\
    &\stackrel{(iii)}{\leq}
    \sqrt{2d\eps'_{\mathrm{score}}}
    \int_{\delta}^{T}\frac{\d u}{\sqrt{e^u-1}}\\
    &\leq
    \sqrt{2d\eps'_{\mathrm{score}}}
    \int_{0}^{\infty}\frac{\d u}{\sqrt{e^u-1}}
    =\pi\sqrt{2d\eps'_{\mathrm{score}}},
\end{align*}
where $(iii)$ follows because $u\mapsto(e^u-1)^{-1/2}$ is decreasing. The proof is now complete.

\section{Proof of \texorpdfstring{\Cref{thm:conv_tv_absorb}}{Theorem 2}}
\label{app:thm:conv_tv_absorb_proof}

In order to analyze the masked diffusion model, our central idea is to
compare Euler's unmasking events with the true probabilistic masking
events viewed in reverse. We first convert the masking times into a
probability scale and define a reference unmasking schedule in \Cref{sec:euler-new-grid}. Building on the grid, we define an auxiliary distribution $q_\Pi$ and introduce an important property for the Euler sampler that relates the Euler marginal output with $q_\Pi$ in \Cref{subsec:euler}.
Then, starting from \Cref{sec:euler-new-decomposition}, for the Euler sampler, we decompose the TV error correspondingly. We then
bound the resulting terms in
\Cref{sec:euler-new-boundary-prediction,sec:euler-new-terms}, and combine them into an $S$-dependent bound in \Cref{sec:euler-new-combine}.
Finally, drawing on the same elements, \Cref{sec:euler-new-vocabulary-free} gives an $S$-independent bound and the final combined sampling guarantee.

Throughout the proof, we assume the score parameterization in \cref{eq:equivalence-between-score-and-mu}, that the clean distribution $q_0$ has no $\mask$ tokens, and that $d\ge2$, $S\ge3$, $T\ge2$, and $0<\delta\le1/2$. We focus on the case where $d\kappa\le c$ with $c>0$ being a
sufficiently small universal constant. Otherwise, the claimed bound trivially holds since TV is no more than 1.

\subsection{Setup: transformed grid and unmasking schedule}
\label{sec:euler-new-grid}

By the independent forward evolution in \eqref{eq:forward-factorize}
and the unit absorbing rate in \eqref{def:rate_absorb}, a clean
coordinate remains visible at forward time $u$ with probability
$\alpha_u=e^{-u}$. A useful result under the absorbing-rate is that we can explicitly write the forward conditional probability as
\begin{multline}
\label{eq:absorb-fwd-cond-prob}
    q_{u|0}(x|x_0) = (1-e^{-u})^{\sum_i \ind{x^i = \mask,x_0^i \neq \mask}} \\
    \cdot (e^{-u})^{\sum_i \ind{x^i = x_0^i \neq \mask}} \cdot
    0^{\sum_i \ind{x^i \neq \mask, x^i \neq x_0^i}}.
\end{multline}
Here and as follows, we distinguish $u$ as the forward time, and preserve $t$ as the reverse time. On the finite grid $u_k \in [\delta, T]$, we have $t_k = T - u_k$. Note that $u_0 = T$, $u_{N} = \delta$, and $u_k-u_{k+1}=\kappa \min\{1,u_k\}$.
Also, by the unit rate in \eqref{def:rate_absorb}, if we let $E_i$ be the forward time when coordinate $i$ becomes masked (or unmasked, if seen in the reversed direction), $E_i$ is a unit-rate exponential random variable, \textit{independent of $x_0$} \citep{norris1997markov-chain}. \textit{In other words, we may sample the clean sequence and its masking schedule independently.} Reading the same forward trajectory backward turns masking events into unmasking events. Thus, the same variables $E_i$, still expressed in forward time, also specify the reference unmasking schedule used in our analysis.

Now, define the transformed grid
\begin{equation}\label{eq:euler-new-grid}
 a_k=e^{-u_k},\qquad w_k=a_{k+1}-a_k.
\end{equation}
Currently, the finite grid covers only $[e^{-T},e^{-\delta}]$ on this scale. In this proof, we further extend this finite grid into a countably infinite one by extending $k$ to $(-\infty,\infty)$. Specifically, we have $a_k$ to be increasing in $k$, that $a_k \to 0$ and $u_k \to \infty$ as $k \to -\infty$, and that $a_k \to 1$ and $u_k \to 0$ as $k \to \infty$. With the extension, we also have $\bigcup_{k\in\mathbb Z}[a_k,a_{k+1})=(0,1)$ and $\sum_{k\in\mathbb Z}w_k=1$.

Intuitively, this extension assigns every coordinate $i$ an unmasking step $k$. Unmasking events outside
the actual interval $[\delta,T]$ will then become part of the initialization and early-stopping error, as we will see later.

Write $U_i=e^{-E_i}$. In the forward process, for $v\in(0,1)$,
\[
 \mathbb P(U_i\le v)=\mathbb P(E_i\ge-\log v)=v,
\]
which implies that the $U_i$ are independent uniform variables on $(0,1)$.
Define
\begin{equation}\label{eq:euler-new-schedule}
 K_i=k \iff U_i\in[a_k,a_{k+1}),
 \quad C_i=\{j:K_j<K_i\}.
\end{equation}
In other words, for each dimension (or token location) $i$, $C_i$ is the set of all visible locations at the start of the transition interval represented by $K_i$. By the MDM property, $K_i$ is well-defined since the masking/unmasking event is unique for each token. Let $\Pi=(K_1,\ldots,K_d)$ be such a (random) (un)masking schedule, and denote its law $W$. In particular, by independence over the dimension, $W(\Pi=(k_1,\ldots,k_d))=\prod_i w_{k_i}$.

Before we proceed, we collect some properties for the mesh points and the Euler sampler, which will be used below to control the distortion of unmasking ranks and the probability of simultaneous unmasking events.

\begin{lemma}\label{lem:euler-new-mesh}
Recall the transformed grid in \eqref{eq:euler-new-grid}. Suppose that $0<\kappa\le1/2$.
Then there exists a universal constant $C$ such that, for every $k\in\mathbb Z$,
\begin{equation*}
    |w_k-w_{k-1}|\le C\kappa w_k,\qquad
 w_k\le C\kappa r(1-r)\le\frac{C\kappa}{4},
 \qquad r\in[a_k,a_{k+1}].
\end{equation*}
\end{lemma}

\begin{proof}
See \Cref{app:lem:euler-new-mesh-proof}.
\end{proof}

\subsection{Setup: Euler sampler}
\label{subsec:euler}

For the sampler, recall the standard Euler transition in \eqref{eq:def_euler}, with
the estimated reverse rate in \eqref{eq:def_est_reverse_rate} and
the score parameterization in \eqref{eq:equivalence-between-score-and-mu}.
Write $c_u=(e^u-1)^{-1}$ and $\eta_k=u_k-u_{k+1}$.
For a masked coordinate $x_k^i=\mask$ and $a\ne\mask$, Euler yields
\begin{align}
 p_{k+1|k}(x_{k+1}^i=a\mid x_k)
 &=\eta_k\hat R_{t_k}(x_k,x_k^{-i}\oplus_i a) = \eta_kc_{u_k} \mu_\theta^i(x_k,u_k)[a].
 \label{eq:euler-new-euler-transition}
\end{align}
The coordinate stays masked with probability
$1-\eta_kc_{u_k}\sum_{a\ne\mask}\mu_\theta^i(x_k,u_k)[a]$;
already unmasked coordinates remain unchanged.
Equivalently, each masked coordinate independently attempts a
prediction with probability $\eta_kc_{u_k}$, then draws from
$\mu_\theta^i(x_k,u_k)$. A mask-valued prediction leaves the
coordinate masked and permits later attempts. This representation
is valid because $\eta_kc_{u_k}\le\kappa\le1/2$.

In order to distinguish the effect of simultaneous unmasking, let us introduce an auxiliary distribution $q_\Pi(x)$ based on a given unmasking schedule $\Pi=(K_1,\ldots,K_d)$. For every clean sequence $x\in([S]\setminus\{\mask\})^d$, define
\begin{equation}\label{eq:euler-new-oracle}
 q_\Pi(x) := \prod_{i=1}^d q_0^i(x^i \mid x^{C_i}),
\end{equation}
where we recall $C_i$ from \eqref{eq:euler-new-schedule}.
Here $q_0^i(x^i \mid x^{C_i})$ is the clean version of $\mu_\theta^i(x,u)[x^i]$ where the only visible tokens in $x$ are at $C_i$.
If every coordinate is unmasked at a different step, we immediately obtain $q_\Pi(x)=q_0(x)$. The only difference happens when several coordinates are unmasked at the same step. Yet, for vanishing $\kappa$, the probability of this event will become rare.

The following lemma establishes an important property for the Euler method.

\begin{lemma}\label{lem:euler-new-domination}
Suppose that $\Pi = (K_1,\ldots,K_d) \sim W$. Define
\begin{equation}
\label{eq:euler-new-domination-G-def}
    G:=\{K_i\in\{0,\ldots,N-1\}\text{ for every }i\}
\end{equation} 
as the ``good'' events where all the unmasking happens within the finite interval $u_k \in [\delta,T]$. Under the assumptions of \Cref{thm:conv_tv_absorb}, initialize Euler
at $\bm{\delta}_{\mask^d}$ and use the designed finite grid $\{u_k\}_{k=0}^N$.
Define, $\forall x\in\operatorname{supp}(q_0)$,
\begin{equation}\label{eq:def-euler-new-gamma}
 \gamma_\Pi(x)=\ind{G} \prod_{i=1}^d
 \min\left\{1,
 \frac{\mu_\theta^i(x^{C_i}\oplus\mask,u_{K_i})[x^i]}
 {q_0^i(x^i\mid x^{C_i})}\right\}.
\end{equation}
Here $x^C\oplus\mask$ denotes the vector equal to $x$ on $C$ and masked elsewhere. 
Then $0\le\gamma_\Pi(x)\le1$, and the Euler output distribution satisfies, $\forall x \in \mathrm{supp}(q_0)$,
\begin{align*}
p_{\mathrm{Euler}}(x)
 &\ge e^{-d\kappa}\E_{\Pi\sim W}
 \sbrc{q_\Pi(x)\gamma_\Pi(x)},
\end{align*}
where $q_\Pi$ is defined in \eqref{eq:euler-new-oracle}.
\end{lemma}

\begin{proof}
See \Cref{app:lem:euler-new-domination-proof}.
\end{proof}

We now provide some intuition for \Cref{lem:euler-new-domination}. Among the contributors on the right-hand side, $\gamma_\Pi$ captures the effect of imperfect estimation and finite-range, and $\gamma_\Pi = 1$ if the estimator is perfect and $\Pi$ entirely lies within the chosen finite range. 
Also, the pre-factor $e^{-d\kappa}$ arises because Euler updates have a slightly different time-scaling from the true one. Indeed, the true conditional probability of unmasking at step $k$ under the reference schedule is
\begin{align*}
 \mathbb P_{\Pi\sim W}(K_i=k\mid K_i\ge k) &= \mathbb P_{\Pi\sim W}(U_i \in [a_k,a_{k+1}) \mid U_i \ge a_k) \\
 &=\frac{a_{k+1}-a_k}{1-a_k} \\
 &=\frac{e^{\eta_k}-1}{e^{u_k}-1} \\
 &=(e^{\eta_k}-1)c_{u_k}.
\end{align*}
Here note that the time-scaling is $(e^{\eta_k}-1)$, which is different from $\eta_k$ used in the Euler update in \eqref{eq:euler-new-euler-transition}.
Since $\eta_k/(e^{\eta_k}-1)\ge e^{-\eta_k}\ge e^{-\kappa}$, Euler's unmasking probability is at least $e^{-\kappa}$ times the reference probability, while its probability of remaining masked is no smaller.
Thus, along any prescribed schedule, the timing comparison loses at most a factor $e^{-\kappa}$ per token, yielding $e^{-d\kappa}$.
Finally, $q_\Pi$ is the output distribution obtained by following the schedule $\Pi$ and independently sampling each batch of tokens from their true conditional distributions given the previously unmasked tokens.
Its expectation over $\Pi\sim W$ is the marginal output of this reference sampler, which \Cref{lem:euler-new-domination} compares with the actual Euler output.

\subsection{Step 1: TV error decomposition using the Euler sampler}
\label{sec:euler-new-decomposition}

In this subsection, we decompose the total error into the simultaneous unmasking error (discretization error), the finite-range error (due to imperfect initialization and early-stopping), and score estimation error.

Write $(x)^+ = \max\{0,x\}$. Then, since $q_0$ is supported on clean sequences, we have
\begin{align}
    &\TV{q_0}{p_{\mathrm{Euler}}} \nonumber \\
    &= \sum_{x} \brc{q_0(x) - p_{\mathrm{Euler}}(x)}^+ = \E_{x \sim q_0}\brc{1-\frac{p_{\mathrm{Euler}}(x)}{q_0(x)}}^+ \nonumber\\
    &\stackrel{(i)}{\leq} \E_{x \sim q_0} \brc{1-\E_{\Pi \sim W} \left[
    e^{-d\kappa}\gamma_\Pi(x)\frac{q_\Pi(x)}{q_0(x)}
    \right]}^+ \nonumber\\
    &= \E_{x \sim q_0} \Bigg( \E_{\Pi \sim W}\sbrc{\log\frac{q_0(x)}{q_\Pi(x)}} + 1-\E_{\Pi \sim W} \left[\frac{q_\Pi(x)}{q_0(x)}+\log\frac{q_0(x)}{q_\Pi(x)}\right] \nonumber\\
    &\qquad +\E_{\Pi \sim W} \left[ \bigl(1-e^{-d\kappa}\gamma_\Pi(x)\bigr) \frac{q_\Pi(x)}{q_0(x)} \right] \Bigg)^+ \nonumber\\
    &\stackrel{(ii)}{\leq} \E_{x \sim q_0} \Bigg( \E_{\Pi \sim W} \sbrc{\log\frac{q_0(x)}{q_\Pi(x)}} + \E_{\Pi \sim W} \left[ \frac{q_\Pi(x)}{q_0(x)}-1 -\log\frac{q_\Pi(x)}{q_0(x)} \right] \nonumber \\
    &\qquad + 2 \E_{\Pi \sim W} \left[1-e^{-d\kappa}\gamma_\Pi(x)\right] \Bigg)^+ \nonumber\\
    &\stackrel{(iii)}{\leq} \E_{x \sim q_0} \brc{\E_{\Pi \sim W} \sbrc{ \log\frac{q_0(x)}{q_\Pi(x)}} }^+ + \E_{x \sim q_0, \Pi \sim W} \sbrc{\frac{q_\Pi(x)}{q_0(x)}-1-\log \frac{q_\Pi(x)}{q_0(x)}} \nonumber\\
    &\qquad + 2 \E_{x \sim q_0, \Pi \sim W} \sbrc{1-e^{-d\kappa}\gamma_\Pi(x)}, \label{eq:euler-tv-decomposition-main}
\end{align}
where $(i)$ follows from \Cref{lem:euler-new-domination} and note that
\[ p_{\mathrm{Euler}}(x)\ge e^{-d\kappa} \E_{\Pi \sim W} \left[q_\Pi(x)\gamma_\Pi(x)\right] \iff 1 - \frac{p_{\mathrm{Euler}}(x)}{q_0(x)} \le 1-\E_{\Pi \sim W} \left[ e^{-d\kappa} \gamma_\Pi(x)\frac{q_\Pi(x)}{q_0(x)} \right], \]
$(iii)$ follows because both terms in the expectation are non-negative, and $(ii)$ is because of the following.
For any $z>0$ and $\beta\in[0,1]$, we have $\beta z\le 2\beta+2(z-1-\log z)$.
Indeed, when $z \le 2$, the claim is obviously true since $\beta z\le2\beta$ and $z-1-\log z\ge0$.
Otherwise, when $z>2$, we can verify that $z-2\log z > 0$ and thus $\beta(z-2) \leq z-2 = 2(z-1)-z \leq 2(z-1-\log z)$. Then, $(ii)$ is established by replacing $\beta = 1-e^{-d\kappa}\gamma_\Pi(x)$ and $z = \frac{q_\Pi(x)}{q_0(x)}$.

We first focus on the last term of \eqref{eq:euler-tv-decomposition-main}. Since $\gamma_\Pi \in [0,1]$, 
\begin{align}
    \E_{x \sim q_0, \Pi \sim W} \sbrc{1-e^{-d\kappa}\gamma_\Pi(x)} &= 1-e^{-d\kappa}+e^{-d\kappa} \E_{x\sim q_0,\Pi\sim W}\sbrc{1-\gamma_\Pi(x)} \nonumber\\
    &\leq d \kappa + e^{-d\kappa} \E_{x\sim q_0,\Pi\sim W}\sbrc{1-\gamma_\Pi(x)}. \label{eq:euler-tv-decomposition-main-term3}
\end{align}
We next bound the second term in
\eqref{eq:euler-tv-decomposition-main}.
Since $q_\Pi$ is a probability distribution, we have
\begin{align}
    &\E_{x\sim q_0,\,\Pi\sim W}
    \sbrc{
        \frac{q_\Pi(x)}{q_0(x)}-1
        -\log\frac{q_\Pi(x)}{q_0(x)}
    } \nonumber\\
    &=
    \E_{\Pi\sim W}
    \sbrc{
        \sum_{x\in\operatorname{supp}(q_0)}q_\Pi(x)-1
    }
    +\E_{x \sim q_0, \Pi \sim W}
    \sbrc{\log\frac{q_0(x)}{q_\Pi(x)}} \nonumber\\
    &\stackrel{(iv)}{\leq}
    \E_{x \sim q_0, \Pi \sim W}
    \sbrc{\log\frac{q_0(x)}{q_\Pi(x)}} \nonumber\\
    &\le \E_{x \sim q_0} \brc{\E_{\Pi \sim W} \sbrc{ \log\frac{q_0(x)}{q_\Pi(x)}} }^+,
    \label{eq:euler-tv-decomposition-main-term2}
\end{align}
where $(iv)$ follows because $q_\Pi$ is a probability distribution, so
$\sum_{x\in\operatorname{supp}(q_0)}q_\Pi(x)\le1$.
Substituting \eqref{eq:euler-tv-decomposition-main-term2}
and \eqref{eq:euler-tv-decomposition-main-term3} into
\eqref{eq:euler-tv-decomposition-main}, we obtain
\begin{equation} \label{eq:euler-tv-decomposition-main2}
    \TV{q_0}{p} \leq 2 \E_{x\sim q_0} \brc{ \E_{\Pi \sim W} \sbrc{\log\frac{q_0(x)}{q_\Pi(x)}} }^+ + 2 \E_{x\sim q_0,\Pi\sim W} \sbrc{1-\gamma_\Pi(x)} + 2 d \kappa.
\end{equation}
Here notice that the first term purely depends on the error from discretization. To understand the second term here, recall the definition of ``good'' events in \eqref{eq:euler-new-domination-G-def}: $G:=\{K_i\in\{0,\ldots,N-1\}\text{ for every }i\}$, i.e., when all the masking times $K_i$ lie in the finite interval. Now, with the $\gamma_\Pi$ in \eqref{eq:def-euler-new-gamma}, the second term in \eqref{eq:euler-tv-decomposition-main2} becomes
\begin{align}
    &\E_{x\sim q_0,\Pi\sim W} \sbrc{1-\gamma_\Pi(x)} \nonumber \\
    &= \E_{x\sim q_0,\Pi\sim W} \sbrc{1-\ind{G} \prod_{i=1}^d \min\left\{1, \frac{\mu_\theta^i(x^{C_i}\oplus\mask,u_{K_i})[x^i]} {q_0^i(x^i\mid x^{C_i})}\right\}} \nonumber \\
    &= \E_{\Pi\sim W} \sbrc{1-\ind{G}} + \E_{x\sim q_0,\Pi\sim W} \sbrc{ \ind{G} \brc{1-\prod_{i=1}^d \min\left\{1, \frac{\mu_\theta^i(x^{C_i}\oplus\mask,u_{K_i})[x^i]} {q_0^i(x^i\mid x^{C_i})}\right\}} } \nonumber \\
    &\leq \mathbb P_{\Pi\sim W} (G^c) + \E_{x\sim q_0,\Pi\sim W} \sbrc{ \ind{G} \brc{1-\prod_{i=1}^d \min\left\{1, \frac{\mu_\theta^i(x^{C_i}\oplus\mask,u_{K_i})[x^i]} {q_0^i(x^i\mid x^{C_i})}\right\}} }.
    \label{eq:euler-tv-decomposition-main2-term2}
\end{align}
Here the first term in \eqref{eq:euler-tv-decomposition-main2-term2} captures the initialization and the early-stopping error (i.e., the error due to the finite sampling interval), whereas the second term in \eqref{eq:euler-tv-decomposition-main2-term2} arises from estimation.
Plugging \eqref{eq:euler-tv-decomposition-main2-term2} back into \eqref{eq:euler-tv-decomposition-main2}, we obtain an overall TV-error decomposition:
\begin{align} \label{eq:euler-tv-decomposition-main3}
    \TV{q_0}{p} &\leq 2 \E_{x\sim q_0} \brc{ \E_{\Pi \sim W} \sbrc{\log\frac{q_0(x)}{q_\Pi(x)}} }^+ + 2 \mathbb P_{\Pi\sim W} (G^c) \nonumber \\
    &+ 2 \E_{x\sim q_0,\Pi\sim W} \sbrc{ \ind{G} \brc{1-\prod_{i=1}^d \min\left\{1, \frac{\mu_\theta^i(x^{C_i}\oplus\mask,u_{K_i})[x^i]} {q_0^i(x^i\mid x^{C_i})}\right\}} } \nonumber \\
    &+ 2 d \kappa.
\end{align}

\subsection{Step 2: bounding the simultaneous-unmasking error}
\label{sec:euler-new-terms}

We first address the first error term in \eqref{eq:euler-tv-decomposition-main3}, namely $\E_{x\sim q_0} \brc{ \E_{\Pi\sim W}\sbrc{\log\frac{q_0(x)}{q_\Pi(x)}} }^+$. We first fix $x\in\mathrm{supp}(q_0)$, and we recall $q_{\Pi}$ from \eqref{eq:euler-new-oracle}. We can first decompose its inner-part as
\begin{align}
 \E_{\Pi\sim W}\sbrc{\log\frac{q_0(x)}{q_\Pi(x)}}
 &=\log q_0(x)
   +\sum_{i=1}^d\E_{\Pi\sim W}
      \sbrc{-\log q_0^i(x^i\mid x^{C_i})}.
 \label{eq:euler-new-log-expansion}
\end{align}
Now, note that on $\{|C_i|=n\}$, $C_i$ is uniformly distributed (over possible $n$-element subsets in $[d]\setminus\{i\}$).
Thus,
\begin{equation*}
    \sum_{i=1}^d\E_{\Pi\sim W}\sbrc{-\log q_0^i(x^i\mid x^{C_i})} = \sum_{i=1}^d \sum_{n=0}^{d-1} \frac{\mathbb P_{\Pi\sim W}(|C_i|=n)}{\binom{d-1}{n}} \sum_{\substack{C\subseteq[d]\setminus\{i\}\\|C|=n}}
   -\log q_0^i(x^i\mid x^C).
\end{equation*}
Compare this with the case of sequential decoding. Let $\sigma = (\sigma_1,\dots,\sigma_d)$ be a uniformly random permutation of the token indices to unmask. Obviously $q_0(x) = \prod_{r=1}^d q_0^{\sigma_r}\!\left(x^{\sigma_r}\mid x^{C_{\sigma_r}^\sigma}\right)$, where $C_{\sigma_r}^{\sigma} := \{\sigma_1,\ldots,\sigma_{r-1}\}$. With this setup, for any subset $C$ such that $|C|=n$, $\mathbb P(C_i^\sigma=C) = \frac{1}{d\binom{d-1}{n}}$, and thus
\begin{align}
 &\frac1d\sum_{i=1}^d\sum_{n=0}^{d-1}
 \frac{1}{\binom{d-1}{n}} \sum_{\substack{C\subseteq[d] \setminus\{i\}\\|C|=n}} - \log q_0^i(x^i\mid x^C) \nonumber \\
 &= \E_\sigma \sum_{i=1}^d -\log q_0^i(x^i\mid x^{C_i^\sigma}) \nonumber \\
 &=\E_\sigma [-\log q_0(x)] \nonumber \\
 &= -\log q_0(x).
 \label{eq:euler-new-chain}
\end{align}
Thus, the key difference is the factor $\mathbb P_{\Pi\sim W}(|C_i|=n)$, which is analyzed in the following \Cref{lem:euler-new-entropy}.

\begin{lemma}\label{lem:euler-new-entropy}
Suppose that $d\kappa\le c$ where $c > 0$ is a universal constant. We have, with some universal constant $C > 0$,
\[ \mathbb P_{\Pi\sim W}(|C_i|=n) \le\frac{1+C\kappa}{d}, \qquad n=0,\ldots,d-1. \]
\end{lemma}

\begin{proof}
See \Cref{app:lem:euler-new-entropy-proof}.
\end{proof}

Now, with \Cref{lem:euler-new-entropy}, we can continue from \eqref{eq:euler-new-log-expansion} and get
\begin{align*}
 &\E_{\Pi\sim W}\sbrc{\log\frac{q_0(x)}{q_\Pi(x)}}\\
 &\le\log q_0(x)
   +\frac{1+C\kappa}{d}
     \sum_{i=1}^d\sum_{n=0}^{d-1}\frac1{\binom{d-1}{n}}
     \sum_{\substack{C\subseteq[d]\setminus\{i\}\\|C|=n}}
     -\log q_0^i(x^i\mid x^C)\\
 &= \log q_0(x)+(1+C\kappa) (-\log q_0(x) )\\
 &= C \kappa (-\log q_0(x)).
\end{align*}
Note that $-\log q_0(x) \geq 0$. Thus, taking the positive part and the expectation over $q_0$, we finally get, with some universal constant $C > 0$,
\begin{align}
 \E_{x\sim q_0}
 \brc{\E_{\Pi\sim W}\sbrc{\log\frac{q_0(x)}{q_\Pi(x)}}}^+
 &\le C\kappa\,\E_{x\sim q_0}[-\log q_0(x)] \nonumber\\
 &=C\kappa H(q_0)
 \le C\kappa d\log S. \label{eq:euler-new-loss-term1}
\end{align}
Here $H(q_0)$ represents the entropy of $q_0$.

\subsection{Step 3: bounding boundary and estimation errors}
\label{sec:euler-new-boundary-prediction}

We now focus on the rest of the error terms in \eqref{eq:euler-tv-decomposition-main3}. 

First, for the second term in \eqref{eq:euler-tv-decomposition-main3}, since each $K_i$ in $\Pi$ corresponds to a uniformly distributed $U_i$, we have
\begin{align}\label{eq:euler-new-boundary}
 \mathbb P_{\Pi\sim W}(G^c) &= \mathbb P_{\Pi\sim W}\cbrc{\text{at least one $U_i$ lies outside of $[a_0,a_N]$}} \nonumber \\
 &\leq \sum_{i=1}^d \mathbb P_{\Pi\sim W}\cbrc{\text{$U_i$ lies outside of $[a_0,a_N]$}} \nonumber \\
 &= d a_0+d(1-a_N) \nonumber\\
 &= de^{-T}+d(1-e^{-\delta}) \nonumber\\
 &\le de^{-T}+d\delta.
\end{align}
Specifically, the two tails represent the error contributions from the initialization and early-stopping, respectively.

Also, for the third term in \eqref{eq:euler-tv-decomposition-main3}, we have
\begin{align}
    &\E_{x\sim q_0,\Pi\sim W} \sbrc{ \ind{G} \brc{1-\prod_{i=1}^d \min\left\{1, \frac{\mu_\theta^i(x^{C_i}\oplus\mask,u_{K_i})[x^i]} {q_0^i(x^i\mid x^{C_i})}\right\}} } \nonumber \\
    &\stackrel{(i)}{\leq} \sum_{i=1}^d \E_{x\sim q_0, \Pi\sim W} \sbrc{ \ind{G} \brc{ 1- \frac{\mu_\theta^i(x^{C_i}\oplus\mask, u_{K_i})[x^i]} {q_0^i(x^i\mid x^{C_i})} }^+ } \nonumber \\
    &\stackrel{(ii)}{\leq} \sum_{k=0}^{N-1}\sum_{i=1}^d \E_{x\sim q_0,\Pi\sim W} \sbrc{ \ind{K_i=k} \brc{ 1- \frac{\mu_\theta^i(x^{C_i}\oplus\mask,u_k)[x^i]} {q_0^i(x^i\mid x^{C_i})} }^+ }, \label{eq:euler-new-prediction-product}
\end{align}
where $(i)$ follows because $1-\prod_i r_i\le\sum_i(1-r_i)$ for $r_i\in[0,1]$ and $1-\min\{1,a\}=(1-a)^+$, and $(ii)$ follows because by definition $G$ captures the good events where all tokens are unmasked within the finite interval $u_k \in [\delta,T]$, which implies that there exists some $k \in \{0,\dots,N-1\}$ such that $K_i = k$.

The following lemma evaluates this expectation.

\begin{lemma}\label{lem:euler-new-prediction}
Write $z = x^{C_i}\oplus\mask$. Under the reference schedule in \eqref{eq:euler-new-schedule} and
the grid and score parameterization of \Cref{thm:conv_tv_absorb}, we have
\[ \sum_{k=0}^{N-1}\sum_{i=1}^d \E_{x\sim q_0,\,\Pi\sim W} \sbrc{ \ind{K_i=k} \brc{ 1- \frac{\mu_\theta^i(x^{C_i}\oplus\mask,u_k)[x^i]} {q_0^i(x^i\mid x^{C_i})} }^+ } \le e^\kappa \mathcal L_{\mathrm{TV}}(s). \]
\end{lemma}

\begin{proof}
See \Cref{app:lem:euler-new-prediction-proof}.
\end{proof}

Applying \Cref{lem:euler-new-prediction} to
\eqref{eq:euler-new-prediction-product} bounds the third term in
\eqref{eq:euler-tv-decomposition-main3} by
$e^\kappa\mathcal L_{\mathrm{TV}}(s)$.

\subsection{Step 4: combining the bounds and obtaining the sampling guarantee}
\label{sec:euler-new-combine}

We now return to the Euler bound \eqref{eq:euler-tv-decomposition-main3}.
Combining \eqref{eq:euler-tv-decomposition-main3}, \eqref{eq:euler-new-loss-term1}, \eqref{eq:euler-new-boundary}, and \Cref{lem:euler-new-prediction} gives, whenever $d\kappa\le c$,
\begin{align}
 \TV{q_0}{p_{\mathrm{Euler}}}
 &\le2C\kappa H(q_0)+2de^{-T}+2d\delta
      +2e^\kappa\mathcal L_{\mathrm{TV}}(s)+2d\kappa
      \nonumber\\
 &\le C'\brc{
     de^{-T}+d\delta+\mathcal L_{\mathrm{TV}}(s)
     +\kappa d\log S}.
 \label{eq:euler-new-entropy-tv}
\end{align}
Here we note that $H(q_0)\le d\log S$.
Under \Cref{ass:score_tv}, the loss term is at most
$\mathcal L_{\mathrm{TV}}(s) \lesssim \sqrt{\eps_{\mathrm{score}}}$.

\subsection{Step 5: vocabulary-independent bound}
\label{sec:euler-new-vocabulary-free}

Before we end the proof, there is also a direct way to use \Cref{lem:euler-new-domination} that yields an upper bound independent of $S$.
Let $J := \{K_i\ne K_j\text{ whenever }i\ne j\}$ on the complete
reference schedule. On $J$, exactly one token is unmasked at each occupied step,
so the ordinary chain rule gives $q_\Pi(x)=q_0(x)$.
Using this identity, the probability lower bound in
\Cref{lem:euler-new-domination} gives
\begin{align}
    p_{\mathrm{Euler}}(x) &\geq e^{-d\kappa}\E_{\Pi \sim W}[q_\Pi(x) \gamma_\Pi(x)] \nonumber\\
    &= e^{-d\kappa} \E_{\Pi \sim W}[\ind{J} q_\Pi(x) \gamma_\Pi(x)] + e^{-d\kappa} \E_{\Pi \sim W}[\ind{J^c} q_\Pi(x) \gamma_\Pi(x)] \nonumber \\
    &= e^{-d\kappa} q_0(x) \E_{\Pi \sim W}[\ind{J} \gamma_\Pi(x)] + e^{-d\kappa} \E_{\Pi \sim W}[\ind{J^c} q_\Pi(x) \gamma_\Pi(x)] \nonumber \\
    &\geq e^{-d\kappa}q_0(x)\E_{\Pi \sim W}[\ind{J} \gamma_\Pi(x)], \label{eq:euler-new-domination-new}
\end{align}
where $\gamma_\Pi$ is still as defined in \eqref{eq:def-euler-new-gamma}.
Then, $\TV{q_0}{p}$ can be upper-bounded in a different way from \eqref{eq:euler-tv-decomposition-main}:
\begin{align}
 \TV{q_0}{p_{\mathrm{Euler}}} &= \E_{x \sim q_0}\brc{1-\frac{p_{\mathrm{Euler}}(x)}{q_0(x)}}^+ \nonumber\\
 &\stackrel{(i)}{\leq} \E_{x \sim q_0} \brc{1-\E_{\Pi \sim W} \left[ \ind{J} e^{-d\kappa}\gamma_\Pi(x) \right]}^+ \nonumber \\
 &= \E_{x \sim q_0} \brc{1-\E_{\Pi \sim W} \left[ \ind{J} e^{-d\kappa}\gamma_\Pi(x) \right]} \nonumber \\
 &= 1 - e^{-d\kappa} \E_{x \sim q_0, \Pi \sim W}[\ind{J} \gamma_\Pi(x)] \nonumber\\
 &= 1 - e^{-d\kappa} \mathbb P_{\Pi \sim W}(J) + e^{-d\kappa} \E_{x \sim q_0, \Pi \sim W}[\ind{J} (1 - \gamma_\Pi(x))] \nonumber\\
 &= 1 - e^{-d\kappa} + e^{-d\kappa} \mathbb P_{\Pi \sim W}(J^c) + e^{-d\kappa} \E_{x \sim q_0, \Pi \sim W}[\ind{J} (1 - \gamma_\Pi(x))] \nonumber \\
 &\le d\kappa + \mathbb P_{\Pi \sim W}(J^c) + \E_{x\sim q_0, \Pi\sim W}\sbrc{1-\gamma_\Pi(x)},
 \label{eq:euler-tv-decomposition-main-S-free}
\end{align}
where $(i)$ follows from \eqref{eq:euler-new-domination-new}.
Also note that $\E_{\Pi\sim W}\sbrc{\ind{J}e^{-d\kappa}\gamma_\Pi(x)}, (1-\gamma_\Pi(x)) \in [0,1]$. 
Here the last term can be addressed in the same way as in \eqref{eq:euler-tv-decomposition-main2-term2} (with the corresponding bounds that follow). For the second term here, note that, for two distinct coordinates $i \neq j$, since $K_i$ and $K_j$ are independent (by independent forward masking), we have $\mathbb P(K_i=K_j)=\sum_k w_k^2$. Thus, the union bound yields
\begin{equation}\label{eq:euler-new-collision}
 \mathbb P_{\Pi \sim W}(J^c)
 \le\binom d2\sum_kw_k^2
 \le\binom d2\max_kw_k\sum_kw_k
 \le C\kappa d(d-1),
\end{equation}
where the last inequality follows from \Cref{lem:euler-new-mesh}.

Finally, continuing from \eqref{eq:euler-tv-decomposition-main-S-free}, we obtain
\begin{align}
 \TV{q_0}{p_{\mathrm{Euler}}}
 &\le de^{-T}+d\delta+e^\kappa\mathcal L_{\mathrm{TV}}(s)
      +d\kappa+C\kappa d(d-1).
 \label{eq:euler-new-vocabulary-free}
\end{align}
Both \eqref{eq:euler-new-entropy-tv} and
\eqref{eq:euler-new-vocabulary-free} bound the same Euler output
distribution. Taking the smaller of the two bounds gives, when
$d\kappa\le c$,
\begin{align}
 \TV{q_0}{p_{\mathrm{Euler}}}
 &\le C\brc{
     de^{-T}+d\delta+\mathcal L_{\mathrm{TV}}(s)
     +d\kappa+\kappa\min\{d(d-1),d\log S\}}
     \nonumber\\
 &\le C'\brc{
     de^{-T}+d\delta+\mathcal L_{\mathrm{TV}}(s)
     +\kappa d\min\{d,\log S\}}.
 \label{eq:euler-new-final}
\end{align}
Under \Cref{ass:score_tv}, we finally get
\begin{align*}
 \TV{q_0}{p_{\mathrm{Euler}}}
 &\lesssim de^{-T}+d\delta+\sqrt{\eps_{\mathrm{score}}}
           +\kappa d\min\{d,\log S\}.
\end{align*}

To obtain the step complexity, similar to \citet{chen2023improved,liang2025absorb},
the constant-step portion of the grid uses $O(T/\kappa)$ steps.
Once $u_k\le1$, the grid satisfies $u_{k+1}=(1-\kappa)u_k$,
so reaching $\delta$ requires $O(\log(\delta^{-1})/\kappa)$
additional steps. Therefore,
\begin{align*}
 N&=\calO\left(\frac{T+\log\delta^{-1}}{\kappa}\right).
\end{align*}
Thus, if we choose
\[ \eps_{\mathrm{score}} \lesssim\eps, \qquad
 \delta\asymp\frac{\sqrt\eps}{d}, \qquad
 T\asymp\log\frac{d}{\sqrt\eps},\quad \kappa\asymp\frac{\sqrt\eps}{d\min\{d,\log S\}}, \]
each term in
\eqref{eq:euler-new-final} is $O(\sqrt\eps)$, and the number of steps is
\[ N = \widetilde{\calO}\left(\frac{d\min\{d,\log S\}}{\sqrt\eps}\right). \]
The proof of \Cref{thm:conv_tv_absorb} is now complete.

\subsection{Remark on \texorpdfstring{$\tau$}{tau}-leaping}

\begin{remark}\label{rem:euler-new-tau}
The same convergence rate extends to $\tau$-leaping.
Adapting the sampler-comparison argument in
\citet[Appendix~H]{liang2025sampler} to the absorbing setting,
with the same initialization, predictor, and time grid, gives
\begin{align*}
    \TV{p_{\mathrm{Euler}}}{p^\tau_{T-\delta}}
    \le 1-e^{-d\kappa}\le d\kappa.
\end{align*}
Here, each coordinate can be unmasked at most once, so the
comparison incurs a multiplicative loss at most $e^{-\kappa}$
for each of at most $d$ unmasking events.
The triangle inequality then shows that replacing Euler with
$\tau$-leaping adds at most $d\kappa$ to the TV error,
which preserves the stated convergence rate.
\end{remark}

\section{Proof of \texorpdfstring{\Cref{thm:lower_absorb}}{Theorem 3}}
\label{app:thm:lower_absorb_proof}

We construct a distribution $q^*$ in which different coordinates share bits that must agree in every valid sequence. 
The underlying idea is that if two coordinates that share a bit are unmasked in the same step, since Euler predicts their tokens independently using the earlier visible tokens, the resulting sampled bits can disagree. Also, this disagreement persists afterwards by the property of MDM. 
As follows, we first reduce the TV lower bound to the probability of such simultaneous unmasking events. We then bound this probability and derive the required number of steps.

\subsection{Setup and illustration}
\label{subsec:euler-lower-q-example}

We use the bit-representation for a token value. For each $a \in [S] \setminus \{\mask\}$, we need $\lceil\log_2(S-1)\rceil$ bits to fully encode $a$. For our construction, we select only $2^b$ non-mask symbols and identify them with all binary strings of length $b$, where $b$ will be specified below. Define the dependency graph as $G = (V,E)$, where the node set $V = [d]$ is exactly the token locations. The graph is undirected, and each edge is written once as $(i,j)$ with $i<j$. For $i,j \in V$ with $i<j$, let $(i,j) \in E$ if and only if there exists a shared bit in these nodes (or equivalently, at the corresponding token locations). The incident edges at each node are assigned to distinct bit positions in any fixed order. Let each shared bit $B_{ij} \sim \mathrm{Bernoulli}(1/2)$, mutually independent. For each node $i \in V$, if a bit is not shared, it is simply set deterministically (say, zero).
Thus, $q^*$ is purely determined by the probability of these shared bits. 

We provide an example for purpose of illustration. Consider $d=3$, $S=5$, consisting of one $\mask$ symbol and four other non-mask symbols, and $b=2$. Let $G$ be the triangle on vertices $V = \{1,2,3\}$ with edges $E = \{\{1,2\}, \{1,3\}, \{2,3\}\}$. The latter correspond to three independent fair bits $B_{12},B_{13},B_{23}$, one for each edge,
and define $q^*$ through
\begin{align*}
    x^1=(B_{12},B_{13}),\qquad
    x^2=(B_{23},B_{12}),\qquad
    x^3=(B_{13},B_{23}).
\end{align*}

With such a dependency construction, under the exact reverse process, once a token is unmasked, the shared bits it contains become known. The conditional distributions of its still-masked neighbors then force their copies of these bits to agree when those tokens are subsequently unmasked. For the Euler sampler, however, neighboring tokens may be unmasked in the same step. Their predictions are sampled independently using the visible tokens at the beginning of that step, so their copies of a shared bit may disagree.

Define
\begin{align}\label{eq:euler-lower-construction}
    b := \min\{d-1,\lfloor\log_2(S-1)\rfloor\},
    \qquad n=\left\lfloor\frac{d}{b+1}\right\rfloor.
\end{align}
The quantity $b$ is the number of shared bits carried by each token within a group, while $n$ is the number of complete groups that fit into the $d$ token locations. We use $2^b\le S-1$ non-mask symbols to encode all binary strings of length $b$. The full non-mask vocabulary is used only when $2^b=S-1$.
When $d$ is small (i.e., $d-1 \leq \lfloor\log_2(S-1)\rfloor$), we have $b=d-1$ and $n=1$, so all $d$ token locations form one group. Otherwise, when $d$ is large, the vocabulary size limits $b$ to $\lfloor\log_2(S-1)\rfloor$. In either case, we divide the first $n(b+1)$ coordinates into $n$ disjoint subsets, where each subset consists of $b+1$ tokens, and leave the remaining $d-n(b+1)$ coordinates isolated. Then, the same construction is repeated within each group, i.e., connecting every pair of locations by an edge. Observe that the number of shared bits carried by each token within a group is equal to the size of that group minus one. In essence, this guarantees that, as long as $x^i$ is in a group, all bits of $x^i$ are shared, and each with one of the neighbors. Otherwise, no bits in $x^i$ are shared, and the isolated token is fixed to the all-zero encoding.
We can calculate the total number of edges:
\begin{align}\label{eq:euler-lower-edge-count}
    |E| = n\binom{b+1}{2},\quad \text{and}~~
    \frac{db}{4}\le |E| \le\frac{db}{2},
\end{align}
where the inequalities follow from $d/[2(b+1)]\le n\le d/(b+1)$, since $b+1\le d$.

Also, throughout the proof, we use the true clean-token conditional under $q^*$ whenever the visible tokens have positive probability. In other words, no estimation error is considered and the score is exact on states of positive forward probability (also see \eqref{eq:equivalence-between-score-and-mu} and \eqref{eq: equivalence-between-true-score-and-clean-distribution}). On inputs whose visible tokens have zero probability under $q^*$, define the predictor to output the all-zero encoding. This specifies a normalized non-mask predictor on every input.

With such a $q^*~(= q_0)$, we are now ready for the lower bound on TV.

\subsection{Step 1: reducing the TV lower bound to simultaneous unmasking}
\label{sec:euler-lower-reduction}

Let $\mathcal E$ denote the event that at least one edge bit has two
unequal copies in non-mask tokens. Since $q^*(\mathcal E)=0$, the
event characterization of TV distance gives
\begin{align}\label{eq:euler-lower-tv-event}
    \TV{q^*}{p_{\mathrm{Euler}}}
    &\ge\abs{q^*(\mathcal E)-p_{\mathrm{Euler}}(\mathcal E)}
      =p_{\mathrm{Euler}}(\mathcal E).
\end{align}
Thus, it suffices to lower-bound the probability that Euler
creates at least one inconsistent shared bit.

Let $u_k = T - t_k$ be the forward diffusion time, with $u_0 = T$ and $u_N = \delta$. Recall the step-size $\eta_k = t_{k+1}-t_k=u_k-u_{k+1}=\kappa\min\{1,u_k\}$. Let $K_i$ denote the actual Euler step at which
coordinate $i$ is first unmasked, with $K_i=\infty$ if it is still masked at the end. We restrict attention to the steps
\begin{align}\label{eq:euler-lower-window}
    \mathcal K=\{k\in\{0,\ldots,N-1\}:1\le u_k\le3/2\},
    \qquad
    C_G=\sum_{(i,j)\in E}\ind{K_i=K_j\in\mathcal K}.
\end{align}
Here $C_G$ counts edges whose two endpoints are unmasked in the same step within $\mathcal K$. The set $\mathcal K$ consists of steps in the fixed forward-time interval $[1,3/2]$, where the step-size is still $\kappa$.
Since the unmasking probability $\eta_k/(e^{u_k}-1)$ depends only on the step (see \eqref{eq:euler-lower-rho}), the Bernoulli variables determining $\Pi=(K_1,\dots,K_d)$ can be drawn independently of the randomness used to sample token values. Thus, conditioning on $\Pi$ fixes the unmasking steps while preserving the predictor's conditional law for the token sampled at each unmasking.
Indeed, conditioned on $\Pi$, the probability of no error is
\begin{align}\label{eq:euler-lower-conditional-agreement}
    \mathbb P\cbrc{\mathcal E^c\mid \Pi} &\stackrel{(i)}{=} \prod_{k=0}^{N-1} \mathbb P(\text{no new disagreement at step }k \mid \Pi,\text{ no earlier disagreement before }k) \nonumber \\
    &\stackrel{(ii)}{=} \prod_{k=0}^{N-1} \prod_{\substack{(i,j)\in E\\K_i=K_j=k}}\frac12 = \prod_{k=0}^{N-1} 2^{-\sum_{(i,j)\in E}\ind{K_i=K_j=k}} \nonumber \\
    &\leq \prod_{k \in \mathcal{K}} 2^{-\sum_{(i,j)\in E}\ind{K_i=K_j=k}} = 2^{-C_G},
\end{align}
where $(i)$ follows by applying the chain rule over the sampling steps, conditioning at each step on no earlier disagreement. We now show $(ii)$ below. Fix an edge $(i,j)\in E$
and a pre-step state $x$ with both endpoints masked and no
earlier disagreement. By the Euler definition in
\eqref{eq:def_euler}, for any non-mask symbol $a$,
\begin{equation*}
    \mathbb P\cbrc{X^i_{t_{k+1}}=a\mid X_{t_k}=x}
    =\eta_k\widehat R_k^i(\mask,a) = \frac{\eta_k}{e^{u_k}-1} q_0^i(a\mid x^{\mathrm{UM}}).
\end{equation*}
Here $q_0=q^*$ is the constructed data distribution, and $q_0^i$ denotes its true clean-token conditional. Since there is no earlier disagreement, the visible tokens in $x$ are consistent with this distribution. Below, conditioning on $x$ means conditioning on $X_{t_k}=x$. Thus, for $v\in\{0,1\}$,
\begin{align*}
    &\mathbb P\cbrc{
        x^i\text{ unmasks with its copy of }B_{ij}=v
        \mid x
    } \\
    &\qquad=
    \sum_{\substack{
        a:\text{the corresponding bit equals }v
    }}
    \frac{\eta_k}{e^{u_k}-1} q_0^i(a\mid x^{\mathrm{UM}})
    =\frac{\eta_k}{2 (e^{u_k}-1)},
\end{align*}
where the last equality holds because $B_{ij}$ appears only in tokens $i$ and $j$, both of which are masked, and is independent of the other edge bits. It therefore remains fair given the visible tokens. Since Euler samples the two coordinates independently given $x$,
\begin{align*}
    &\mathbb P\cbrc{
        \text{both $x^i$ and $x^j$ unmask and their copies agree}\mid x
    } \\
    &\qquad=
    2 \brc{\frac{\eta_k}{2 (e^{u_k}-1)}}^2 = \frac{\eta_k^2}{2 (e^{u_k}-1)^2}.
\end{align*}
Meanwhile, the probability that both endpoints unmask is
$\frac{\eta_k^2}{(e^{u_k}-1)^2}$. Thus,
\begin{equation*}
    \mathbb P\cbrc{
        \text{corresponding bits agree}\mid x, \text{both $x^i$ and $x^j$ unmask}
    }
    =\frac{\eta_k^2/(2(e^{u_k}-1)^2)}{\eta_k^2/(e^{u_k}-1)^2}
    =\frac12.
\end{equation*}
For distinct edges whose endpoints unmask in this step, the agreement events are independent given $x$ and $\Pi$. Thus, the probability of no new disagreement is the product of one factor $1/2$ for each edge with $K_i=K_j=k$. This probability is the same for every pre-step state with no earlier disagreement, and averaging over these states establishes $(ii)$, and thus \eqref{eq:euler-lower-conditional-agreement}.

Now, combining \eqref{eq:euler-lower-tv-event} and \eqref{eq:euler-lower-conditional-agreement} yields the overall reduction:
\begin{equation}\label{eq:euler-lower-tv-reduction}
    \TV{q^*}{p_{\mathrm{Euler}}} \ge\E\sbrc{1-2^{-C_G}} \ge\frac12\mathbb P\cbrc{C_G>0}.
\end{equation}
The remaining task is to show that at least one of the $|E|$ edges has its two endpoints unmasked in the same step with probability of order $\min\{|E|\kappa,1\}$.

\subsection{Step 2: bounding the first and second moments}
\label{sec:euler-lower-one-edge}
\label{sec:euler-lower-all-edges}

To bound $\mathbb P\cbrc{C_G>0}$ in
\eqref{eq:euler-lower-tv-reduction}, we first apply the
Cauchy--Schwarz inequality:
\begin{align*}
    \E C_G
    =\E\sbrc{C_G\ind{C_G>0}}
    \le\sqrt{\E C_G^2}\sqrt{\mathbb P\cbrc{C_G>0}}.
\end{align*}
Rearranging gives
\begin{align}\label{eq:euler-lower-moment-reduction}
    \mathbb P\cbrc{C_G>0}
    \ge\frac{(\E C_G)^2}{\E C_G^2}.
\end{align}
Thus, it suffices to lower-bound the first moment and upper-bound
the second moment. Intuitively, the first moment counts the expected number of
edges whose endpoints unmask in the same step within $\mathcal K$.
The second moment accounts for the dependence between these events. In particular, we will show below that
$\E C_G\asymp |E|\kappa$ and
$\E C_G^2\lesssim |E|\kappa+|E|^2\kappa^2$, which together yield
the desired probability of order $\min\{|E|\kappa,1\}$.

\paragraph{First moment} Since $C_G$ is a summation of indicator variables, to characterize its first moment, we essentially need to investigate the quantity $\mathbb P\cbrc{K_i=K_j=k}$ for $k \in \mathcal K$.
For a coordinate that is still masked at forward time $u_k$, the exact-score Euler unmasking probability is, by
\eqref{eq:def_euler} and the score parameterization in
\eqref{eq:equivalence-between-score-and-mu},
\begin{align}\label{eq:euler-lower-rho}
    \rho_k
    &=\frac{\eta_k}{e^{u_k}-1}
      \sum_{a\ne\mask}\mu_\theta^i(x,u_k)[a]
      =\frac{\eta_k}{e^{u_k}-1}.
\end{align}
The last equality holds because $\mu_\theta^i$ assigns no mass to $\mask$. 
In particular, $\rho_k$ depends only on $k$, and $0\le\rho_k\le\kappa$ since $e^{u_k}-1\ge u_k$.
Thus, for each $i \in [d]$, the probability that the unmasking happens at step $k$ is
\begin{align}\label{eq:euler-lower-unmasking-law}
    v_k:=\mathbb P\cbrc{K_i=k}
    =\rho_k\prod_{j<k}(1-\rho_j),
    \qquad 0\le k\le N-1.
\end{align}
Also, note that the $K_i$ are independent and identically distributed. As a result, the first moment of $C_G$ satisfies
\begin{align}\label{eq:euler-lower-first-moment}
    \E C_G
    &=\sum_{(i,j)\in E} \mathbb P\cbrc{K_i=K_j\in\mathcal K} \nonumber \\
    &\stackrel{(i)}{=} \sum_{(i,j)\in E} \sum_{k\in\mathcal K} \mathbb P\cbrc{K_i = k} \mathbb P\cbrc{K_j = k} = |E| \sum_{k\in\mathcal K} v_k^2, \nonumber \\
    &= |E| \sum_{k\in\mathcal K} \rho_k^2 \prod_{j<k} (1-\rho_j)^2,
\end{align}
where $(i)$ follows by independence between $K_i$ and $K_j$.
Since $T\ge2$ and $\delta\le1/2$, the grid traverses the whole
interval $[1,3/2]$ with spacing $\kappa$. Taking $\kappa_0\le1/4$,
we have
\begin{align}\label{eq:euler-lower-window-size}
    \frac{1}{4\kappa}
    \le\left\lfloor\frac{1}{2\kappa}\right\rfloor
    \le|\mathcal K|
    \le\frac{1}{2\kappa}+1
    \le\frac{1}{\kappa}.
\end{align}
To lower bound the first moment in \eqref{eq:euler-lower-first-moment}, note that, for every $k\in\mathcal{K}$, the adaptive schedule has step-size $t_{k+1}-t_k = \kappa$, and thus
\[
    \rho_k
    =\frac{\kappa}{e^{u_k}-1}
    \geq \frac{\kappa}{e^{3/2}-1}.
\]
On the other hand, using $\log(1-z)\geq-2z$ for $z\in[0,1/2]$, we have
\begin{align*}
    \prod_{j<k} (1-\rho_j)^2 &\geq \exp\brc{-4 \sum_{j<k} \rho_j} \geq \exp\brc{-4 \int_{u_k}^T \frac{\d u}{e^u - 1}} \geq \exp\brc{-4 \int_{1}^\infty \frac{\d u}{e^u - 1}} = c_0,
\end{align*}
where $c_0 > 0$ is a universal constant.
Thus, continuing from \eqref{eq:euler-lower-first-moment}, the first moment satisfies that
\begin{align}\label{eq:pair-error-lower-bound}
    \E C_G \ge |E| |\mathcal K| c_0 \frac{\kappa^2}{(e^{3/2}-1)^2} \ge c_1|E|\kappa,
\end{align}
where the last inequality uses \eqref{eq:euler-lower-window-size}, and $c_1 > 0$ is a universal constant.

Later, when bounding $\E C_G^2$, we also need an upper bound on the first moment. Since $1-\rho_k \leq 1$ and $\rho_k \leq \frac{\kappa}{e-1}$, \eqref{eq:euler-lower-first-moment} also yields
\begin{align}\label{eq:pair-error-upper-bound}
    \E C_G \leq |E| \sum_{k\in\mathcal K} \rho_k^2 \leq |E| |\mathcal K| \frac{\kappa^2}{(e-1)^2} \leq c_2 |E|\kappa,
\end{align}
where the last inequality uses \eqref{eq:euler-lower-window-size}, and $c_2 > 0$ is a universal constant.

\paragraph{Second moment} Expanding the square of $C_G$, we have
\[ \E C_G^2 =\E C_G+ \sum_{\substack{(i,j),(h,m)\in E\\(i,j)\ne(h,m)}} \mathbb P\cbrc{K_i=K_j\in\mathcal K,\ K_h=K_m\in\mathcal K}. \]
Here the sum runs over ordered pairs of distinct edges, of which there
are $2\binom{|E|}{2}$. 
We now divide into two cases.

\emph{Case 1: the two edges share a token.}
Without loss of generality assume $i=h$. Then, the two simultaneous-unmasking events require all three
distinct tokens to unmask in the same step in $\mathcal K$.
Similarly to the argument for $(i)$, by independence of their unmasking steps, we have
\begin{align}\label{eq:euler-lower-moment-sums}
    &\mathbb P\cbrc{K_i=K_j\in\mathcal K, K_i=K_m\in\mathcal K}
      \nonumber\\
    &\qquad=\sum_{k\in\mathcal K}v_k^3
      \le\max_{k\in\mathcal K}v_k\sum_{k\in\mathcal K}v_k^2
      \stackrel{(iii)}{\leq} \frac{\kappa}{e-1}\frac{\E C_G}{|E|}
      \le c_3 \kappa^2,
\end{align}
where $c_3 > 0$ is a universal constant. Here $(iii)$ uses $v_k\le\rho_k\le \frac{\kappa}{e-1}$ and $\sum_{k\in\mathcal K}v_k^2=\E C_G/|E|$ from \eqref{eq:euler-lower-first-moment}, while the last inequality follows from \eqref{eq:pair-error-upper-bound}.

\emph{Case 2: the two edges have no shared token.}
The two edges involve four distinct tokens, so their
simultaneous-unmasking events are independent. Each pair must
unmask together, but the two pairs may unmask in different steps.
Thus,
\begin{align*}
    &\mathbb P\cbrc{K_i=K_j\in\mathcal K,\ K_h=K_m\in\mathcal K}\\
    &\qquad=
      \mathbb P\cbrc{K_i=K_j\in\mathcal K}
      \mathbb P\cbrc{K_h=K_m\in\mathcal K}\\
    &\qquad=\left(\sum_{k\in\mathcal K}v_k^2\right)^2
      =\left(\frac{\E C_G}{|E|}\right)^2 \le c_4 \kappa^2,
\end{align*}
where $c_4 > 0$ is a universal constant. Here the last inequality again follows from
\eqref{eq:pair-error-upper-bound}.

Combining both cases, we get
\begin{align}\label{eq:euler-lower-second-moment}
    \E C_G^2
    &=\E C_G+
      \sum_{\substack{(i,j),(h,m)\in E\\(i,j)\ne(h,m)}}
      \mathbb P\cbrc{K_i=K_j\in\mathcal K,\ K_h=K_m\in\mathcal K}
      \nonumber\\
    &\leq \E C_G+2\binom{|E|}{2} \max\{c_3, c_4\} \kappa^2 \nonumber \\
    &\leq c_5 \brc{|E|\kappa+|E|^2\kappa^2},
\end{align}
where $c_5 > 0$ is a universal constant. Here the last inequality follows again from
\eqref{eq:pair-error-upper-bound}.

Finally, substituting the first- and second-moment bounds in \eqref{eq:pair-error-lower-bound} and \eqref{eq:euler-lower-second-moment} into
\eqref{eq:euler-lower-moment-reduction}, we conclude that
\begin{align}\label{eq:euler-lower-collision-probability}
    \mathbb P\cbrc{C_G>0} \ge \frac{(\E C_G)^2}{\E C_G^2} \ge c_6 \min\{|E|\kappa,1\},
\end{align}
where $c_6 > 0$ is a universal constant.

\subsection{Step 3: combining the bounds and obtaining step complexity}
\label{sec:euler-lower-combine}

Substituting \eqref{eq:euler-lower-collision-probability} into
\eqref{eq:euler-lower-tv-reduction} and then applying
\eqref{eq:euler-lower-edge-count}, we obtain
\begin{align}\label{eq:euler-lower-combined}
    \TV{q^*}{p_{\mathrm{Euler}}}
    &\ge\frac12\mathbb P\cbrc{C_G>0}
      \ge \frac{c_6}{2} \min\{|E|\kappa,1\}
      \ge \frac{c_6}{8} \min\{db\kappa,1\}.
\end{align}
From \eqref{eq:euler-lower-construction}, it is easy to see that $b\asymp\min\{d,\log S\}$. Thus, we finally arrive at
\begin{align*}
    \TV{q^*}{p_{\mathrm{Euler}}}
    \ge c\min\{\kappa d\min\{d,\log S\},1\}.
\end{align*}
To calculate step complexity, for sufficiently small $\eps$, in order for the TV distance to be $\calO(\sqrt\eps)$, the quantity inside the minimum must (be less than one and) satisfy
\begin{align*}
    \kappa d\min\{d,\log S\} \lesssim \sqrt\eps.
\end{align*}
Since the grid uses $\Omega(1/\kappa)$ steps merely to traverse the window $u\in[1,3/2]$, we have
\begin{align*}
    N\ge|\mathcal K| \gtrsim\frac1\kappa \gtrsim\frac{d\min\{d,\log S\}}{\sqrt\eps}.
\end{align*}
The proof of \Cref{thm:lower_absorb} is now complete. Note that this matches the upper bound in \Cref{thm:conv_tv_absorb} up to logarithmic factors.

\section{Proof of \texorpdfstring{Proposition~\ref{prop: equivalence between lse and nelbo}}{Proposition 4}}
\label{app:prop: equivalence between lse and nelbo}

Suppose $x_t$ has $i$-th token being $\mask$, i.e., $x_t = x_t^{-i} \oplus_i \mask$. Recall the connection between the score estimator $s_t(y,x)$ and $\mu_\theta$ in \eqref{eq:equivalence-between-score-and-mu}:
\begin{equation*}
s_t(x_t^{-i} \oplus_i a,x_t)
= \frac{\alpha_t}{1-\alpha_t}\, \mu_\theta^i(x_t,t)[a],
\end{equation*}
and the connection between concrete score and clean data distribution in \eqref{eq: equivalence-between-true-score-and-clean-distribution}:
\begin{equation*}
\frac{q_t(x_t^{-i}\oplus_i a)}{q_t(x_t)}
= \frac{\alpha_t}{1-\alpha_t}\,
q_0^i(a|x_t^{UM}),
\end{equation*}
where $x_t^{UM}$ denotes the unmasked token collection of $x_t$. Here we also recall that $\alpha_t = e^{-t}$.

Then, since there is no mask in the data, we have
\begin{equation}
\label{eq: alphat / 1-alphat}
    \sum_{a\neq \mask} \frac{q_t(x_t^{-i}\oplus_i a)}{q_t(x_t)} = \frac{\alpha_t}{1-\alpha_t} \sum_{a\neq \mask} q_0^i(a|x_t^{UM}) = \frac{\alpha_t}{1-\alpha_t}.
\end{equation}

Recall the definition of the integrated $\mathcal{L}_{SE}$ loss in \eqref{eq:def_score_entropy}. We have
\begin{align}
&\int_0^{+\infty} \mathcal{L}_{SE}(s,t)\, \mathrm{d}t \nonumber\\
&= \int_0^{+\infty}
\mathbb{E}_{x_t \sim q_t}
\sum_{y \neq x_t}
R_t(y,x_t)
\Big(
s_t(y,x_t)
- \frac{q_t(y)}{q_t(x_t)}
- \frac{q_t(y)}{q_t(x_t)}\log \tfrac{s_t(y,x_t)}{q_t(y)/q_t(x_t)}
\Big)
\, \mathrm{d}t \nonumber \\
&\overset{(i)}{=}\int_0^{+\infty} \mathbb{E}_{x_t \sim q_t} \sum_{l: x_t^l = \mask}
\sum_{a \neq \mask}
\left(
\frac{\alpha_t}{1-\alpha_t} \mu_\theta^l(x_t,t)[a]
- \frac{q_t(x_t^{-l}\oplus_l a)}{q_t(x_t)}
\right) \mathrm{d}t \nonumber \\
&\quad - \int_0^{+\infty} \mathbb{E}_{x_t \sim q_t} \sum_{l: x_t^l = \mask}\sum_{a \neq \mask}
\frac{q_t(x_t^{-l}\oplus_l a)}{q_t(x_t)}
\log
\frac{
\frac{\alpha_t}{1-\alpha_t} \mu_\theta^l(x_t,t)[a]
}{
\frac{q_t(x_t^{-l}\oplus_l a)}{q_t(x_t)}
} \mathrm{d}t \nonumber \\
& \overset{(ii)}{=} 0 - \int_0^{+\infty} \mathbb{E}_{x_t \sim q_t} \sum_{l: x_t^l = \mask}\sum_{a \neq \mask}
\frac{q_t(x_t^{-l}\oplus_l a)}{q_t(x_t)}
\log
\frac{
\frac{\alpha_t}{1-\alpha_t} \mu_\theta^l(x_t,t)[a]
}{
\frac{q_t(x_t^{-l}\oplus_l a)}{q_t(x_t)}
} \mathrm{d}t \nonumber\\
& \overset{(iii)}{=} - \int_0^{+\infty} \mathbb{E}_{x_t \sim q_t} \sum_{l: x_t^l = \mask}\sum_{a \neq \mask}
\frac{\alpha_t}{1-\alpha_t}
q_0^l(a|x_t^{UM})
\log
\frac{
\mu_\theta^l(x_t,t)[a]
}{
q_0^l(a|x_t^{UM})
} \mathrm{d}t \nonumber\\
&=\int_0^{+\infty}
- \frac{\alpha_t}{1-\alpha_t}
\mathbb{E}_{x_t}
\sum_{l: x_t^l = \mask}
\sum_{a \neq \mask}
q_0^l(a|x_t^{UM})
\log \mu_\theta^l(x_t,t)[a]
\, \mathrm{d}t \nonumber \\
&\quad -\int_0^{+\infty}
\frac{\alpha_t}{1-\alpha_t}
\mathbb{E}_{x_t}
\sum_{l: x_t^l = \mask}
H\left(q_0^l(\cdot|x_t^{UM})\right)
\, \mathrm{d}t, \label{eq:intlse}
\end{align}
where $H(\cdot)$ in the last equality denotes the entropy.
In the series of equations above, $(i)$ follows from \eqref{eq:equivalence-between-score-and-mu}, $(ii)$ follows from \eqref{eq: alphat / 1-alphat} and because $\sum_{a\neq \mask}\mu_{\theta}^l(x_t,t)[a] = 1$, and $(iii)$ follows from \eqref{eq: equivalence-between-true-score-and-clean-distribution}. 

Note that \eqref{eq:intlse} further implies that $\int_0^{+\infty} \mathcal{L}_{SE}(s,t)\, \mathrm{d}t = 0$ if and only if $q_0^l(\cdot|x_t^{UM}) = \mu_\theta^l(x_t,t)[\cdot]$ for (almost surely) all $t$, $x_t$, and $l$. The ``if'' part is obvious. For the ``only if'' part, it is because the right-hand-side is the scaled expected sum of KL divergences between $q_0^l(\cdot|x_t^{UM})$ and $\mu_\theta^l(x_t,t)[\cdot]$, which is non-negative. Thus, if the integrated loss is exactly zero, any of the KL terms must also be zero, and thus $q_0^l(a|x_t^{UM}) = \mu_\theta^l(x_t,t)[a]$ for (almost surely) all $a$, $t$, $x_t$, and $l$.

Now we want to further connect the two terms in \eqref{eq:intlse} with the training loss and, in particular, the expected NELBO.
Indeed, the first term in \eqref{eq:intlse} is exactly
\begin{equation}\label{eq:first term in prop 2}
    \int_0^{+\infty}- \frac{\alpha_t}{1-\alpha_t}\mathbb{E}_{x_t}\sum_{l: x_t^l = \mask}\sum_{a \neq \mask}q_0^l(a|x_t^{UM})\log \mu_\theta^l(x_t,t)[a]\, \mathrm{d}t = \mathbb{E}_{x_0 \sim q_0}[\mathcal{L}_\infty(x_0)].
\end{equation}
To see this, by the definition of NELBO $\mathcal{L}_\infty(x_0)$ in \eqref{eq:def of nelbo}, we have
\begin{align*}
\mathbb{E}_{x_0 \sim q_0} [\mathcal{L}_\infty(x_0)] &= \int_0^{+\infty} - \frac{\alpha_t}{1-\alpha_t} \mathbb{E}_{x_t \sim q_t, x_0 \sim q_0} \sbrc{ \sum_{l: x_t^l = \mask} \log \mu_\theta^l(x_t,t)[x_0^l] } \mathrm{d}t \\
& = \int_0^{+\infty} - \frac{\alpha_t}{1-\alpha_t} \mathbb{E}_{x_t \sim q_t} \sbrc{ \mathbb{E}_{x_0 \sim q_{0|t}(\cdot|x_t)} \sbrc{ \sum_{l: x_t^l = \mask} \log \mu_\theta^l(x_t,t)[x_0^l]} } \mathrm{d}t \\
&= \int_0^{+\infty} - \frac{\alpha_t}{1-\alpha_t} \mathbb{E}_{x_t \sim q_t} \sbrc{ \sum_{l: x_t^l = \mask} \mathbb{E}_{x_0^l \sim q_{0|t}^l(\cdot|x_t)} \sbrc{ \log \mu_\theta^l(x_t,t)[x_0^l] } } \mathrm{d}t \\
&\stackrel{(iv)}{=} \int_0^{+\infty} - \frac{\alpha_t}{1-\alpha_t} \mathbb{E}_{x_t \sim q_t} \sbrc{ \sum_{l: x_t^l = \mask} \mathbb{E}_{x_0^l \sim q_0^l(\cdot|x_t^{UM})} \sbrc{ \log \mu_\theta^l(x_t,t)[x_0^l] } } \mathrm{d}t \\
&= \int_0^{+\infty} - \frac{\alpha_t}{1-\alpha_t} \mathbb{E}_{x_t \sim q_t} \sum_{l: x_t^l = \mask} \sum_{a \neq \mask} q_0^l(a|x_t^{UM}) \log \mu_\theta^l(x_t,t)[a] \mathrm{d}t,
\end{align*}
where $(iv)$ follows because $q_{0|t}^l(\cdot|x_t) = q_{0}^l(\cdot|x_t^{UM})$ by properties of MDMs (cf. \citet[Theorem~1]{ou2025absorb} and \citet[Lemma~1]{liang2025absorb}).
This establishes \eqref{eq:first term in prop 2}.

We now move to the second term in \eqref{eq:intlse}, for which we will show that
\begin{equation}
\label{eq:second term in prop 2}
     \int_0^{+\infty}
\frac{\alpha_t}{1-\alpha_t}
\mathbb{E}_{x_t}
\sum_{l: x_t^l = \mask}
H\left(q_0^l(\cdot|x_t^{UM})\right)
\, \mathrm{d}t = \sum_{k = 1}^d \frac{1}{k} \mathbb{E}_{x_0 \sim q_0} \mathbb{E}_{\calM_k}
\sum_{l \in \calM_k}
H\left(q_0^l(\cdot|x_0^{\calM_k^c})\right).
\end{equation}
Here we recall that $\calM_k \subseteq [d]$ (with $|\calM_k|=k$ for $k=1,\ldots,d$) denotes the subset with exactly $k$ token indices sampled uniformly from the total index set, and $\calM_k^c=[d] \setminus \calM_k$. We also write $x_0^{\calM_k^c}:=\{x_0^l\}_{l\in \calM_k^c}$, and $q_0^l(\cdot|x_0^{\calM_k^c})$ denotes the conditional distribution of $x_0^{l}$ given $x_0^{\calM_k^c}$. 
Indeed, we have
\begin{align*}
&\int_0^{+\infty} \frac{\alpha_t}{1-\alpha_t} \mathbb{E}_{x_t \sim q_t} \sum_{l: x_t^l = \mask} H\left(q_0^l(\cdot|x_t^{UM})\right) \mathrm{d}t \\
& \qquad = \int_0^{+\infty} \frac{\alpha_t}{1-\alpha_t} \mathbb{E}_{k \sim \mathrm{Binom}(d, 1-\alpha_t)}\mathbb{E}_{x_0 \sim q_0}\mathbb{E}_{\calM_k} \sum_{l\in \calM_k} H\left(q_0^l(\cdot|x_0^{\calM_k^c})\right) \mathrm{d}t\\
& \qquad = \int_0^{+\infty} \frac{\alpha_t}{1-\alpha_t} \sum_{k = 1}^d \binom{d}{k}(1-\alpha_t)^k\alpha_t^{d-k}\ \mathbb{E}_{x_0 \sim q_0}\mathbb{E}_{\calM_k} \sum_{l\in \calM_k} H\left(q_0^l(\cdot|x_0^{\calM_k^c})\right) \mathrm{d}t\\ 
& \qquad \stackrel{(v)}{=} \int_0^1 \frac{1}{1-\alpha_t} \sum_{k = 1}^d \binom{d}{k}(1-\alpha_t)^k\alpha_t^{d-k}\ \mathbb{E}_{x_0 \sim q_0}\mathbb{E}_{\calM_k} \sum_{l\in \calM_k} H\left(q_0^l(\cdot|x_0^{\calM_k^c})\right) \mathrm{d}\alpha_t\\
& \qquad = \int_0^1 \sum_{k = 1}^d \frac{d!}{k!(d-k)!}(1-\alpha_t)^{k-1}\alpha_t^{d-k}\ \mathbb{E}_{x_0 \sim q_0}\mathbb{E}_{\calM_k}
\sum_{l\in \calM_k}
H\left(q_0^l(\cdot|x_0^{\calM_k^c})\right) \mathrm{d}\alpha_t\\
& \qquad = \sum_{k = 1}^d \frac{1}{k} \left(\int_0^1  \frac{d!}{(k-1)!(d-k)!}(1-\alpha_t)^{k-1}\alpha_t^{d-k}\, \mathrm{d}\alpha_t\right)\ \mathbb{E}_{x_0 \sim q_0}\mathbb{E}_{\calM_k}
\sum_{l\in \calM_k}
H\left(q_0^l(\cdot|x_0^{\calM_k^c})\right) \\
& \qquad \stackrel{(vi)}{=} \sum_{k = 1}^d \frac{1}{k} \mathbb{E}_{x_0 \sim q_0}\mathbb{E}_{\calM_k}
\sum_{l\in \calM_k}
H\left(q_0^l(\cdot|x_0^{\calM_k^c})\right),
\end{align*}
where $(v)$ follows by the change-of-variable formula: with any function $F$, since $\alpha_t = e^{-t}$, we have $\d \alpha_t = - \alpha_t \d t$, and thus
\[ \int_0^\infty \frac{\alpha_t}{1-\alpha_t} F(\alpha_t) \d t = -\int_1^0 \frac{1}{1-\alpha} F(\alpha) \d \alpha = \int_0^1 \frac{1}{1-\alpha} F(\alpha) \d \alpha. \]
Also, $(vi)$ follows because the integrand is exactly the density of a beta distribution, which integrates to 1.
This establishes \eqref{eq:second term in prop 2}.

Finally, plugging the equalities in \eqref{eq:first term in prop 2} and \eqref{eq:second term in prop 2} into \eqref{eq:intlse}, we obtain 
the result as desired.
The proof is now complete.

\section{FHS Sampling in CTMC Framework: \texorpdfstring{Proposition~\ref{prop: fhs exact}}{Proposition 12} and its Proof}

In this section, we establish \Cref{prop: fhs exact}, which provides statistical characterization of the FHS sampling process and show that it exactly matches the reverse process under the CTMC framework for absorbing-rate diffusion models. This will highly facilitate us to develop the error analysis for \Cref{thm: error bound of fhs}.

Let $\mathbb{P}_{\mathrm{FHS}}$ and $\mathbb{P}_{\mathrm{CTMC}}$ denote the probability measure over the entire path generated by the FHS sampling algorithm and the ground truth reverse absorbing-rate continuous-time Markov chain (CTMC). Both processes admit a common simple representation. Since the sequence contains $d$ tokens, starts from the fully masked state, and each token remains fixed once it is unmasked, both processes are fully characterized by: (i) the $d$ unmasking times, (ii) which token is unmasked at each such time, and (iii) the token value to which that token is unmasked. We therefore introduce notation for these random variables. Let $k \in {0,1,\ldots,d}$ index the unmasking events. 
\begin{enumerate}
\item For $k\in[d]$, let $\tau_k$ denote the time at which the $(d-k)$-th unmasking occurs. Specifically, $\tau_d$ denotes the initialization time of the sampling processes. Clearly, unmasking times $\{\tau_k\}_k$ are increasing monotonically with respect to $k$, satisfying
\begin{equation*}
\tau_0 < \tau_1 < \cdots < \tau_{d-1} < \tau_d = +\infty,
\end{equation*}
where $\tau_d = +\infty$ is initialized by $\alpha(\tau_d) = e^{-\tau_d} = 0$.

\item Let $x_{\tau_k}$ denote the sequence state immediately after time $\tau_k$, i.e., $\tau_k - \delta$. Between consecutive unmasking events, the state remains constant; specifically, we set $x_t \equiv x_{\tau_k}$ for $t \in (\tau_{k-1},\tau_k]$ (equivalently, until the next unmasking time). 

\item Let $i_{\tau_k}\in [d]$ denote the token index unmasked at time $\tau_k$.

\item Let $a_{\tau_k}\in [S]$ denote the realized token value assigned at that time. Thus, the newly revealed token satisfies $x_{\tau_k}^{i_{\tau_k}} = a_{\tau_k}$. 
\end{enumerate}
To simplify the notation, we use $x_k,i_k,a_k$ to denote $x_{\tau_k}, i_{\tau_k}, a_{\tau_k}$, respectively. Further note that $x_k$ is fully determined by $\{i_j\}_{j=d}^k$ and $\{a_j\}_{j=d}^k$.

The following proposition provides useful statistical properties of the FHS and reverse CTMC processes and their connections.
\begin{proposition}\label{prop: fhs exact}
Suppose there are no mask tokens in the data. The FHS and reverse absorbing-rate CTMC processes satisfy the following properties:
\begin{equation*}
    \begin{aligned}
    &\mathbb{P}_{\mathrm{FHS}}(\tau_{k-1}| x_k, \tau_k) = \mathbb{P}_{\mathrm{CTMC}}(\tau_{k-1}| x_k, \tau_k)=\mathbb{P}_{\mathrm{CTMC}}(\tau_{k-1}|\tau_k),\\
    &\mathbb{P}_{\mathrm{FHS}}(i_{k-1}| x_k, \tau_{k-1}) = \mathbb{P}_{\mathrm{CTMC}}(i_{k-1}| x_k, \tau_{k-1}) = \frac{1}{k},\\
    &\mathbb{P}_{\mathrm{FHS}}(a_{k-1}| x_k, \tau_{k-1},i_{k-1}) = \mu_\theta^{i_{k-1}}(x_k,\tau_{k-1})[a_{k-1}],\\
    &\mathbb{P}_{\mathrm{CTMC}}(a_{k-1}| x_k, \tau_{k-1},i_{k-1}) 
    = q_0^{i_{k-1}}(a_{k-1}|x_k^{UM}). \end{aligned}
\end{equation*}
\end{proposition}

Clearly, the second equality of the first property shows that given $\tau_k$, the next umasking time $\tau_{k-1}$ is independent of $x_k$.
\begin{proof}
The proof of Proposition~\ref{prop: fhs exact} contains three steps: Step 1 shows the basic properties of transition dynamics of reverse CTMC, Step 2 further develops the statistical distributions for reverse CTMC, and Step 3 establishes the connection between the statistics of FHS and reverse CTMC processes.

{\bf Step 1. Basic properties of reverse CTMC:} We use $\widetilde{R}_t$ to denote the rate matrix of the reverse absorbing-rate CTMC indexed by forward time $t$. 
This is for the convenience to analyze both processes with the same time indexing. 
Specifically, we define the reverse rate matrix $\widetilde{R}_t$ indexed by the \textit{forward} time $t$ as, for $y \neq x$, 
\begin{equation}
\label{eq: def of r tilde}
\widetilde R_t(x,y)
\;:=\;
\lim_{\delta \downarrow 0}
\frac{
\mathbb{P}\cbrc{ X_{t-\delta} = y \,\middle|\, X_t = x }
}{\delta},
\end{equation}
or equivalently,
\begin{equation*}
    \mathbb{P}\cbrc{X_{t-\Delta t} = y|X_{t} = x} = \ind{y = x} + \widetilde{R}_t(x,y)\Delta t + o(\Delta t).
\end{equation*}

The diagonal entries are defined by
\begin{equation*}
\widetilde R_s(x,x)
\;=\;
- \sum_{y \neq x} \widetilde R_s(x,y).
\end{equation*}
Note that when horizon $T$ is finite, the defined reverse rate matrix $\widetilde{R}_t$ indexed by the forward time is equivalent to previous reverse rate matrix $\rvec{R}_t$ by
\begin{equation*}
    \widetilde{R}_t(x,y) = \rvec{R}_{T-t}(x,y),\quad \forall y\neq x.
\end{equation*}

Based on the notations and definitions above, we next have several basic results concerning time-inhomogeneous CTMCs in \Cref{lem: CTMC rate property}.
\begin{lemma}\label{lem: CTMC rate property}
Let $(Y_t)_t$ denote the reverse absorbing-rate time-inhomogeneous CTMC with the generator matrix $\{\widetilde{R}_t\}$ defined in \eqref{eq: def of r tilde}. Let the unmasking time in state $x$ be $ \tau := \sup\{s<t : Y_s \neq x\} $. Then
\begin{enumerate}
\item The instantaneous total rate of leaving $x$ at time $t$ is
\[
\Lambda(t,x)= \sum_{y\neq x} \widetilde{R}_t(x,y).
\]
\item The unmasking time $\tau$ in state $x$ satisfies
\[
\mathbb{P}\cbrc{\tau<t-h\mid Y_t=x}
= \exp\!\left(-\int_{t-h}^{t} \Lambda(u,x)\,du\right).
\]
\item Given that unmasking occurs at time $\tau$, the distribution of the predicted token value satisfies
\[
\lim\limits_{\delta\downarrow0}\ \mathbb{P}\cbrc{Y_\tau = y \mid Y_t=x, \tau \in (t-\delta,t]}
= 
\frac{\widetilde{R}_t(x,y)}{\Lambda(t,x)}.
\]
\end{enumerate}

\end{lemma}
\begin{proof}
    See \Cref{app:lem: CTMC rate property}.
\end{proof}

{\bf Step 2. Distributions induced by $\mathbb{P}_{\mathrm{CTMC}}$:} We characterize some important distributions induced by $\mathbb{P}_{\mathrm{CTMC}}$, which will be useful to establish the connection between $\mathbb{P}_{\mathrm{FHS}}$ and $\mathbb{P}_{\mathrm{CTMC}}$. For any $x_k$ with $k\ge 1$, there must exist $k$ $\mask$ tokens in $x_k$ due to the definition of $x_k$. Then without loss of generality, let $x^i$ be one of the $\mask$ tokens in the following proof.

Consider the reverse rate matrix $\widetilde{R}_t$ with the forward time index $t$ of CTMC. For any $a \neq \mask$, we have
\begin{equation}
\label{eq:reverse-rate-h}
\widetilde{R}_t\big(x,x^{-i}\oplus_i a\big)
=
R_{t}\big(x^{-i}\oplus_i a, x\big)\cdot \frac{q_{t}(x^{-i}\oplus_i a)}{q_{t}(x)}
=
 \frac{q_{t}(x^{-i}\oplus_i a)}{q_{t}(x)}.
\end{equation}

Following from Claim 1 of Lemma~\ref{lem: CTMC rate property}, the total rate of leaving $x$ at time $t$ on index $i$ is given by
\begin{equation}
\label{eq: total rate of leaving x at time t on index i}
    r_i(t,x):=\sum_{b\neq \mask}\widetilde{R}_{t}(x,x^{-i}\oplus_i b) = \frac{\sum_{b\neq \mask}q_{t}(x^{-i}\oplus_i b)}{q_{t}(x)} = \frac{\alpha_t}{1-\alpha_t},
\end{equation}
where the last equality follows from \eqref{eq: alphat / 1-alphat}.
Then the total reverse jump rate of leaving $x$ at time $t$ is given by
\begin{equation}
\label{eq: Lambda}
\Lambda(t,x) = \sum_{y\neq x} \widetilde{R}_t(x,y) = \sum_{i: x^i=\mask} r_i(t,x)=m(x) \cdot \frac{\alpha_{t}}{1-\alpha_{t}},
\end{equation}
where $m(x)$ denotes the number of masked tokens in $x$.

Then, substituting \eqref{eq: Lambda} into Claim 2 of Lemma~\ref{lem: CTMC rate property}, we can derive the conditional distribution of $\tau_{k-1}$ given $x_k$ and $\tau_k$ as:
\begin{align*}
\mathbb{P}_{\mathrm{CTMC}}(\tau_{k-1} < t|x_k,\tau_k )& = \exp{\left(-\int_t^{\tau_k}\Lambda(u,x_k)\mathrm{d}u\right)}\\
& =\exp\!\Big(-\!\!\int_{t}^{\tau_k} k\cdot \frac{e^{-u}}{1-e^{-u}}\,du\Big)\\
& = \exp\!\Big(-\!\! k\cdot \log(1-e^{-u})|_{u = t}^{\tau_k}\Big)\\
& = \left(\frac{1-e^{-t}}{1-e^{-\tau_k}}\right)^k.
\end{align*}
Here we note that, given $\tau_k$, $\tau_{k-1}$ does not depend on $x_k$.

Next, we derive $\mathbb{P}_{\mathrm{CTMC}}(i_{k-1}|x_k,\tau_{k-1})$. Applying Claim 3 in Lemma~\ref{lem: CTMC rate property} together with \eqref{eq: total rate of leaving x at time t on index i} and \eqref{eq: Lambda}, we obtain 
\begin{align*}
    \mathbb{P}_{\mathrm{CTMC}}(i_{k-1}|x_k,\tau_{k-1}) & = \sum_{b\neq \mask}  \frac{\widetilde{R}_{\tau_{k-1}}(x_k,x_k^{-i_{k-1}}\oplus_{i_{k-1}} b)}{\Lambda (\tau_{k-1},x_k)} \\
    &= \frac{r_{i_{k-1}}(\tau_{k-1},x_k)}{\Lambda(\tau_{k-1},x_k)} = \frac{\frac{\alpha_t}{1-\alpha_t}}{k\cdot \frac{\alpha_t}{1-\alpha_t}} = \frac{1}{k}.
\end{align*}
 
Finally, we derive the conditional distribution over the predicted token value $a$. Combining \eqref{eq: total rate of leaving x at time t on index i} and Claim 3 in Lemma~\ref{lem: CTMC rate property}, we have that, for all $t$ and $i$,
\begin{align*}
    \mathbb{P}_{\mathrm{CTMC}}(a|x_k,t,i)& =\frac{ \widetilde{R}_t\big(x_k, x_k^{-i} \oplus_i a\big)/\Lambda(t,x_k) }{ \sum_{b\neq \mask} \widetilde{R}_t\big(x_k, x_k^{-i}\oplus_i b \big)/\Lambda(t,x_k) }   \\
&= \frac{ \widetilde{R}_t\big(x_k, x_k^{-i} \oplus_i a\big) }{ \sum_{b\neq \mask} \widetilde{R}_t\big(x_k, x_k^{-i}\oplus_i b \big) } \\
&\overset{(i)}{=} \frac{q_{t} \big(x_k^{-i}\oplus_i a\big) }{ \sum_{b\neq \mask} q_{t}(x_k^{-i}\oplus_i b)  }\\
&=  \frac{ q_{t} \big(x_k) }{ \sum_{b\neq \mask} q_{t}(x_k^{-i}\oplus_i b)  }  \cdot \frac{ q_{t}\big(x_k^{-i}\oplus_i a\big) }{ q_{t}\big(x_k) }  \\ 
& = \frac{1-\alpha_{t}}{\alpha_{t}}\cdot \frac{ q_{t}\big(x_k^{-i}\oplus_i a\big) }{ q_{t}\big(x_k)}\\
& \overset{(ii)}{=} q_0^i(a|x_k^{UM}),
\end{align*}
where $(i)$ follows by substituting \eqref{eq:reverse-rate-h} and canceling $q_{t}(x_k)$ in the numerator and the denominator, and $(ii)$ follows from \eqref{eq: equivalence-between-true-score-and-clean-distribution}.
Thus, 
\begin{equation*}
    \mathbb{P}_{\mathrm{CTMC}}(a_{k-1}|x_k,\tau_{k-1},i_{k-1}) = q_0^{i_{k-1}}(a_{k-1}|x_k^{UM}) .
\end{equation*}

{\bf Step 3. Connection between $\mathbb{P}_{\mathrm{FHS}}$ and $\mathbb{P}_{\mathrm{CTMC}}$:} 
We derive the distribution of $\mathbb{P}_{\mathrm{FHS}}$ and show its connection to $\mathbb{P}_{\mathrm{CTMC}}$. 

Recall in the FHS algorithm, the update rule sets $\tau_{k-1}$ as 
\begin{equation*}
   \alpha_{\tau_d} \gets 0,\quad \tau_{k-1} \gets \alpha^{-1}(1-u_k^{1/k}(1-\alpha_{\tau_k})), \text{for }k = 1,2,\cdots,d,
\end{equation*}
where $\alpha_t = e^{-t}$, $\alpha^{-1}(z) = -\log(z)$, and $\{u_k\}_{k = 1}^d \overset{\mathrm{i.i.d}}{\sim} \mathrm{Uniform}([0,1])$. The above update rule for $\tau_{k-1}$
can be rewritten as
\begin{equation*}
    \left(\frac{1-e^{-\tau_{k-1}}}{1-e^{-\tau_k}}\right)^k = u_k\sim \mathrm{Uniform}([0,1]),
\end{equation*}
which, based on the inverse CDF method (or inverse transform sampling), is equivalent to sample $\tau_{k-1}$ from the distribution with the CDF of $\mathbb{P}_{\mathrm{CTMC}}(\tau_{k-1} < t|x_k,\tau_k ) = \left(\frac{1-e^{-t}}{1-e^{-\tau_k}}\right)^k$. Thus,
\begin{equation*}
    \mathbb{P}_{\mathrm{FHS}}(\tau_{k-1}|x_k,\tau_k ) = \mathbb{P}_{\mathrm{CTMC}}(\tau_{k-1} |x_k,\tau_k ) = \mathbb{P}_{\mathrm{CTMC}}(\tau_{k-1} |\tau_k ),
\end{equation*}
and we further note that $\tau_{k-1}$ does not depend on $x_k$ given $\tau_k$ for both FHS and reverse CTMC processes.

Next, note that FHS updates $i_{k-1}$ by randomly and uniformly selecting an index $i$ from $\calM(x_k) = \{l:x_k^l = \mask\}$ for each $k$, and hence
\begin{equation*}
    \mathbb{P}_{\mathrm{FHS}}(i_{k-1}|x_k,\tau_{k-1}) = \frac{1}{k} = \mathbb{P}_{\mathrm{CTMC}}(i_{k-1}|x_k,\tau_{k-1}).
\end{equation*}

Finally, given $x_k$, $\tau_{k-1}$, and $i_{k-1}$, FHS samples the predicted token $a_{k-1}$ by 
\begin{equation*}
    \mathbb{P}_{\mathrm{FHS}}(a_{k-1}|x_k,\tau_{k-1},i_{k-1}) = \mu_\theta^{i_{k-1}}(x_k,\tau_{k-1})[a_{k-1}].
\end{equation*}
Here, $\mathbb{P}_{\mathrm{FHS}}(a_{k-1}|x_k,\tau_{k-1},i_{k-1}) = \mu_\theta^{i_{k-1}}(x_k,\tau_{k-1})[a_{k-1}]$ is an estimator of the ground-truth distribution token-prediction distribution $\mathbb{P}_{\mathrm{CTMC}}(a_{k-1}|x_k,\tau_{k-1},i_{k-1}) = q_0^{i_{k-1}}(a_{k-1}|x_k^{UM})$.

The proof is now complete.    
\end{proof}

\section{Proof of \texorpdfstring{\Cref{thm: error bound of fhs}}{Theorem 5}}
\label{app:thm: error bound of fhs}

As shown by Proposition~\ref{prop: fhs exact}, both the FHS sampling process and the reverse CTMC have a decoupling structure: they share the same statistical distributions on the unmasking times $\tau_k$ and the indices $i_k$ of tokens that are unmasked. Their only difference lies in the mechanism used to sample the predicted token value $a_{k}$ for each $k$: FHS predicts $a_{k-1}$ via $\mu_\theta^{i_{k-1}}$ as an approximation of $q_0^{i_{k-1}}(\cdot |x_k^{UM})$ used in the ground-truth reverse CTMC. Because of this, the estimation error in $\mu_\theta$ affects only the sampling of the token value $a_k$ and does not propagate to $\tau_k$ or $i_k$. As a consequence, the subsequent convergence analysis of FHS can be carried out by technically accumulating only the errors arising from sampling the predicted token value $a_{k-1}$ for $k = {1,2,\ldots,d}$.

In the following, we first derive the KL divergence between the path distributions of the two processes, i.e., $\KL{\mathbb{P}_{\mathrm{CTMC}}(\mathrm{path})}{\mathbb{P}_{\mathrm{FHS}}(\mathrm{path})}$. Then, by applying the data processing inequality, we obtain the error bound on the distance between the ground-truth distribution $q_0(\cdot)$ and the output distribution of FHS, i.e.,  $\KL{q_0(x_0)}{\mathbb{P}_{\mathrm{FHS}}(x_0)}$.

Since the reverse absorbing-rate CTMC $(Y_t)_t$ with generator $\{\widetilde R_t\}$ is a time-inhomogeneous Markov jump process on the finite state space $[S]^d$, the strong Markov property at its unmasking times implies that, conditionally on $\{\tau_j\}_{j=d}^{k}$, $\{i_j\}_{j=d}^{k}$ and $\{a_j\}_{j=d}^{k}$, the joint law of $(\tau_{k-1}, i_{k-1}, a_{k-1})$ depends only on $(x_k,\tau_k)$ and is given by \Cref{lem: CTMC rate property}. Hence the joint path distribution can be factorized as
\begin{multline}
\label{eq: P path}
    \mathbb{P}_{\mathrm{CTMC}}(\mathrm{path}) = p_d(x_d, \tau_d) \prod_{k=d}^{1} \mathbb{P}_{\mathrm{CTMC}}(\tau_{k-1}|x_k,\tau_k) \\ \cdot \mathbb{P}_{\mathrm{CTMC}}(i_{k-1}|x_k, \tau_{k-1}) \mathbb{P}_{\mathrm{CTMC}}(a_{k-1}|x_k, \tau_{k-1}, i_{k-1}),
\end{multline}
and
\begin{multline}
\label{eq: Q path}
    \mathbb{P}_{\mathrm{FHS}}(\mathrm{path}) = p_d(x_d, \tau_d) \prod_{k=d}^{1} \mathbb{P}_{\mathrm{FHS}}(\tau_{k-1}|x_k,\tau_k) \\ \cdot \mathbb{P}_{\mathrm{FHS}}(i_{k-1}|x_k, \tau_{k-1}) \mathbb{P}_{\mathrm{FHS}}(a_{k-1}|x_k, \tau_{k-1}, i_{k-1}),
\end{multline}
where 
$$\mathbb{P}_{\mathrm{CTMC}}(a_{k-1}|x_k, \tau_{k-1}, i_{k-1}) = q_0^{i_{k-1}}(a_{k-1}|x_k^{UM}), $$ 
and
$$\mathbb{P}_{\mathrm{FHS}}(a_{k-1}|x_k, \tau_{k-1}, i_{k-1}) = \mu_\theta^{i_{k-1}}(x_k,\tau_{k-1})[a_{k-1}]$$ 
as shown in Proposition~\ref{prop: fhs exact}. Note that $p_d(x_d, \tau_d) = 1$ is the common initialization distribution with $x_d$ being all masked and $\tau_d$ satisfying $\alpha_{\tau_d} = 0$.

Dividing \eqref{eq: P path} by \eqref{eq: Q path} and applying the properties in \Cref{prop: fhs exact}, we have
\begin{equation*}
\frac{\mathbb{P}_{\mathrm{CTMC}}(\mathrm{path})}{\mathbb{P}_{\mathrm{FHS}}(\mathrm{path})} = \prod_{k=1}^{d}  \frac{\mathbb{P}_{\mathrm{CTMC}}(a_{k-1}|x_k, \tau_{k-1}, i_{k-1})}{\mathbb{P}_{\mathrm{FHS}}(a_{k-1}|x_k, \tau_{k-1}, i_{k-1})}.   
\end{equation*}

Then, we derive the KL divergence between the path distributions as follows:
\begin{align}
\label{eq: total error decomposition}
      &\KL{\mathbb{P}_{\mathrm{CTMC}}(\mathrm{path})}{\mathbb{P}_{\mathrm{FHS}}(\mathrm{path})} \nonumber \\
      &= \mathbb{E}_{\mathrm{path}\sim \mathbb{P}_{\mathrm{CTMC}}}\left[\log \frac{\mathbb{P}_{\mathrm{CTMC}}(\mathrm{path})}{\mathbb{P}_{\mathrm{FHS}}(\mathrm{path})}\right] \nonumber \\
      &= \mathbb{E}_{\mathrm{path}\sim \mathbb{P}_{\mathrm{CTMC}}}\left[\sum_{k=1}^{d}\log \left(\frac{\mathbb{P}_{\mathrm{CTMC}}(a_{k-1}|x_k, \tau_{k-1}, i_{k-1})}{\mathbb{P}_{\mathrm{FHS}}(a_{k-1}|x_k, \tau_{k-1}, i_{k-1})}\right)\right] \nonumber \\
      &= \sum_{k=1}^{d}\mathbb{E}_{\mathrm{path}\sim \mathbb{P}_{\mathrm{CTMC}}}\left[\log \left(\frac{\mathbb{P}_{\mathrm{CTMC}}(a_{k-1}|x_k, \tau_{k-1}, i_{k-1})}{\mathbb{P}_{\mathrm{FHS}}(a_{k-1}|x_k, \tau_{k-1}, i_{k-1})}\right)\right] \nonumber \\
      &= \sum_{k=1}^{d} \mathbb{E}_{\substack{(x_k, \tau_{k-1}, i_{k-1})\sim \mathbb{P}_{\mathrm{CTMC}} \\ (a_{k-1}|x_k, \tau_{k-1}, i_{k-1})\sim \mathbb{P}_{\mathrm{CTMC}}(a_{k-1}|x_k, \tau_{k-1}, i_{k-1})}}\sbrc{\log \left(  \frac{\mathbb{P}_{\mathrm{CTMC}}(a_{k-1}|x_k, \tau_{k-1}, i_{k-1})}{\mathbb{P}_{\mathrm{FHS}}(a_{k-1}|x_k, \tau_{k-1}, i_{k-1})} \right)} \nonumber \\
      &=  \sum_{k=1}^{d} \mathbb{E}_{(x_k, \tau_{k-1}, i_{k-1})\sim \mathbb{P}_{\mathrm{CTMC}}}\left[  \KL{\mathbb{P}_{\mathrm{CTMC}}(\cdot|x_k,\tau_{k-1},i_{k-1})}{ \mathbb{P}_{\mathrm{FHS}}(\cdot|x_k,\tau_{k-1},i_{k-1})}\right].
\end{align}

Thus, the KL divergence between two path distributions $\mathbb{P}_{\mathrm{CTMC}}$ and $\mathbb{P}_{\mathrm{FHS}}$ has been decomposed into summation of the expected KL divergence of two token-prediction distributions, which is the error arising from predicting $x_{k-1}^{i_{k-1}}$ via $\mu_\theta$ in FHS. 

We next analyze $$\mathbb{E}_{(x_k, \tau_{k-1}, i_{k-1})\sim \mathbb{P}_{\mathrm{CTMC}}}\left[\KL{ \mathbb{P}_{\mathrm{CTMC}}(\cdot|x_k,\tau_{k-1},i_{k-1})}{ \mathbb{P}_{\mathrm{FHS}}(\cdot|x_k,\tau_{k-1},i_{k-1})}\right]$$ for each $k$.
The key is to observe that, by Proposition~\ref{prop: fhs exact}, $\tau_{k-1}$ only depends on $\tau_k$ and that $i_{k-1}$ is conditionally uniform over the \(k\) masked coordinates given \(x_k\). Thus, such an independence structure directly enables us to simplify the KL divergence and connect the KL divergence with the score entropy.

The following lemma shows that $\alpha_{\tau_{k-1}}$ follows the Beta distribution.
\begin{lemma}
    \label{lem: alpha follows beta distribution}
The random variable $\alpha_{\tau_{k-1}}$ sampled by FHS algorithm follows a Beta distribution, i.e., for each $k = 1,2,\cdots,d$, 
    \[\alpha_{\tau_{k-1}} \sim Beta(d-k+1,k).\]
\end{lemma}
\begin{proof}
    See \Cref{app:lem: alpha follows beta distribution}.
\end{proof}

Thus, we have
\begin{align}
\label{eq: mu bar}
    \mathbb{E}_{\tau_{k-1}}\left[-\log(\mu_\theta^{i_{k-1}}(x_k,\tau_{k-1})[a])\right] &\stackrel{(i)}{=} \mathbb{E}_{z \sim \mathrm{Beta}(d-k+1,k)}\left[-\log(\mu_\theta^{i_{k-1}}(x_k,\alpha^{-1}(z))[a])\right] \nonumber \\
    &\stackrel{(ii)}{=} -\log(\bar{\mu}_\theta^{i_{k-1}}(x_k)[a]),
\end{align}
where $(i)$ follows from Lemma~\ref{lem: alpha follows beta distribution}, and in $(ii)$ we have defined
\begin{equation}
\label{eq: def-mu-bar}
\log \bar\mu_\theta(x_k)
:= \mathbb{E}_{z \sim \mathrm{Beta}(d-k+1,k)}
\big[ \log \mu_\theta(x_k, \alpha^{-1}(z)) \big].
\end{equation}
Now, given \eqref{eq: mu bar}, we can simplify the KL divergence above as
\begin{align}
    &\mathbb{E}_{(x_k,\tau_{k-1},i_{k-1}) \sim \mathbb{P}_{\mathrm{CTMC}}} \left[\KL{\mathbb{P}_{\mathrm{CTMC}}(\cdot | x_k,\tau_{k-1},i_{k-1})}{\mathbb{P}_{\mathrm{FHS}}(\cdot | x_k,\tau_{k-1},i_{k-1})}\right] \nonumber\\
    &= \mathbb{E}_{(x_k,\tau_{k-1},i_{k-1}) \sim \mathbb{P}_{\mathrm{CTMC}}} \Bigg[\mathbb{E}_{a \sim \mathbb{P}_{\mathrm{CTMC}} (\cdot| x_k,\tau_{k-1},i_{k-1})} \nonumber\\ 
    & \qquad \qquad \bigg[-\log\left(\mathbb{P}_{\mathrm{FHS}}(a|x_k,\tau_{k-1},i_{k-1})\right) +\log\left(\mathbb{P}_{\mathrm{CTMC}}\left(a|x_k,\tau_{k-1},i_{k-1}\right)\right)\bigg]\Bigg] \nonumber\\
    & \overset{(iii)}{=} \mathbb{E}_{(x_k,\tau_{k-1},i_{k-1}) \sim \mathbb{P}_{\mathrm{CTMC}}} \left[ \mathbb{E}_{a \sim q_0^{i_{k-1}}(\cdot |x_k^{UM})}\left[-\log\left(\mu_\theta^{i_{k-1}}(x_k,\tau_{k-1})[a]\right)+ \log(q_0^{i_{k-1}}(a|x_k^{UM}))\right]\right] \nonumber\\
    & = \mathbb{E}_{x_k} \mathbb{E}_{\tau_{k-1}|x_k} \mathbb{E}_{i_{k-1}|x_k,\tau_{k-1}} \nonumber\\
    &\qquad \qquad \mathbb{E}_{a\sim q_0^{i_{k-1}}(\cdot |x_k^{UM})}\left[-\log\left(\mu_\theta^{i_{k-1}}(x_k,\tau_{k-1})[a]\right)+ \log(q_0^{i_{k-1}}(a|x_k^{UM}))\right]\nonumber\\
    & \overset{(iv)}{=} \mathbb{E}_{x_k} \mathbb{E}_{\tau_{k-1}} \mathbb{E}_{i_{k-1}} \mathbb{E}_{a\sim q_0^{i_{k-1}}(\cdot |x_k^{UM})}\left[-\log\left(\mu_\theta^{i_{k-1}}(x_k,\tau_{k-1})[a]\right)+ \log(q_0^{i_{k-1}}(a|x_k^{UM})) \right] \nonumber\\ 
    & = \mathbb{E}_{x_k} \frac{1}{k}\sum_{l: x_k^l = \mask} \mathbb{E}_{a\sim q_0^{l}(\cdot |x_k^{UM})}  \mathbb{E}_{\tau_{k-1}} \left[-\log\left(\mu_\theta^{l}(x_k,\tau_{k-1})[a]\right)+ \log(q_0^{l}(a|x_k^{UM}))\right]\nonumber\\ 
    & \overset{(v)}{=}  \mathbb{E}_{x_k} \frac{1}{k} \sum_{l: x_k^l = \mask}\mathbb{E}_{a\sim q_0^l(\cdot |x_k^{UM})} \left[ -\log(\bar{\mu}_\theta^{l}(x_k)[a])\right] - \mathbb{E}_{x_k} \frac{1}{k}\sum_{l: x_k^l = \mask} H(q_0^l(\cdot|x_k^{UM})), \label{eq: single KL}
\end{align}
where we omit $\mathbb P_{\rm CTMC}$ where it is obvious. Here $(iii)$ follows from the last two claims in Proposition~\ref{prop: fhs exact}, and $(v)$ follows from \eqref{eq: mu bar}. For $(iv)$, we use two facts. First, by Proposition~\ref{prop: fhs exact}, $i_{k-1}$ is conditionally uniform over the masked coordinates of $x_k$, independently of $\tau_{k-1}$. Second, $\tau_{k-1}$ is independent of $x_k$. Note that Proposition~\ref{prop: fhs exact} alone only gives that $\tau_{k-1}$ is conditionally independent of $x_k$ given $\tau_k$, i.e., $P_{\rm CTMC}(\tau_{k-1}\mid x_k,\tau_k)=P_{\rm CTMC}(\tau_{k-1}\mid\tau_k)$, whereas $(iv)$ requires the marginal independence, which does not follow from conditional independence alone. It follows instead from the factorization (63): by Proposition~\ref{prop: fhs exact}, the $\tau$-kernels $P_{\rm CTMC}(\tau_{k-1}\mid\tau_k)$ do not involve $(i_j,a_j)_j$, while the kernels $P_{\rm CTMC}(i_{k-1}\mid x_k)=1/k$ and $P_{\rm CTMC}(a_{k-1}\mid x_k,i_{k-1})=q_0^{i_{k-1}}(a_{k-1}\mid x_k^{\rm UM})$ do not involve $(\tau_j)_j$. Hence the path density is the product of a function of $(\tau_j)_j$ and a function of $(i_j,a_j)_j$, so $(\tau_j)_j$ is independent of $(i_j,a_j)_j$. Since $x_k$ is a function of $(i_j,a_j)_{j\ge k}$, $\tau_{k-1}$ is independent of $x_k$, and therefore $\mathbb E_{\tau_{k-1}\mid x_k}[\cdot]=\mathbb E_{\tau_{k-1}}[\cdot]$. 

Note that the distribution of $x_k$ is
\[ \tilde q_k = \E_{x_0 \sim q_0} \tilde{q}_{k|0} (x_k|x_0), \]
where $\tilde{q}_{k|0} (x_k|x_0)$ denotes the discrete forward process which randomly and uniformly masks $k$ tokens of $x_0$.
We emphasize that a key step in the derivation above is that the two expectations $\mathbb{E}_{\tau_{k-1}}$ and $\mathbb{E}_{a\sim q_0^{i_{k-1}}(\cdot | x_k^{UM})}$ are interchangeable thanks to the true token-prediction distribution $q_0^{i_{k-1}}(\cdot | x_k^{UM})$ is independent of time $\tau_{k-1}$.

Now, to connect the first term above with the training objective, \citet[Proposition~3.1]{zheng2025fhs} shows that NELBO can be equivalently reformulated as 
\begin{equation*}
    \mathcal{L}_\infty(x_0) = - \sum_{k=1}^d \mathbb{E}_{x_k\sim \tilde q_{k|0}(\cdot \mid x_0)} \left[ \frac{1}{k} \sum_{l: x_k^l = \mask} \log \bar\mu_\theta^{(l)}(x_k)[x_0^l] \right],
\end{equation*}
where $\log \bar\mu_\theta(x_k)$ is defined in \eqref{eq: def-mu-bar}.
Thus,
\begin{align}
&\mathbb{E}_{x_{0} \sim q_0}\left[\mathcal{L}_{\infty}\left(x_{0}\right)\right] \nonumber \\
&\qquad = \mathbb{E}_{x_{0} \sim q_0} \sum_{k=1}^{d} \mathbb{E}_{x_k\sim \tilde q_{k|0}(\cdot \mid x_0)} \frac{1}{k} \sum_{l: x_k^l = \mask} -\log \bar{\mu}_{\theta}\left(x_k\right)[x_{0}^{l}] \nonumber \\
&\qquad = \sum_{k=1}^{d} \frac{1}{k} \mathbb{E}_{x_0 \sim q_0, x_k\sim \tilde q_{k|0}(\cdot \mid x_0)} \sum_{l: x_k^l = \mask} -\log \bar{\mu}_{\theta}\left(x_{k}\right)[x_{0}^{l}] \nonumber \\
&\qquad = \sum_{k=1}^{d} \frac{1}{k} \mathbb{E}_{x_k \sim \tilde q_k, x_{0} \sim \tilde q_{0|k}(\cdot \mid x_k)} \sum_{l: x_k^l = \mask} -\log \bar{\mu}_{\theta}\left(x_k\right)[x_{0}^{l}] \nonumber \\
&\qquad \stackrel{(vi)}{=} \sum_{k=1}^{d} \frac{1}{k} \mathbb{E}_{x_k \sim \tilde q_k} \mathbb{E}_{x_{0} \sim q_0(\cdot \mid x_k^{UM})} \sum_{l: x_k^l = \mask} -\log \bar{\mu}_{\theta}\left(x_k\right)[x_{0}^{l}] \nonumber \\
&\qquad =  \sum_{k=1}^{d} \frac{1}{k} \mathbb{E}_{x_k} \sum_{l: x_k^l = \mask} \mathbb{E}_{a\sim q_0^l(\cdot | x_k^{UM})} \left[-\log \bar{\mu}_{\theta}\left(x_k\right)\left[a\right]\right], \label{eq: L_infty and error}
\end{align}
which is precisely the first term in \eqref{eq: single KL}.
Here $(vi)$ follows because 
\begin{align*}
    \tilde q_{0|k}(x_0\mid x_k) &= \frac{ q_0(x_0)\tilde q_{k|0}(x_k\mid x_0) }{ \tilde q_k(x_k) } \\
    &\propto q_0(x_0) \frac{1}{\binom{d}{k}} \ind{x_0^{\calM^c(x_k)} = x_k^{\calM^c(x_k)}} \\
    &\propto q_0(x_0) \ind{x_0^{\calM^c(x_k)}=x_k^{\calM^c(x_k)}} \\
    &\propto q_0(x_0 | x_k^{UM}),
\end{align*}
where $\calM^c(x)$ is the set of indices of un-masked tokens in $x$. Since both are probability distributions, $\tilde q_{0|k}(x_0\mid x_k) \propto q_0(x_0 | x_k^{UM}) \implies \tilde q_{0|k}(x_0\mid x_k) = q_0(x_0 | x_k^{UM})$.

Finally, plugging \eqref{eq: single KL} and \eqref{eq: L_infty and error} into \eqref{eq: total error decomposition}, we have
\begin{align*}
&\KL{\mathbb{P}_{\mathrm{CTMC}}(\mathrm{path})}{\mathbb{P}_{\mathrm{FHS}}(\mathrm{path})} \\
        & = -\sum_{k = 1}^d \mathbb{E}_{x_k \sim \tilde q_k} \frac{1}{k}\sum_{l: x_k^l = \mask}\mathbb{E}_{a\sim q_0^l(\cdot |x_k^{UM})} \left[\log(\bar{\mu}_\theta^{l}(x_k)[a])\right]\\
        &\quad \ - \sum_{k = 1}^d\mathbb{E}_{x_k \sim \tilde q_k} \frac{1}{k}\sum_{l: x_k^l = \mask} H(q_0^l(\cdot|x_k^{UM}))\\
        & = \mathbb{E}_{x_0\sim q_0}\left[\mathcal{L}_{\infty}(x_0)\right]  - \sum_{k = 1}^d\mathbb{E}_{x_k \sim \tilde q_k} \frac{1}{k}\sum_{l: x_k^l = \mask} H(q_0^l(\cdot|x_k^{UM}))\\
        & = \mathbb{E}_{x_0\sim q_0}\left[\mathcal{L}_{\infty}(x_0)\right]  - \sum_{k = 1}^d \frac{1}{k} \mathbb{E}_{x_0 \sim q_0}\mathbb{E}_{\calM_k}
\sum_{l\in \calM_k}
H\left(q_0^l(\cdot|x_0^{\calM_k^c})\right).
\end{align*}

Applying the equivalence between integrated $\mathcal{L}_{SE}$ and expected NELBO in Proposition~\ref{prop: equivalence between lse and nelbo} and \Cref{ass: fhs integrated Lse error}, we have
\begin{equation*}
    \KL{\mathbb{P}_{\mathrm{CTMC}}(\mathrm{path})}{\mathbb{P}_{\mathrm{FHS}}(\mathrm{path})} \le \varepsilon''_{\mathrm{score}}.
\end{equation*}

Using the data processing inequality, we finally come to the conclusion:
\begin{equation*}
    \KL{q_0(x_0)}{\mathbb{P}_{\mathrm{FHS}}(x_0)}\le \KL{\mathbb{P}_{\mathrm{CTMC}}(\mathrm{path})}{\mathbb{P}_{\mathrm{FHS}}(\mathrm{path})}  \le \varepsilon''_{\mathrm{score}}.
\end{equation*}

The proof is now complete.

\section{Proof of \texorpdfstring{\Cref{prop: fhs upper bound is tight}}{Theorem 6}}
\label{app:prop: fhs upper bound is tight}

We construct both a target distribution $q_0$ and a special form of the estimated predictive model $\mu_\theta$, which yield a lower error bound. 

To proceed, fix $\bm{a},\bm{b} \in [S]^d$ where $\bm{b}^i\neq \bm{a}^i$ and $\bm{a}^i,\bm{b}^i \neq \mask$ for all $i \in [d]$. Let $q_0 = \bm{\delta}_{\bm{a}}$. Also, for any $x_t$, $t$, and $i$, let $\mu_\theta$ be as follows:
\begin{equation*}
  \mu_\theta^i(x_t,t)[y^i] = \begin{cases}
        1-\rho, &\text{ if } y^i = \bm{a}^i,\\
        \rho, &\text{ if } y^i =\bm{b}^i.
    \end{cases}
\end{equation*}
where we choose $\rho = 1-e^{-\varepsilon''_{\mathrm{score}}/d}$.

We next show that the above constructed $\mu_\theta$ satisfies
\Cref{ass: fhs integrated Lse error}, i.e., $\int_0^{+\infty} \mathcal{L}_{SE}(s;t)\, \mathrm{d}t \le \varepsilon''_{\mathrm{score}}$. First note that under $q_0 = \bm{\delta}_{\bm{a}}$, we have $q_0^l(\cdot|x_0^{\calM_k^c}) = \bm{\delta}^l_{\bm{a}^l}$. Then, we can apply \Cref{prop: equivalence between lse and nelbo} and compute the integrated $\mathcal{L}_{SE}$ as follows.
\begin{align*}
        \int_0^{+\infty} \mathcal{L}_{SE}(s;t)\, \mathrm{d}t & =\mathbb{E}_{x_0 \sim q_0}[\mathcal{L}_\infty(x_0)]
    - \sum_{k = 1}^d \frac{1}{k} \mathbb{E}_{x_0\sim q_0}\mathbb{E}_{\calM_k}\sum_{l\in \calM_k}H\left(q_0^l(\cdot|x_0^{\calM_k^c})\right) \\
& = \mathcal{L}_\infty(\bm{a}) - 0  \\
& = \int_0^{+\infty}\frac{-
\alpha_t}{1-\alpha_t}\mathbb{E}_{q_{t|0}(x_t|\bm{a})}\left[\sum_{l:x_t^l = \mask}\log\left(\mu_\theta^l(x_t,t)[\bm{a}^l]\right)\right]\mathrm{d}t \\
& = \int_0^{+\infty}\frac{-
\alpha_t}{1-\alpha_t}\mathbb{E}_{q_{t|0}(x_t|\bm{a})}\left[\sum_{l:x_t^l = \mask}\log\left(1-\rho\right)\right]\mathrm{d}t\\
& = \int_0^{+\infty}\frac{-
\alpha_t}{1-\alpha_t} \log\left(1-\rho\right) \mathbb{E}_{q_{t|0}(x_t|\bm{a})}\left[m(x_t)\right]\mathrm{d}t\\
& = \int_0^{+\infty}\frac{-
\alpha_t}{1-\alpha_t} \log\left(1-\rho\right) d(1-\alpha_t)\mathrm{d}t\\
& =  -d \log(1-\rho) = \varepsilon''_{\mathrm{score}},
\end{align*}
where $m(x_t)$ denotes the number of $\mask$ tokens in $x_t$, and the last inequality follows from the choice of $\rho$.

Now, under FHS, since each token is predicted exactly once and $\mu_\theta$ predicts the correct token $\bm{a}^i$ at each position $i$ with the probability $(1-\rho)$, we have $\mathbb{P}_{\mathrm{FHS}}(\bm{a}) = (1-\rho)^d$.
Then, since $q_0 = \bm{\delta}_{\bm{a}}$, it follows that 
\begin{equation}
\label{eq: constructed fhs KL}
    \KL{q_0}{\mathbb{P}_{\mathrm{FHS}}} = - \log\left(\mathbb{P}_{\mathrm{FHS}}(a)\right) = -d \log(1-\rho)= \varepsilon''_{\mathrm{score}},
\end{equation}
where the last equality follows from the choice of $\rho$. 

The proof is now complete. Note that this special case with such a $q_0$ and $\mu_\theta$ indicates that the upper bound argued in \Cref{thm: error bound of fhs} is tight.

\section{Proofs of Auxiliary Lemmas}

In this section, we provide all the proofs for the auxiliary lemmas used in this paper.

\subsection{Proof of 
\texorpdfstring{\Cref{lem:euler-new-mesh}}{Lemma 7}}
\label{app:lem:euler-new-mesh-proof}

Write $\eta_k := t_{k+1} - t_k = u_k - u_{k+1}$. The extension uses constant steps above $T$ and geometric steps below
$\delta$, so the stated step-size rule holds throughout the extended grid.
With the transformed grid in \eqref{eq:euler-new-grid}, we have
\begin{equation} \label{eq:euler-new-mesh-help1}
    w_k=e^{-u_{k+1}} - e^{-u_{k}} = e^{-u_{k}} (e^{u_{k}-u_{k+1}} - 1) = a_k(e^{\eta_k}-1),
\end{equation}
and note that
$(1-\kappa)\eta_k\le\eta_{k+1}\le\eta_k\le\kappa$. Thus,
\[
 1-\kappa \leq e^{\eta_k} \frac{\eta_{k+1}}{\eta_k e^{\eta_k}}
 \stackrel{(i)}{\leq} e^{\eta_k}\frac{e^{\eta_{k+1}}-1}{e^{\eta_k}-1}
 =\frac{w_{k+1}}{w_k}\leq e^\kappa,
\]
where $(i)$ follows because
$e^{\eta_{k+1}}-1\ge\eta_{k+1}$ and
$e^{\eta_k}-1 = \int_0^{\eta_k}e^v \d v \le\eta_ke^{\eta_k}$, and $(ii)$ follows because from \eqref{eq:euler-new-mesh-help1} we have $\frac{w_{k+1}}{w_k} = \frac{a_{k+1} (e^{\eta_{k+1}}-1)}{a_k (e^{\eta_k}-1)} = e^{\eta_k}\frac{e^{\eta_{k+1}}-1}{e^{\eta_k}-1} \le e^{\eta_k} \le e^\kappa$.
Similarly, shifting the index gives
$e^{-\kappa}\le w_{k-1}/w_k\le(1-\kappa)^{-1}$. Thus,
\begin{align*}
    |w_k-w_{k-1}| &= w_k\abs{1-\frac{w_{k-1}}{w_k}} \\
    &\le w_k\max\cbrc{1-e^{-\kappa}, \frac1{1-\kappa}-1} \\
    &\le w_k \max\cbrc{\kappa,\frac{\kappa}{1-\kappa}} \\
    &\le 2\kappa w_k.
\end{align*}
This establishes the first bound.

For the second bound,
\begin{align*}
    w_k &= a_k (e^{\eta_k}-1) \le a_k \eta_k e^{\eta_k} \le a_k \kappa \min\{1,u_k\} e^\kappa \\
    &\stackrel{(ii)}{\le} a_k \kappa (1-e^{-1})^{-1} (1-e^{-u_k}) e^{1/2}\\
    &= C a_k \kappa (1-a_k),
\end{align*}
where $C > 0$ is a universal constant. Here $(ii)$ follows because $\kappa \leq 1/2$ and $\min\{1,u\}/(1-e^{-u})\le (1-e^{-1})^{-1}$.
Since $u_{k+1}\ge(1-\kappa)u_k$, concavity gives
$1-a_{k+1}\ge(1-\kappa)(1-a_k)$.
Thus, for $r\in[a_k,a_{k+1}]$, we have
$a_k(1-a_k)\le r(1-r)/(1-\kappa) \leq 2 r(1-r)$.
This proves the weighted bound, and $r(1-r)\le1/4$ gives the
stated maximum-step bound. The proof is now complete.

\subsection{Proof of \texorpdfstring{\Cref{lem:euler-new-domination}}{Lemma 8}}
\label{app:lem:euler-new-domination-proof}

First, each factor in \eqref{eq:def-euler-new-gamma} retains the fraction of the
true conditional probability supported by the learned probability
of the target token, so obviously $0\le\gamma_\Pi(x)\le1$.
We prove the lower bound by retaining Euler trajectories on which
every coordinate predicts its target token on its first attempt
within the finite grid.

To obtain a lower bound on $p_{\mathrm{Euler}}(x)$, we first fix $x$ and $\Pi = (K_1,\dots,K_d)$. Let $A_\Pi(x)$ be the event that token $i$ is unmasked at step $K_i$ and is assigned the value $x_i$. Note that, by \eqref{eq:euler-new-euler-transition}, the Euler unmasking probability at step $k$ is 
\begin{align*}
    p_{k+1|k}(x_{k+1}^i \neq \mask \mid x_k) &= \sum_{a \neq \mask} \eta_k c_{u_k} \mu_\theta^i(x_k,u_k)[a] = \eta_k c_{u_k}, \\ 
    p_{k+1|k}(x_{k+1}^i = a \mid x_k) &= \eta_k c_{u_k} \mu_\theta^i(x_k,u_k)[a].
\end{align*}
Here we note that $\mu_\theta^i(x_k,u_k)[\cdot]$ is normalized and $\mu_\theta^i(x_k,u_k)[\mask] = 0$ by the underlying parameterization.
Recall $C_i$ from \eqref{eq:euler-new-schedule}. Thus, by the chain rule across steps and conditional independence of the coordinate updates,
\begin{align*}
    \mathbb P_{\mathrm{Euler}}(A_\Pi(x)) &= \prod_{i=1}^d \brc{ p_{K_i+1|K_i}(x_{K_i+1}^i = x^i \mid x_{K_i}) \prod_{j=0}^{K_i-1} p_{j+1|j}(x_{j+1}^i = \mask \mid x_j) } \\
    &= \prod_{i=1}^d
    \brc{ \mu_\theta^i(x^{C_i}\oplus\mask,u_{K_i})[x^i] \brc{ \eta_{K_i} c_{u_{K_i}} \prod_{j=0}^{K_i-1} (1 - \eta_j c_{u_j}) } } \\
    &\stackrel{(i)}{\geq} \prod_{i=1}^d
    \brc{ \mu_\theta^i(x^{C_i}\oplus\mask,u_{K_i})[x^i] e^{-\kappa} \frac{w_{K_i}}{1-a_{K_i}} \prod_{j=0}^{K_i-1} \frac{1-a_{j+1}}{1-a_j} } \\
    &= \frac{e^{-d \kappa}}{(1-a_0)^d} \prod_{i=1}^d
    \brc{ \mu_\theta^i(x^{C_i}\oplus\mask,u_{K_i})[x^i] w_{K_i} },
\end{align*}
where $(i)$ follows because $\frac{w_k}{1-a_k} = \frac{a_k(e^{\eta_k}-1)}{1-a_k} = (e^{\eta_k}-1)c_{u_k}$ and thus
\begin{align*}
 1- \eta_j c_{u_j} &\ge 1- (e^{\eta_j}-1) c_{u_j} = 1-\frac{w_j}{1-a_j}
 = \frac{1-a_{j+1}}{1-a_j},\\
 \eta_kc_{u_k}
 &=\frac{\eta_k}{e^{\eta_k}-1}\frac{w_k}{1-a_k}
 \ge e^{-\eta_k}\frac{w_k}{1-a_k}
 \ge e^{-\kappa}\frac{w_k}{1-a_k}.
\end{align*}
We also note that 
\[ \prod_{i=1}^d w_{K_i} = W(\Pi), \]
because, by the definition of $K_i$ and $U_i$ in \eqref{eq:euler-new-schedule}, $U_i$ is uniform over $(0,1)$ and thus $\mathbb P(K_i=k)=a_{k+1}-a_k=w_k$.

Now, recall that $G:=\{K_i\in\{0,\ldots,N-1\}\text{ for every }i\}$ as the ``good'' events where all the unmasking happens within the finite interval $u_k \in [\delta,T]$. (And we simply set $\gamma_\Pi(x) = 0$ on $G^c$.) With this, we have
\begin{align*}
 p_{\mathrm{Euler}}(x) &\geq \sum_{\Pi\in G} \mathbb P_{\mathrm{Euler}}(A_\Pi(x)) \\
 &\ge \frac{e^{-d\kappa}}{(1-a_0)^d}
 \sum_{\Pi\in G} W(\Pi)\prod_i
 \mu_\theta^i(x^{C_i}\oplus\mask,u_{K_i})[x^i]\\
 &=\frac{e^{-d\kappa}}{(1-a_0)^d}
 \E_{\Pi \sim W}\left[
 \ind{G} \prod_i
 \mu_\theta^i(x^{C_i}\oplus\mask,u_{K_i})[x^i]\right]\\
 &\ge e^{-d\kappa}\E_{\Pi \sim W}[q_\Pi(x)\gamma_\Pi(x)],
\end{align*}
where the last inequality uses $(1-a_0)^{-d}\ge1$ and, with the definition of $\gamma_\Pi$ in \eqref{eq:def-euler-new-gamma} and noting that $q_\Pi(x) = \prod_i q_0^i(x^i\mid x^{C_i})$, we have
\[
 q_\Pi(x)\gamma_\Pi(x)
 =\ind{G} \prod_i\min\left\{
 q_0^i(x^i\mid x^{C_i}),
 \mu_\theta^i(x^{C_i}\oplus\mask,u_{K_i})[x^i]\right\}.
\]
The proof is now complete.

\subsection{Proof of \texorpdfstring{\Cref{lem:euler-new-entropy}}{Lemma 9}}
\label{app:lem:euler-new-entropy-proof}

First, note that under $\Pi \sim W$,
\begin{align}
 &\mathbb P_{\Pi\sim W}(|C_i|=n) \nonumber \\
 &=\sum_{k\in\mathbb Z} \mathbb P_{\Pi\sim W}(K_i = k) \mathbb P_{\Pi\sim W}\brc{ |\{j \in [d]\setminus\{i\}: K_j < k\}| = n \mid K_i = k}  \nonumber \\
 &\stackrel{(i)}{=} \sum_{k\in\mathbb Z}w_k
   \binom{d-1}{n}a_k^n(1-a_k)^{d-1-n},
 \label{eq:euler-new-omega}
\end{align}
where $(i)$ follows because $\mathbb P_{\Pi\sim W}(K_i = k) = \mathbb P_{\Pi\sim W}(U_i \in [a_k, a_{k+1})) = a_{k+1} - a_k = w_k$ and, on $K_i=k$, each other coordinate belongs to $C_i$ independently with probability $a_k$.
We will compare the sum in
\eqref{eq:euler-new-omega} with the corresponding integral:
\begin{align*}
 \int_0^1\binom{d-1}{n}r^n(1-r)^{d-1-n}\,\mathrm dr
 &=\frac1d.
\end{align*}

Replacing the left-endpoint sum by its trapezoidal sum produces
a difference controlled by summation by parts. On each interval
$[a_k,a_{k+1}]$, the integrated trapezoidal error is at most
$w_k^2/8$ times the integral of the absolute second derivative.
Combining these two comparisons gives
\begin{align}
 &\left|\mathbb P_{\Pi\sim W}(|C_i|=n)-\frac1d\right|
 \nonumber\\
 &\le\frac12\sum_{k\in\mathbb Z}|w_k-w_{k-1}|
       \binom{d-1}{n}a_k^n(1-a_k)^{d-1-n}\nonumber\\
 &\quad+\frac18\sum_{k\in\mathbb Z}w_k^2
   \int_{a_k}^{a_{k+1}}
   \left|\frac{\mathrm d^2}{\mathrm d r^2}
     \left[\binom{d-1}{n}r^n(1-r)^{d-1-n}\right]\right|
   \,\mathrm dr\nonumber\\
 &\stackrel{(ii)}{\leq} C\kappa\,\mathbb P_{\Pi\sim W}(|C_i|=n)\nonumber\\
 &\quad+C\kappa^2\int_0^1 r^2(1-r)^2
   \left|\frac{\mathrm d^2}{\mathrm d r^2}
     \left[\binom{d-1}{n}r^n(1-r)^{d-1-n}\right]\right|
   \,\mathrm dr,
 \label{eq:euler-new-rank-quadrature}
\end{align}
where $(ii)$ follows from the two bounds in \Cref{lem:euler-new-mesh}, including $w_k\le C\kappa r(1-r)$ for $r\in[a_k,a_{k+1}]$.
Note that the boundary terms in summation by parts vanish because $w_k\to0$ at both ends of the extended grid and the polynomial is bounded on $[0,1]$.

It remains to bound the integral in \eqref{eq:euler-new-rank-quadrature} uniformly in $d$ and $n$. Write $f(r):=\binom{d-1}{n}r^n(1-r)^{d-1-n}$. Indeed, we have
\begin{align*}
 &\int_0^1 r^2(1-r)^2
   \left|f''(r)\right|
   \d r\\
 &\stackrel{(iii)}{\leq} \int_0^1 f(r)
   \brc{(n-(d-1)r)^2+n(1-r)^2+(d-1-n)r^2}
   \,\mathrm dr\\
&= \int_0^1 f(r)
   \brc{n(n+1)-2nd r+d(d-1)r^2}
   \,\mathrm dr\\
 &\stackrel{(iv)}{=} \frac{n(n+1)}d -\frac{2n(n+1)}{d+1} +\frac{(d-1)(n+1)(n+2)}{(d+1)(d+2)}\\
&=\frac{n+1}{d(d+1)(d+2)}
\brc{ n(d+1)(d+2)-2nd(d+2)+d(d-1)(n+2) }\\
&=\frac{2(d-1)(n+1)(d-n)}{d(d+1)(d+2)} \\
 &\stackrel{(v)}{\leq} \frac12.
\end{align*}
We now explain the series of equations above. To see $(iii)$, note that
\[
\frac{f'(r)}{f(r)} = (\log f(r))'
=\frac{n}{r}-\frac{d-1-n}{1-r},
\]
and also that $\brc{\frac{f'(r)}{f(r)}}' = \frac{f''(r)}{f(r)} - \brc{\frac{f'(r)}{f(r)}}^2$. Thus,
\[
\frac{f''(r)}{f(r)} = \brc{\frac{f'(r)}{f(r)}}^2 + \brc{\frac{f'(r)}{f(r)}}'
=
\left(\frac{n}{r}-\frac{d-1-n}{1-r}\right)^2
-\frac{n}{r^2}-\frac{d-1-n}{(1-r)^2}.
\]
Thus, by the triangle inequality,
\begin{align*}
r^2(1-r)^2|f''(r)|
&=f(r)\left|
(n-(d-1)r)^2-n(1-r)^2-(d-1-n)r^2
\right|\\
&\le f(r)\brc{(n-(d-1)r)^2+n(1-r)^2+(d-1-n)r^2},
\end{align*}
This shows $(iii)$.
Also, $(iv)$ follows from the beta-integral identity, where
\[
\int_0^1 f(r)\,\mathrm dr=\frac1d,\qquad
\int_0^1 rf(r)\,\mathrm dr=\frac{n+1}{d(d+1)},\qquad
\int_0^1 r^2f(r)\,\mathrm dr
=\frac{(n+1)(n+2)}{d(d+1)(d+2)}.
\]
Finally, $(v)$ follows from $(n+1)(d-n)\le(d+1)^2/4$.
Thus, from \eqref{eq:euler-new-rank-quadrature}, we further have
\begin{align*}
 \mathbb P_{\Pi\sim W}(|C_i|=n)
 &\le\frac{1/d+C\kappa^2}{1-C\kappa}
 \le\frac{1+C'\kappa}{d}
\end{align*}
for $d\kappa\le c$ where $c$ is sufficiently small. The proof is now complete.

\subsection{Proof of \texorpdfstring{\Cref{lem:euler-new-prediction}}{Lemma 10}}
\label{app:lem:euler-new-prediction-proof}

For this proof, write
\begin{align}
    b_k:=\frac{w_k}{1-a_k}
    =\frac{e^{\eta_k}-1}{e^{u_k}-1}
    =(e^{\eta_k}-1)c_{u_k}.
    \label{eq:euler-new-clock-prob}
\end{align}
Recall that each $K_i$ is ordered in reverse time, i.e., a larger $K_i$ means a later unmasking event. To evaluate each expectation, let us define, for each $K_j$,
\begin{align*}
    Z_k^j=
    \begin{cases}
        x^j, & K_j<k,\\
        \mask, & K_j\ge k.
    \end{cases}
\end{align*}
Thus, by \eqref{eq:forward-factorize} and
\eqref{eq:absorb-fwd-cond-prob}, we have $Z_k\sim q_{u_k}$ when $x \sim q_0$ and $\Pi \sim W$.
Moreover, on $\{K_i=k\}$, the visible tokens are exactly
$x^{C_i}$, so $x^{C_i}\oplus\mask=Z_k$.

As follows, let $\mathbb P$ denote the joint distribution of $x\sim q_0$ and $\Pi\sim W$, which are sampled independently. Note that $U_j$ is uniform over $(0,1)$, giving
\begin{align*}
    \mathbb P(K_j<k) &= \mathbb P_{\Pi\sim W}(U_j < a_k) = a_k = e^{-u_k}.
\end{align*}
Also, the events $\{ K_j < k \}$ are independent across coordinates and
independent of $x$, and thus, for $z$ with $q_{u_k}(z)>0$ and $i\in\calM(z)$,
\begin{align}
 & \mathbb P(K_i=k\mid Z_k=z) \nonumber \\
 &= \mathbb P(K_i=k\mid K_i\ge k) \nonumber \\
 &=\mathbb P(U_i \in [a_k,a_{k+1}) \mid U_i \ge a_k) \nonumber \\
 &=\frac{a_{k+1}-a_k}{1-a_k} = \frac{e^{\eta_k}-1}{e^{u_k}-1} \nonumber \\
 &=(e^{\eta_k}-1)c_{u_k}. \label{eq:lem:euler-new-prediction-bk}
\end{align}
Now, conditioning on $Z_k$, using $Z_k\sim q_{u_k}$ and
$x^{C_i}\oplus\mask=Z_k$ on $\{K_i=k\}$, gives
\begin{align*}
    &\E_{x\sim q_0,\,\Pi\sim W}
    \sbrc{
        \ind{K_i=k}
        \brc{
            1-
            \frac{\mu_\theta^i(x^{C_i}\oplus\mask,u_k)[x^i]}
                 {q_0^i(x^i\mid x^{C_i})}
        }^+
    }\\
    &= \E_{z\sim q_{u_k}} \sbrc{
    \ind{i\in\calM(z)}
    \sum_{a:q_0^i(a\mid z^{UM})>0}
    \mathbb P\brc{K_i=k,\ x^i=a\mid Z_k=z}
    \brc{
        1-\frac{\mu_\theta^i(z,u_k)[a]}
                {q_0^i(a\mid z^{UM})}
    }^+ } \\
    &= \E_{z\sim q_{u_k}} \Bigg[
    \ind{i\in\calM(z)}
    \\
    &\qquad \cdot \sum_{a:q_0^i(a\mid z^{UM})>0}
    \mathbb P(K_i=k\mid Z_k=z) \mathbb P(x^i=a\mid Z_k=z,K_i=k)\brc{
        1-\frac{\mu_\theta^i(z,u_k)[a]}
                {q_0^i(a\mid z^{UM})}
    }^+ \Bigg] \\
    &\stackrel{(i)}{=} (e^{\eta_k}-1)c_{u_k} \E_{z\sim q_{u_k}}\sbrc{
        \ind{i\in\calM(z)}
        \sum_{a:q_0^i(a\mid z^{UM})>0}
        q_0^i(a\mid z^{UM})
        \brc{
            1-\frac{\mu_\theta^i(z,u_k)[a]}
                    {q_0^i(a\mid z^{UM})}
        }^+ }\\
    &\stackrel{(ii)}{=} (e^{\eta_k}-1)c_{u_k}\E_{z\sim q_{u_k}}
    \sbrc{
        \ind{i\in\calM(z)}
        \sum_{a\ne\mask}
        \brc{
            q_0^i(a\mid z^{UM})-\mu_\theta^i(z,u_k)[a]
        }^+
    },
\end{align*}
where $(i)$ follows from \eqref{eq:lem:euler-new-prediction-bk} and that the conditional distribution
of $x^i$ given $Z_k=z$ is $q_0^i(\cdot \mid z^{UM})$ (because masking is independent of token values, and also $K_i$ has no effect here), and $(ii)$ follows because those $a:q_0^i(a\mid z^{UM})=0$ contributes only zero to the positive part.

Summing this identity over coordinates $i \in [d]$ and finite steps $k = 0,\dots,N-1$, we have
\begin{align*}
    &\sum_{k=0}^{N-1}(e^{\eta_k}-1)c_{u_k}
    \E_{z\sim q_{u_k}}
    \sbrc{
        \sum_{i\in\calM(z)}\sum_{a\ne\mask}
        \brc{
            q_0^i(a\mid z^{UM})-\mu_\theta^i(z,u_k)[a]
        }^+
    } \\
    &\le
    \sum_{k=0}^{N-1}(e^{\eta_k}-1)c_{u_k}
    \E_{z\sim q_{u_k}}
    \sbrc{
        \sum_{i\in\calM(z)}\sum_{a\ne\mask}
        \abs{
            q_0^i(a\mid z^{UM})-\mu_\theta^i(z,u_k)[a]
        }
    } \\
    &\stackrel{(iii)}{\leq}
    e^\kappa
    \sum_{k=0}^{N-1}\eta_k c_{u_k}
    \E_{z\sim q_{u_k}}
    \sbrc{
        \sum_{i\in\calM(z)}\sum_{a\ne\mask}
        \abs{
            q_0^i(a\mid z^{UM})-\mu_\theta^i(z,u_k)[a]
        }
    } \\
    &\stackrel{(iv)}{=} e^\kappa\mathcal L_{\mathrm{TV}}(s),
\end{align*}
where $(iii)$ follows because $e^{\eta_k}-1\le e^\kappa\eta_k$, and $(iv)$ follows from the true-score identity in \eqref{eq: equivalence-between-true-score-and-clean-distribution} and the parameterization in \eqref{eq:equivalence-between-score-and-mu}. The proof is now complete.

\subsection{Proof of \texorpdfstring{Lemma~\ref{lem: CTMC rate property}}{Lemma 13}}
\label{app:lem: CTMC rate property}

\noindent\textbf{(1) Instantaneous total rate of leaving $x$ at time $t$.} For any $t$ and $x\neq y$, the reverse CTMC process
satisfies the local transition given by
\begin{equation}
\label{eq:local-forward}
\mathbb{P}\cbrc{Y_{t-dt}=y \mid Y_t=x}
= \widetilde{R}_t(x,y)\,dt + o(dt).
\end{equation}
Taking summation over all $y\neq x$, we obtain
\begin{equation}
\label{eq: not leave x}
  \mathbb{P}\cbrc{Y_{t-dt}\neq x \mid Y_t=x}
= \sum_{y\neq x} \widetilde{R}_t(x,y)\,dt + o(dt)
= \Lambda(t,x)\,dt + o(dt),  
\end{equation}
where
\begin{equation*}
    \Lambda(t,x)
= \sum_{y\neq x} \widetilde{R}_t(x,y).
\end{equation*}
Thus $\Lambda(t,x)$ is the instantaneous rate of leaving
state $x$ at time $t$. 

\medskip
\noindent\textbf{(2) Unmasking time $\tau$ in state $x$.}
Fix $t$ and $x$, and define the survival function
\begin{equation*}
   S(h) := \mathbb{P}\cbrc{\tau < t-h \mid Y_t=x}, \qquad h\ge 0, 
\end{equation*}
where
\begin{equation*}
    \tau := \sup\{s<t : Y_s \neq x\}
\end{equation*}
is the first unmasking time after $t$ for the reverse process indexed by the forward time. 

To derive a differential equation for $S(h)$, let $\delta>0$ be a small constant, and $\mathbf{1}_{\{\cdot\}}$ denote the indicator function. Using the Markov property at time $t-h$, we have
\begin{equation*}
    \begin{aligned}
        S(h+\delta) & = \mathbb{P}\cbrc{\tau<t-h-\delta | Y_t=x}\\
        & = \mathbb{E}_\tau\left[\mathbf{1}_{\{\tau<t-h\}} \mathbf{1}_{\{\tau<t-h-\delta\}}|Y_t=x \right]\\
        & = \mathbb{E}_{Y_{t-h}}\left[\mathbb{E}_{\tau|Y_{t-h}}\left[\mathbf{1}_{\{\tau<t-h\}} \mathbf{1}_{\{\tau<t-h-\delta\}}|Y_{t-h}, Y_t = x\right]|Y_t = x\right]\\
        & \overset{(i)}{=} \mathbb{E}_{Y_{t-h}}\left[\mathbb{E}_{\tau|Y_{t-h}}\left[\mathbf{1}_{\{\tau<t-h\}} \mathbf{1}_{\{\tau<t-h-\delta\}}|Y_{t-h}=x, Y_t = x\right]|Y_t = x\right]\\
        & \overset{(ii)}{=} \mathbb{E}_{Y_{t-h}}\left[
        \mathbb{E}_{\tau|Y_{t-h}} \sbrc{\mathbf{1}_{\{\tau<t-h\}} | Y_{t-h}=x, Y_t=x}\mathbb{P}\cbrc{\tau < t-h-
    \delta|Y_{t-h}=x}|Y_{t} = x
        \right]\\
        & = \mathbb{E}_{Y_{t-h}}\left[
        \mathbb{E}_{\tau|Y_{t-h}} \sbrc{\mathbf{1}_{\{\tau<t-h\}} | Y_{t-h}=x, Y_t=x}|Y_{t} = x
        \right] \mathbb{P}\cbrc{\tau < t-h-
    \delta|Y_{t-h}=x}\\
    & \overset{(iii)}{=} \mathbb{E}_{Y_{t-h}}\left[
        \mathbb{E}_{\tau|Y_{t-h}} \sbrc{\mathbf{1}_{\{\tau<t-h\}} | Y_{t-h}, Y_t=x}|Y_{t} = x
        \right] \mathbb{P}\cbrc{\tau < t-h-
    \delta|Y_{t-h}=x}\\
    & =\mathbb{E}_{\tau}\sbrc{\mathbf{1}_{\{\tau<t-h\}} | Y_t=x}
     \mathbb{P}\cbrc{\tau < t-h-
    \delta|Y_{t-h}=x},
    \end{aligned}
\end{equation*}
Here $(i)$ follows because $\mathbf{1}_{\{\tau<t-h\}}$ is non-zero only when $Y_{t-h} = Y_t = x$. Also, $(ii)$ follows because $\mathbf{1}_{\{\tau<t-h\}}$ is a (deterministic) function of $Y_{t-h}$ and $Y_t$. Also, by Markov property, $\mathbb{P}(\tau < t-h-\delta|Y_{t-h}, Y_t) = \mathbb{P}(\tau < t-h-\delta|Y_{t-h})$. Then, $(iii)$ follows again because $\mathbf{1}_{\{\tau<t-h\}}$ is non-zero only when $Y_{t-h} = Y_t = x$.
Thus,
\begin{align}
S(h+\delta)
&= S(h)\,\mathbb{P}\cbrc{\tau<t-h-\delta \mid Y_{t-h}=x} \nonumber \\
&=S(h)\big(1 - \Lambda(t-h,x)\,\delta + o(\delta)\big) \label{eq:sh}
\end{align}
where the second equality follows because given $Y_{t-h}=x$, the condition $\tau<t-h-\delta$ exactly defines the process that does not leave $x$ in $[t-h-\delta,\,t-h]$, and hence  
\eqref{eq: not leave x} yields
\begin{equation*}
    \mathbb{P}\cbrc{\tau<t-h-\delta \mid Y_{t-h}=x}
= 1 - \Lambda(t-h,x)\,\delta + o(\delta).
\end{equation*}

Equation \eqref{eq:sh} further yields
\begin{equation*}
 \frac{S(h+\delta)-S(h)}{\delta}
= -\Lambda(t-h,x)\,S(h) + o(1),   
\end{equation*}
which implies the following ODE as $\delta\to 0$:
\begin{equation*}
   S'(h) = -\Lambda(t-h,x)\,S(h),
\qquad S(0)=1. 
\end{equation*}
Solving the above ODE, we get
\begin{equation*}
    S(h)
= \exp\!\Big(-\int_0^h \Lambda(t-u,x)\,du\Big)
= \exp\!\Big(-\int_{t-h}^{t} \Lambda(u,x)\,du\Big),
\end{equation*}
which completes the proof of the second claim.

\medskip
\noindent\textbf{(3) Distribution of the predicted token.}
We first derive:
\begin{equation*}
\begin{aligned}
        \mathbb{P}&(Y_\tau = y \mid Y_t=x, \tau \in (t - \delta,t])= \frac{\mathbb{P}\cbrc{Y_{\tau} = y\mid Y_t = x}}{\mathbb{P}\cbrc{Y_{\tau} \neq x\mid Y_t = x}}= \frac{\widetilde{R}_t(x,y)\delta + o(\delta)}{\Lambda(t,x)\delta + o(\delta)}.
\end{aligned}
\end{equation*}

Then, letting $\delta \downarrow 0$, we have
\begin{equation*}
\lim\limits_{\delta\downarrow0}\ \mathbb{P}\cbrc{Y_\tau = y \mid Y_t=x, \tau \in (t-\delta,t]}
= 
\frac{\widetilde{R}_t(x,y)}{\Lambda(t,x)},
\end{equation*}
which implies the distribution of next-state.
The proof is now complete.

\subsection{Proof of \texorpdfstring{Lemma~\ref{lem: alpha follows beta distribution}}{Lemma 14}}
\label{app:lem: alpha follows beta distribution}

Following from the proof of Proposition 4.1 in \citep{zheng2025fhs}, we have 
\begin{equation*}
    1-\alpha_{\tau_{k-1}} = u_{(k)},
\end{equation*}
where $\{u_{(k)}\}_{k = 1}^d$ are order statistics of $d$ independent and uniformly distributed variables on $[0,1]$, satisfying $u_{(d)} > \cdots >u_{(1)}$. It is well known that the $k$-th order statistic of $d$ i.i.d.\ random variables with the distribution of $\mathrm{Uniform}([0,1])$ follows a $Beta(k,d+1-k)$ distribution given by
\begin{equation*}
    u_{(k)} \sim Beta(k,d+1-k).
\end{equation*}

We use the following fact for the Beta distribution. Suppose that $X \sim Beta(a,b)$, then $1-X \sim Beta(b,a)$. 
Thus, we obtain the distribution of $\alpha_{\tau_{k-1}}$ as
\begin{equation*}
    \alpha_{\tau_{k-1}} \sim Beta(d+1-k,k).
\end{equation*}

The proof is now complete.

\vskip 0.2in
\bibliography{diffusion}

\end{document}